\documentclass[journal,onecolumn]{IEEEtran}	
\IEEEoverridecommandlockouts
\usepackage{cite}
\usepackage{amsmath,amssymb,amsfonts}
\usepackage{graphicx}
\usepackage{textcomp}
\usepackage{xcolor}
\usepackage{hyperref}
\usepackage{amsbsy}
\usepackage{amsthm}
\usepackage{algorithmic}
\usepackage[]{algorithm}
\usepackage{enumerate}
\usepackage{caption}
\usepackage{subcaption}
\usepackage{multirow}
\usepackage{tabu}
\usepackage[normalem]{ulem}
\usepackage{array}
\newcolumntype{P}[1]{>{\centering\arraybackslash}m{#1}}

\providecommand{\abs}[1]{\lvert#1\rvert}
\renewcommand{\b}[1]{\ensuremath{\mathbf{#1}}} 
\renewcommand{\c}[1]{\ensuremath{\mathcal{#1}}} 
\newcommand{\Ex}[1]{\ensuremath{\mathbb{E}[#1]}}  
\newcommand{\norm}[1]{\ensuremath{\left\|#1\right\|}} 
\newcommand{\ns}[1]{\ensuremath{\left\|#1\right\|^2}} 
\newcommand{\tb}[1]{\ensuremath{\tilde{\mathbf{#1}}}} 
\DeclareMathOperator*{\argmin}{arg\,min} 
\providecommand{\ip}[2]{\langle #1, #2 \rangle} 
\newcommand{\col}[1]{\textcolor{blue}{#1}} 

\def \x {{\b{x}}}
\def \u {{\b{u}}}

\def \k {{\b{k}}}

\def \w {{\b{w}}}

\def \tw {{\tb{w}}}

\def \tz {{\tilde{z}}}
\def \tf {{\tilde{f}}}
\def \th {{\tilde{h}}}
\def \g {{\tilde{g}}}

\def \K {{\b{K}}}

\def \cL {{\c{L}}}
\def \cF {{\c{F}}}

\def \cX {{\c{X}}}
\def \cY {{\c{Y}}}
\def \cH {{\c{H}}}
\def \cD {{\c{D}}}

\def \cS {{\c{S}}}
\def \cU {{\c{U}}}
\def \O {{\mathcal{O}}}

\def \ctD {{\tilde{\cD}}}
\def \gh {{\hat{g}}}

\def \Ga {{\Gamma}}

\def \EE {{\mathbb{E}}}
\def \II {{\mathbb{I}}}
\def \Rn {{\mathbb{R}}}
\def \Nn {{\mathbb{N}}}
\def \PP {{\mathbb{P}}}

\DeclareMathOperator*{\argmax}{arg\,max}
\newtheorem{assumption}{Assumption}
\theoremstyle{remark}
\newtheorem{rem}{\bf Remark}
\newtheorem{theorem}{\bf Theorem}
\newtheorem{example}{\bf Example}
\newtheorem{lemma}{\bf Lemma}

\newtheorem{corollary}[theorem]{\bf Corollary}

\begin{document}
	
	\title{Sparse Representations of Positive Functions via \\ First and Second-Order Pseudo-Mirror Descent
	 \author{Abhishek Chakraborty, Ketan Rajawat, and Alec Koppel}
		\thanks{Abhishek Chakraborty is currently affiliated to NetApp, India. This work was done as part of his M.Tech thesis at IIT Kanpur, email: chakrabortyabhishek1994@gmail.com. Ketan Rajawat is with Dept. of EE, IIT Kanpur, India, email: ketan@iitk.ac.in. Alec Koppel is with Amazon, Supply Chain Optimization Technologies, Bellevue, WA, USA Email: aekoppel@amazon.com. A preliminary version of this work appeared as \cite{asilomar_paper}, but under more restrictive assumptions of convexity, and without any development of second-order methods.}
	}
	\maketitle
	\begin{abstract}
We consider expected risk minimization problems when the range of the estimator is required to be nonnegative, motivated by the settings of maximum likelihood estimation (MLE) and trajectory optimization. To facilitate nonlinear interpolation, we hypothesize that the search space is a Reproducing Kernel Hilbert Space (RKHS). We develop first and second-order variants of stochastic mirror descent employing (i) \emph{pseudo-gradients} and (ii) complexity-reducing projections. Compressive projection in the first-order scheme is executed via kernel orthogonal matching pursuit (KOMP), which overcomes the fact that the vanilla RKHS parameterization grows unbounded with the iteration index in the stochastic setting. Moreover, pseudo-gradients are needed when gradient estimates for cost are only computable up to some numerical error, which arise in, e.g., integral approximations. Under constant step-size and compression budget, we establish tradeoffs between the radius of convergence of the expected sub-optimality and the projection budget parameter, as well as non-asymptotic bounds on the model complexity. To refine the solution's precision, we develop a second-order extension which employs recursively averaged pseudo-gradient outer-products to approximate the Hessian inverse, whose convergence in mean is established under an additional eigenvalue decay condition on the Hessian of the optimal RKHS element, which is unique to this work. Experiments demonstrate favorable performance on inhomogeneous Poisson Process intensity estimation in practice.
	\end{abstract}
%
%


\section{Introduction} \label{sec:introduction}

Function interpolation over nonnegative range arises in numerous settings: trajectory optimization \cite{francis2017stochastic}, as well as maximum likelihood estimation \cite{drazek2013intensity}. In this work, we focus on nonnegative function estimation problems defined by an expected value minimization of a non-convex cost satisfying Polyak--\L{ojasiewicz} condition. Further, the feasible set is a Reproducing Kernel Hilbert Space (RKHS) \cite{berlinet2011reproducing}, motivated by their nonlinear interpolation attributes. We develop first and second-order functional variants of pseudo-mirror descent to solve this problem that, different from prior work, explicitly trade the dynamic memory or representational complexity of the RKHS estimates with their accuracy. Our main theoretical results establish trade-offs between sublinear convergence in expectation and the function complexity. The driving example throughout is to obtain accurate and efficient representations of the intensity of an inhomogeneous Poisson Process \cite{drazek2013intensity} from incrementally revealed samples.

To contextualize our approach, let us set aside the issue of range nonnegativity and the function parameterization for the moment. Finding a linear statistical model may be formulated as an expected value minimization over Euclidean space, which, when analytical expressions for the optimizer are unavailable (i.e., objective is not a quadtratic), one can employ Newton or gradient method to obtain the globally optimal solution\cite{boyd2004convex}. But, doing so hinges upon gradient information being efficiently computable, which fails for expected value objectives over an unknown distribution \cite{dall2020optimization}. For such settings, stochastic approximations are necessary \cite{slavakis2014modeling}, which use stochastic gradients (or stochastic Quasi--Newton updates) in lieu of exact updates. Their performance depends upon the properties of the unknown data distribution, and one may guarantee their performance only probabilistically  \cite{shapiro2014lectures}. In this work, we build upon a stochastic variant of proximal gradient method \cite{rockafellar2009variational}, i.e., stochastic mirror descent, which substitutes Euclidean distance by a Bregman divergence in the reformulation of gradient iteration as a local quadratic minimization. Doing so yields convergence benefits when the feasible set is structured, e.g., the probability simplex \cite{beck2003mirror}. We further develop Quasi--Newton iterations based upon tracking a matrix of recursively constructed gradient outer-products. 

To contextualize the RKHS specification, we note that
the strong guarantees for optimization over Euclidean space is belied by the fact that linear models are outperformed by universal function approximators \cite{kuurkova1992kolmogorov}. Specifically, deep neural networks (DNNs) \cite{anthony2009neural} and kernel methods \cite{hofmann2008kernel} have set the standard across computer vision \cite{krizhevsky2012imagenet} and natural language processing \cite{yu2011introduction}, among other fields. We focus on RKHS parameterizations since (i) under a suitable choice of kernel, they, in fact, can be equivalent to certain DNNs \cite{mairal2016end}; (ii) their training may explicitly trade-off the accuracy of an iterative update with a function's representational complexity \cite{kivinen2004online}; and (iii) one may directly co-design the kernel and Bregman divergence for likelihood objectives \cite{lei2020convergence}.

One challenge associated with the RKHS setting is that the function's representational complexity, owing to the Representer Theorem \cite{kimeldorf1971some}, is proportional to the sample size, which for expected risk minimization, approaches infinity. A myriad of approaches to approximating RKHS functions exist: those based upon matrix factorization \cite{williams2001using}, spectral properties of the kernel matrix \cite{cedric}, and random feature techniques \cite{rahimi2008random}. Random features may be combined with random sampling \cite{yang2017randomized} or subspace projections \cite{wang} based upon greedy compression \cite{vincent} to develop approximated online RKHS updates. Among kernel approximations, subspace projection methods that fix the per-update error rather than the complexity have been shown to obtain the strongest performance in theory and practice \cite{koppel2019parsimonious,koppel2021consistent}, and hence we adopt them here to sparsify stochastic descent directions in RKHS.

Two intertwined challenges emerge when learning nonnegative functions: enforcing the range constraint and evaluation of the stochastic gradient. Upon suitable initialization, one may ensure updates remain within the space of positive functions via specifying Bregman divergence as the KL divergence \cite{csiszar1975divergence}. However, the gradient may only be partially evaluable, as in the numerical evaluation of an integral. This issue arises in point process likelihoods \cite{yang2019learning}, as well as the case where the true kernel is a latent convolution of one's chosen kernels \cite{mairal2014convolutional}. In these cases, \emph{pseudo-gradients} \cite{poljak1973pseudogradient}, i.e., stochastic descent directions positively correlated with the true gradient, must be used. Doing so in online pseudo-mirror descent has recently achieved state of the art performance in intensity estimation in point processes \cite{yang2019learning}. In this work, we improve upon this approach by designing projected variants that yield tunable trade-offs between convergence accuracy and complexity, and in doing so, set a new standard for performance in practice among first-order schemes.

To refine the numerical precision in practice,  we also develop a Quasi--Newton iteration in mirror space that fixes the RKHS subspace dimension, which permits using any infinite-dimensional kernel such as a Gaussian, in contrast to \cite{calandriello2017second, calandriello2017efficient}. This iteration constructs Hessian approximations using recursively averaged pseudo-gradient outer-products. Overall, then, our technical contributions to:
\begin{enumerate}[(i)]
	\item formulate the expected risk minimization problem over RKHS functions with nonnegative range (Sec. \ref{sec:problem});
	\item develop a projected pseudo-mirror descent (Algorithms \ref{SPFPMD_algorithm}) that uses pseudo-gradients in the RKHS setting. To maintain a finite model order, the first-order scheme employs a compressive projection over the dual iterates defined by the Bregman divergence conjugate (Sec. \ref{sec:algorithm});
	
	\item develop a stochastic Quasi-Newton variant as Algorithm \ref{Newton_step} that employs employs a fixed-subspace projection in RKHS. Together, we obtain specific parametric updates in terms of weights and past training examples; Additionally, we fuse the merits of the dynamic and fixed-subspace iterations via a hybrid first and second-order scheme (Algorithm \ref{SPPPOT_Newton_step} in Sec. \ref{sec:algorithm_Newton});
	\item establish that the resulting sub-optimality of the first-order scheme converges linearly in expectation up to a bounded error neighborhood under constant step-size (Theorem \ref{thm_sync}), and provide a non-asymptotic complexity characterization of the learned function's parameterization (Sec. \ref{sec:convergence}). Further, we establish the convergence of the second-order scheme in Theorem \ref{thm_Newton} up to a persistent bias that arises from function approximation error;
	\item evaluate numerically (Sec. \ref{sec:experiments}) the proposed approach for fitting normalized intensity of an inhomogeneous Poisson Process on synthetic 1-D Gaussian toy dataset, real world NBA Stephen Curry dataset \cite{yang2019learning} and 2-D Chicago crime dataset \cite{flaxman2017poisson}. In all instances, we observe a favorable tradeoff relative to offline and online benchmarks.
\end{enumerate}

In addition, we show the similarity of functional dual-averaging to mirror descent in Appendix \ref{app:func_dual_averaging} and extensions of our first order algorithm for supervised learning problems and experimentation for multi-class kernel logistic regression on MNIST dataset \cite{lecun1998mnist} in Appendices \ref{app:SPPPOT_KLR} and \ref{app:KLR_simulations} respectively. We also provide an almost sure convergence proof for first order algorithms when we restrict the selection of pseudo-gradients to stochastic gradients (Appendix \ref{app:as_proof}).

\smallskip
\noindent {\bf Additional Context.} 
Most similar stochastic first-order methods in RKHS \cite{koppel2019parsimonious} mandate the Bregman divergence of mirror descent as the RKHS norm and hence violate non-negativity. Similarly \cite{yang2019learning} is an online algorithm that preserves positivity but requires the dictionary defining the functional estimate to be static at the outset of training, which incurs bias that is difficult to quantify. By contrast, our dynamic memory approach leads to a flexible representation whose error is directly tunable in theory and practice. Offline reformulations of optimization over RKHS with the non-negative range also have been encapsulated as a semidefinite program in \cite{marteau2020non, aubin2020hard, aubin2021handling}. Such approaches are not easily extendable into an incremental setting, as  each new sample introduces an additional constraint.

Quasi--Newton iterations in RKHS, rather than employing classical Broyden updates \cite{NoceWrig06}[Ch. 6], extend gradient outer-product constructions of Hessian approximations for vector-valued iterates, as in \cite{calandriello2017second, calandriello2017efficient}. These approaches fix RKHS subspace sizes to ensure gradient outer products belong to a static matrix subspace. In these works, regret bounds are developed in terms of the RKHS weight sequence, which fail to quantify the function approximation error. By contrast, our convergence results explicitly establish that the memory reduction employed to make Hessian estimates computable manifests as a persistent bias in the algorithm's radius of convergence, which we quantify in terms of the decay rate of the eigenvalues of the Hessian of the optimal objective.

%

	\section{Problem Formulation}\label{sec:problem}
\noindent {\bf Notation.} We denote functions by small scalar alphabets such as $f$, $z$ and $g$. The Hilbert space is denoted by $\cH$ and its dual space by $\cH^\ast$. All norms applied to elements of $\cH$ are Hilbert norms, henceforth denoted simply by $\norm{\cdot}$. The norm applied on the elements of the $\cH^\ast$ are represented as $\norm{\cdot}_{\ast}$.

Consider the problem of expected risk minimization in the online setting: independent identically distributed (i.i.d.) training samples  $\{\x_t\}_{t\geq 1}$ are observed in a sequential manner. Here, $\x_t \in \cX \subset \Rn^d$ denote the instances or input data points from an unknown distribution $\mathbb{P}(\x)$. Based upon this input stream, we seek to fit a (possibly unnormalized) density $f$ that belongs to a hypothesized function class $\cH$ according to its ability to minimize a loss function $\ell(f(\x))$. This loss is associated with the negative log-likelihood of a probabilistic model. We are interested in minimizing its expectation over the unknown data distribution $\mathbb{P}(\x)$, i.e.,
\begin{align} \label{ff}
	f^\star=\argmin_{f \in \cH_+} R(f) := \Ex{\ell(f(\x))} \; .
\end{align}
Throughout, we occasionally abbreviate the instantaneous cost as $r_t(f):=\ell(f(\x_t))$. We focus on the case that the range of $f$ is nonnegative, which arises inherently in the estimation of probabilistic models. An important point to underscore is the fact that the optimization of function $f$ is over the whole space $\cH$, and the positivity of $f$ is not enforced explicitly by hard constraints. Instead, by proper selection of positivity-preserving Bregman divergence (e.g., the KL divergence), one may implicitly enforce that constraint, as detailed in Sec. \ref{sec:algorithm}.

Next, we specify the function class $\cH\supset\cH_+$ to which estimators $f$ belong.  Specifically, we hypothesize that $\cH$ is a Hilbert space associated with symmetric positive definite kernel $\kappa$ that satisfies $\cH:=\overline{\text{span}(\kappa(\x,\cdot))}$, where every $f \in \cH$ can be written as linear combination of the kernel evaluations. The elements of $\cH$ also satisfy the reproducing property which states that $\ip{f}{\kappa(\x,\cdot)} = f(\x)$ for all $\x\in\cX$. Hilbert spaces of this type are called RKHS \cite{berlinet2011reproducing}. Examples range from simple kernel functions such as the Gaussian  $\kappa(\x,\x') = \exp(-\norm{\x-\x'}^2/2 \tilde{c})$ and polynomial  $\kappa(\x,\x')=(\x^\top\x+ \tilde{b})^{\tilde{c}}$, to data-dependent convolutional kernels \cite{mairal2014convolutional}.

With the objective and the search space clarified, we present a representative of Poisson intensity estimation, which is our driving application throughout.

\smallskip
\begin{example}\textit{Poisson Point Process intensity estimation: }\label{eg:poisson}
	Poisson Processes are a family of probabilistic models used to count the number of events $N(\mathcal{T})$ within an interval $\mathcal{T}\subset \mathcal{S}\subset \mathbb{R}^d$, where  $N(\cdot)$ is a stochastic process. A fundamental question that arises in its use is the intensity parameter $\lambda(\cdot)$, which determines the rate $\lambda(s)$ at which new events occur in an infinitesimally small time-increment, i.e., $\lambda(s)=\lim_{\Delta s \rightarrow 0} \mathbb{E} [N(\Delta s)]/(\Delta s)$. In inhomogeneous cases, this parameter is a nonlinear function, such that the likelihood associated with Poisson points $\{\mathbf{t}_n\}_{n=1}^N$ takes the form:
	\begin{align} \label{eq:poisson_likelihood}
		L(f) = \prod_{n=1}^N \lambda(\mathbf{t}_n) \exp\left\{-\int_{\mathcal{S}} \lambda( \mathbf{t}) d\mathbf{t} \right\} .
	\end{align}
	%
	Then, one may construct an instantiation of \eqref{ff} by considering the negative log-likelihood of \eqref{eq:poisson_likelihood}, inspired by \cite{flaxman2017poisson}:
	%
	%
	\begin{align} \label{neg_log_likelihood2}
		R(f) = - \sum_{n=1}^N \log(\lambda(\mathbf{t}_n)) + \int_{\cS} \lambda(\mathbf{t}) d\mathbf{t} ,
	\end{align}
	where one may identify that the Poisson points $\mathbf{t}_n$ play the role of $\x_n$, and $\lambda(\cdot)$ is the unknown function $f(\cdot)$ we seek to estimate which is required to be nonnegative. This instance of unsupervised learning has been studied both when $\{\mathbf{t}_n\}_{n=1}^N$ are available all at once or revealed incrementally (with $N\rightarrow \infty$), respectively, in \cite{flaxman2017poisson} and \cite{yang2019learning}.  We focus on online approaches to \eqref{neg_log_likelihood2} -- see \cite{drazek2013intensity} for further details.
	%
\end{example}

In this work, we focus on algorithms to solve \eqref{ff} by designing search directions that allow the function to move in the interior of the space of nonnegative ranged functions. Next, we present a clarifying remark about the complexity implications of searching over RKHS.

\begin{rem}\label{remark_erm} {\bf (Empirical Risk Minimization)}
	An important special case of \eqref{ff} is \emph{empirical risk minimization} (ERM) where a fixed collection of data $\cD:=\{(\x_n)_{n=1}^N\}$ is available, and we seek to find the empirical predictor
	\begin{align}\label{empirical}
		\hat{f}_N = \arg\min_{f\in\cH} \frac{1}{N}\sum_{n=1}^N r_{n}(f) .
	\end{align}
	Observe that $\hat{f}_N$ in \eqref{empirical} belongs to the Hilbert space, and hence is infinite-dimensional. However, the Representer Theorem \cite{repthm2} establishes that $\hat{f}_N$ takes the form:
	\begin{align}\label{rrt}
		\hat{f}_N(\cdot) = \sum_{n=1}^N w_{n}\kappa(\x_{n},\cdot) ,
	\end{align} 
	where $\{w_m\}_{m=1}^N$ are some real-valued weights. Substituting \eqref{rrt} into \eqref{empirical} reduces searching over $\cH$ to real-valued space $\mathbb{R}^N$:
	\begin{align}\label{eq:empirical}
		\hat{\w}_N &= \arg\min_{\w\in \mathbb{R}^N} \frac{1}{N}\sum_{n=1}^N r_{n}(\w^\top\k_{\cD}(\x_{n})) , 
	\end{align}
	where we have collected kernel evaluations $\{\kappa(\x_{n},\x_{m})\}_{n}$ into a vector called the empirical kernel map $\k_{\cD}(\x_{n}) \in \Rn^N$ and $[\kappa(\x_n,\x_m)]_{n,m=1}^{N,N}$ into the Gram, or kernel, matrix $\K_{\cD\cD} \in \Rn^{N \times N}$. 
	%
	Observe that as the sample size $N \rightarrow \infty$, due to the curse of kernelization, one must balance solving \eqref{ff} to optimality with ensuring that the memory remains bounded. In this work, we focus on \eqref{ff} with the understanding that \eqref{empirical} is a special case when the i.i.d. data $\{\x_n\}_{n=1}^N$ are sampled uniformly at random from a stationary distribution.
\end{rem}

Hereafter, we introduce some technical machinery needed for our algorithmic development in the following section.

\subsection{Bregman Divergence and Pseudo-Gradients}\label{subsec:prelim}
{\bf \noindent Bregman Divergence.} We proceed by presenting technicalities pertaining to the mirror map and the Bregman divergence before defining the proposed algorithm.
	%
	Let $\psi: \cH \rightarrow \Rn$ be a proper, closed, smooth, and strongly convex functional with respect to the RKHS norm $\left \lVert \cdot \right \rVert$\footnote{We denote the RKHS norm $\|f \|^2=\|f \|_{\mathcal{H}}^2$ of $f\in\mathcal{H}$ as $\|f \|_{\mathcal{H}}^2 = \langle f, f \rangle=\sum_{i,j} w_i, w_j \kappa(\mathbf{x}_i, \mathbf{x}_j)$, with $f$ admitting the weighted expansion $f=\sum_i w_i \kappa(\mathbf{x}_i, \cdot)$ via the Representer Theorem, with $w_i\in\mathbb{R}$.}
	. The Fenchel conjugate of $\psi$ is denoted as $\psi^\ast : \cH^\ast \rightarrow \Rn$, where $\cH^\ast$ is the dual space of $\cH$. Define the shorthand for the objective evaluated at the gradient of the dual $R_\psi(z)=(R \circ \nabla\psi^\ast)(z)=R(\nabla\psi^\ast(z))$ where $z \in \cH^\ast$. This composition allows one to write $\nabla R_\psi(\nabla \psi (f))=\nabla R(f)$, $f \in \cH$ since $\nabla \psi^\ast = (\nabla \psi)^{-1}$. This identity will be employed later in the convergence analysis.
	The purpose of defining functional $\psi$ is that it induces a distance-like Bregman divergence $B_{\psi}:\cH\times\cH\Rightarrow \Rn$ \cite{frigyik2008functional}:
	\begin{align}\label{bregsplit}
		B_{\psi}(f,\tf):=\psi(f)-\psi(\tf)- \langle \nabla \psi(\tf), f-\tf\rangle_{\cH}.
	\end{align}
	The functional Bregman divergence in \eqref{bregsplit} behaves mostly like its vector-valued counterpart, i.e., it satisfies non-negativity, strong-convexity in the first argument, as well as a generalized Pythagorean theorem. We present a few examples next.
	\begin{enumerate}[i)]
		\item \textit{KL-divergence or I-divergence:}\label{kl_divergence} Let $\cH$ be the space of probability density functions, constructed from a radial kernel such as Gaussian, Student-t, or Laplacian, such that $f$ takes positive values and integrates to $1$. Defining $\psi(f) = \ip{f}{\log(f)-1}_{\cH}$ similar to \cite{yang2019learning} yields $B_{\psi}(f,\tilde{f}) = \ip{f}{\log(f/\tf)}_{\cH}$ if $f$ is absolutely continuous with respect to $\tf$. For this case, $B_{\psi}$ is the KL-divergence or I-divergence associated with the convex map $\psi(f)$  \cite{csiszar1975divergence}.
		\item \textit{Squared RKHS norm difference:}\label{squared_difference} The choice $\psi(f) = \frac{1}{2}\ns{f}$ leads to $B_{\psi}(f,\tf) = \frac{1}{2}\ns{f-\tf}$. 
		\item \textit{Squared Mahalanobis difference:}\label{squared_mahalanobis_distance} Let $\cL$ be a compact self-adjoint operator, then the choice $\psi(f) = \frac{1}{2}\ip{\cL f}{f}_{\cH}$ results in the Bregman divergence $B_{\psi}(f,\tf)) = \frac{1}{2}\ip{\cL(f-\tf)}{f-\tf}_{\cH}^2$, which generalizes the squared difference. For the simple case when $\cL$ is a diagonal operator, it amounts to pointwise multiplication with a weight function.
	\end{enumerate}
	{\bf \noindent Pseudo-Gradient.}  With Bregman divergence clarified, we shift to defining pseudo-gradients, which are those directions $g_t$ with nonnegative expected (unnormalized) cosine similarity with the gradient $\nabla_f R(f)$:
	\begin{align}\label{inner_product_pseudograd}
		\langle \nabla R(f), \EE[g_t | \cF_t] \rangle \geq 0 ,
	\end{align}
	where $g_t$ is the pseudo gradient associated with sample $\x_t$ and $\cF_t$ denotes the past sigma algebra which contains all the past data points one iteration back, i.e. $\{\x_i\}_{i=1}^{t-1}$. Equivalently, a pseudo-gradient $g_t$ is any search direction that forms an acute angle with the original gradient $\nabla R(f)$ in the dual space, as stated in \cite{poljak1973pseudogradient}, and may be used in lieu of the true gradient when its evaluation is costly or intractable \cite{yang2019learning}. Next we present a few instances to build intuition.
	\begin{enumerate}[i)]
		\item \textit{Stochastic Gradients:} If $g_t$ is the stochastic gradient of $R(\cdot)$ at $f_t$, then $\EE[g_t | \cF_t] = \nabla R(f_t)$. So the inner product $\langle \nabla R(f_t), \EE[g_t | \cF_t] \rangle = \ns{\nabla R(f_t)}_{\ast} \geq 0$.
		\item \textit{Kernel embedding:} $\kappa:\cX \times \cX \rightarrow \Rn_{+}$ is a symmetric and positive definite kernel. So $\ip{G}{\ip{\kappa}{G}} \geq 0$ for all vector $G \in \cH$. Now $g_t$ defined as $g_t = \ip{\kappa(\x_t,\cdot)}{\nabla R(f_t)}$ is a pseudo-gradient since \eqref{inner_product_pseudograd} is satisfied by the symmetric and positive definiteness property as given above.
		\item \textit{Gradient sign:} Define the gradient as $g_t = \text{sgn}(\nabla R(f_t))$ where sgn denotes the sign operator. Then for all data points $\x \in \cX$, the inner product of \eqref{inner_product_pseudograd} is given as $\ip{\nabla R(f_t)}{\text{sgn}(\nabla R(f_t))} = \int_{\cX} \vert \nabla R(f_t(\x)) \vert d\x \geq 0$.
	\end{enumerate}
	%
	%
	The Bregman divergence and pseudo-gradient are employed in the following section for our main algorithm development.

%

	\section{Projected Pseudo-Mirror Descent}\label{sec:algorithm}

Now we shift to deriving an iterative approach to solving \eqref{ff} via a functional extension of mirror descent \cite{agarwal2011distributed}. We select a Bregman divergence that ensures positivity of the range of the function $f$ during optimization, such as the KL divergence, and present a generalization of a gradient called a ``pseudo-gradient" \cite{poljak1973pseudogradient}. The merit of employing pseudo-gradients is the ability to define approximate search directions that arise in point process intensity estimation as stated in Example \ref{eg:poisson}, as well as a broader ability to incorporate samples into the functional representation through a kernel embedding. 

To define our iterative approach to solving \eqref{ff} based upon streams of samples, we first define a first-order functional variant of stochastic mirror descent using pseudo-gradients. Next, we develop for a Quasi--Newton kernelized algorithm in mirror space. While theoretically we can only show it matches the convergence rate of its first-order counterparts, experimentally it achieves state of the art accuracy for inhomogeneous intensity estimation in Poisson Processes (Sec. \ref{sec:experiments}).
%
%
%
\subsection{First-Order Method for Sparse Positive Functions}\label{subsec:spppot}
We build upon first-order functional variant of stochastic mirror descent:
\begin{align}\label{md0}
	f_{t+1}=\arg \min \limits_{f \in \cH} &\Big(\ip{\nabla r_t(f_t)}{f}_{\cH} + \frac{1}{\eta_t} B_\psi (f,f_t) \Big),
\end{align}
where $\eta_t>0$ is a nonnegative step-size that may be constant or diminishing, but we subsequently fix as constant $\eta_t=\eta$.
Note that $B_{\psi}(f,\tilde{f})=\frac{1}{2}\|f-\tilde{f}\|^2$ reduces \eqref{md0} to functional SGD. Moreover, for \eqref{md0} to be implementable, we require closed-form expressions for Bregman divergence. Note that for the squared Mahalanobis distance, the update takes the form $f_{t+1} = f_t - \eta\cL^{-1}\nabla r_t(f_t)$, where the inverse $\cL^{-1}$ must be computable in closed-form. Worth mentioning is also the fact that finding functions with non-negative range \eqref{ff} necessitates specifying the Bregman divergence to preserve positivity.

We prioritize the use of pseudo-gradients $g_t$ in lieu of stochastic gradients $\nabla r_t(f_t)$ in \eqref{md0} since they inherently arise in evaluating gradient estimates of likelihood models arising in Poisson Processes. The resultant iteration then takes the form starting from function $f_t$ using samples $\{\x_t\}$: 
\begin{align}\label{md1}
	\tf_{t+1}=\arg \min \limits_{f \in \cH} \Big(\ip{g_t}{f}_{\cH} + \frac{1}{\eta} B_\psi (f,f_t) \Big),
\end{align}
where $g_t$ is any pseudo-gradient \eqref{inner_product_pseudograd} at sample $\x_t$.

\noindent {\bf Mirror Descent with Pseudo-Gradients.}
To clarify how \eqref{md1} is executable in practice, suppose that true stochastic gradient: 
\begin{align} \label{stochasticgrad2}
	\nabla r_t(f_t)= \ell'(f_t(\x_t)) \kappa(\x_t, \cdot)
\end{align}
is used in the mirror descent update [cf. \eqref{md0}] instead of the pseudo-gradient, i.e.,
%
%
%
$g_t=\nabla r_t(f_t)$.
%
For this case, the first-order optimality condition for \eqref{md1} may be written as
\begin{align}
	g_t+ \frac{1}{\eta} \nabla_{\tilde{f}_{t+1}} B_\psi (\tilde{f}_{t+1},f_t) = 0 . \label{I_div_opt}
\end{align}
Noting the fact $\nabla_{\tilde{f}_{t+1}} B_\psi (\tilde{f}_{t+1},f_t) = \nabla \psi(\tilde{f}_{t+1}) - \nabla \psi(f_t)$, \eqref{I_div_opt} can be expressed as 
\begin{align} \label{FPMD1}
	g_t+ \frac{1}{\eta}(\nabla \psi(\tf_{t+1}) - \nabla \psi(f_t)) = 0 .
\end{align}
Define $\tz_{t+1} = \nabla \psi(\tf_{t+1})$ and $z_t = \nabla \psi(f_t)$\footnote{Subsequently, we note that $\{z_t\}\subset\cH_\ast$, i.e., $z_t$ is an element of the dual space $\cH_\ast$ of the RKHS $\cH$ defined at the outset of Sec. \ref{sec:problem}.}, which after rearranging \eqref{FPMD1}, yields 
\begin{align}\label{FPMD2}
	\tz_{t+1} = z_t - \eta g_t .
\end{align}
Observe that $z_t = \nabla \psi(f_t)$, which implies that the function f can be recovered through the gradient of the conjugate $f_t = \nabla \psi^\ast (z_t)$. Specifically, $\psi^\ast$ is the Fenchel conjugate of $\psi$ and  $\nabla \psi^\ast = (\nabla \psi)^{-1}$. Hence the update \eqref{FPMD2} becomes
\begin{align}
	\tz_{t+1} 
	&= z_t - \eta \ell'(\nabla \psi^\ast(z_t),\x_t) \kappa(\x_t, \cdot) \label{upeqgen} ,
\end{align}
where \eqref{upeqgen} is derived via applying the chain rule and the reproducing property of the kernel to $\nabla r_t(\nabla \psi^\ast(z_t))$ analogous to \eqref{stochasticgrad2}. Thus, the functional update is executable parametrically via a data set $\mathcal{D}_{t}$ consisting of $t-1$ points that grows by one per time step with an associated coefficient vector $\mathbf{w}_{t}$:
\begin{align}\label{mddict}
	\!\!\!	\cD_{t+1} \!=\! \cD_{\!t} \! \cup \{\x_t\}  , \ 
	&[\mathbf{w}_{t+1}]_n \!=\!\!\begin{cases}
		\! [\mathbf{w}_{t}]_n & \!\! n<t \!\\
		\!-\eta \ell'(\nabla \psi^\ast(z_t\!),\x_t) & \!\! n=t ,\!
	\end{cases}\!\!\!
\end{align}
where $[\mathbf{w}_{t}]_n$ denotes the $n$-th coordinate of the vector $\mathbf{w}_{t}$. Different from classical online learning with kernels \cite{kivinen2004online}, here the weights $[\mathbf{w}_{t}]_n$ and basis points $\cD_t$ represent the function $z_t$ associated with the gradient of the Fenchel dual in \eqref{upeqgen}.


%
We note that the function evaluation $f$ at point $\x\in\cX$ then takes the form at time $t+1$ 
\begin{align}\label{iterateconv}
	f_{t+1}(\x) = \nabla \psi^\ast (z_{t+1}(\x)) = \nabla \psi^\ast(\w_{t+1}^\top \k_{\cD_{t+1}}(\x)) .
\end{align}
A technical challenge emerges due to the fact that for general pseudo-gradients $g_t \neq \nabla r_t(f_t)$, the expression \eqref{FPMD1} cannot be evaluated in closed form. However, for the pseudo-gradient defined by the kernel embedding (ii), \eqref{mddict} takes the form 
\begin{align}\label{eq:pseudo_gradient_parametric_recursion}
	\cD_{t+1} \!=\! \cD_{t} \! \cup \!\{\!\x_t\!\} \; , \ 
	&[\mathbf{w}_{t+1}]_n\!=\!\!\begin{cases}
		\![\mathbf{w}_{t}]_n &\! \x_n\! \in\! \cD_t \\
		\!-\eta g_t^\prime &\! \x_n \!=\! \x_t ,
	\end{cases}
\end{align}
where the pseudo-gradient admits an expansion via the chain rule $g_t= g_t^\prime \kappa(\x_t,\cdot)$, where $g_t^\prime$ is the scalar component of the derivative of the cost $r_t(f)$. The kernel embedding, or any other pseudo-gradient, gives rise to the specific form of $g_t^\prime$. This condition is valid when the pseudo-gradient of the objective with respect to RKHS element $f$ respects the chain rule. \medskip

\noindent {\bf Subspace Projections.}
Owing to the RKHS parameterization in terms of weights and feature vectors $\x_t$ [cf. \eqref{eq:pseudo_gradient_parametric_recursion}], the complexity of the function grows unbounded with time $t$. To surmount this issue, we project functions onto low-dimensional subspaces near the current search direction. Doing so may be achieved via Kernel Orthogonal Matching Pursuit (KOMP)\cite{vincent}, that takes an input dictionary $\ctD_{t+1}$ and a weight vector $\tw_{t+1}$, returning a lower-dimensional (compressed) function $z_{t+1}$ with dictionary $\cD_{t+1}$ and weights $\w_{t+1}$ that are $\epsilon$-away in the RKHS norm, where $\epsilon$ denotes the compression budget. 
Overall, then, {\bf S}parse {\bf P}ositive {\bf P}rojected {\bf P}seudo-Mirr{\bf o}r Descen{\bf t} (SPPPOT), takes the form 
\begin{equation}\label{eq:main_alg}
	\{z_{t+1},\cD_{t+1},\w_{t+1}\} = \text{KOMP}(\tilde{z}_{t+1}, \ctD_{t+1},\tw_{t+1},\epsilon) ,
\end{equation}
where $\tilde{z}_{t+1}$  defined in \eqref{FPMD2} employs general pseudo-gradients. The iterative process is summarized as Algorithm \ref{SPFPMD_algorithm}. Note that different from  \cite{koppel2019parsimonious}: the RKHS-norm approximation criterion is in terms of the dual sequence $\{z_t\}$ \eqref{upeqgen} -- see Algorithm \ref{kompalgo}.

\smallskip 
%
%
%
\begin{algorithm}[t]
	\caption {SPPPOT: {\bf S}parse {\bf P}ositive Functions via  {\bf P}rojected {\bf P}seudo-Mirr{\bf o}r Descen{\bf t}}
	\label{SPFPMD_algorithm}
	\begin{algorithmic}[1]
		\REQUIRE{kernel $\kappa$, step-size $\eta$, compression parameter $\epsilon$}
		\STATE \textbf{Initialize } Arbitrary small $z_0$ 
		\FOR{$t = 1, 2, \ldots $}
		\STATE \textbf{\hspace{-2mm}Read: } data realization $\x_t$ 
		\STATE \textbf{\hspace{-2mm}Evaluate: } Pseudo Gradient $g_t$
		\STATE \textbf{\hspace{-2mm}Update: } $\tz_{t+1}$ as per \eqref{FPMD2} 
		\STATE \textbf{\hspace{-2mm}Compress: } $\{\cD_{t+1},\w_{t+1}\} = \text{KOMP}(\ctD_{t+1},\tw_{t+1},\epsilon)$
		\STATE \textbf{\hspace{-2mm}Broadcast: } $z_{t+1}$
		\ENDFOR
		\STATE \textbf{Evaluation of actual function $f_{t+1}$ at $\x$:} As per \eqref{iterateconv}
	\end{algorithmic}
\end{algorithm}

{\noindent \bf Examples.} With Algorithm \ref{SPFPMD_algorithm} defined, we next specify the updates in terms of two Bregman divergences defined in Sec. \ref{sec:algorithm}, one that ensures positivity, \eqref{kl_divergence} KL-divergence, and one that does not: \eqref{squared_difference} squared RKHS norm.
\begin{enumerate}[\hspace{-1mm}i)\hspace{-1mm}]
	\item \textit{KL-divergence or I-divergence:} As in \eqref{kl_divergence}, the update for the auxiliary variables $\tilde{z}_{t+1}$ and $z_t$ \eqref{FPMD2} have the explicit forms $\tz_{t+1} = \log(\tf_{t+1})$ and $z_t = \log(f_t)$, which may be substituted  into \eqref{FPMD2} to yield an expression for $\tilde{f}_{t+1}$ as 
	%
	\begin{align}
		\log(\tf_{t+1}) = \log(f_t) - \eta g_t,
	\end{align}
	which, following exponentiation, permits us to write
	\begin{align}\label{kldiv_raw_update}
		\tf_{t+1} = f_t \exp (-\eta g_t).
	\end{align}
	Observe from \eqref{kldiv_raw_update} that the update for $f_t$ is non-linear and hence a vanilla application of the Representer Theorem \eqref{rrt} is not possible. However, by focusing on the auxiliary sequences $\tz_{t+1} = \log(\tf_{t+1})$ and $z_t = \log(f_t)$, \eqref{kldiv_raw_update} permits linearization of the form \eqref{FPMD2}, and hence allows a parametric weight update and dictionary update with respect to the auxiliary function $\tz_{t+1}$ akin to \eqref{mddict}.
	Note that from \eqref{kldiv_raw_update}, the function $\tf_{t+1}$ can be expressed recursively as a product of functions from $f_0$ to $f_t$. So if the initialization $f_0$ is nonnegative, and the projection in \eqref{eq:main_alg} is onto the space of functions with nonnegative range, then function positivity is ensured throughout training. This is salient for Example \ref{eg:poisson}, and MLE more broadly.
	
	\begin{algorithm}[t]
		\caption {Destructive Kernel Orthogonal Matching Pursuit}
		\label{kompalgo}
		\begin{algorithmic}
			\STATE\textbf{Require:} $\tz$ in form of $(\ctD, \tw)$, budget $\epsilon$
			\STATE\textbf{Initialize:} $z=\tz$ so that $(\cD,\w)=(\ctD,\tw)$ 
			\WHILE{$\cD \neq \emptyset$}
			\FOR{$\x_j \in \cD$}
			\STATE Evaluate $\w_j^\star\!=\!\arg\min_{\w}\! \gamma_j(\w)\!\!:=\!\|\tz-\!\!\!\!\! \sum\limits_{\x_n \in \cD \setminus \{\x_j\}\!\!\!\!\!\!}\!\!\!\!\!\!\!\!w_n \kappa(\!\x_n,\!\cdot\!)\!\|_\ast $
			\ENDFOR
			\IF{$\gamma_j(\w_{j}^{\star}) > \epsilon$ for all $\x_j \in \cD$}
			\STATE \textbf{break}
			\ELSE
			\STATE Prune $\cD \leftarrow \cD \setminus \{\x_{j^\star}\}$ where $j^\star = \arg\min\gamma_j$
			\STATE Update weights $\w \leftarrow \w_{j^\star}^\star$
			\ENDIF
			\ENDWHILE
			\STATE\textbf{return} $z$, such that $\norm{z-\tz}_\ast \leq \epsilon$
		\end{algorithmic}
		\label{algo}
	\end{algorithm}  
	
	\item \textit{Contrasting example: Squared RKHS-norm difference:} With \eqref{squared_difference}, the quantities in \eqref{FPMD2} take the form $\tz_{t+1} = \nabla \psi(\tf_{t+1}) = \tf_{t+1}$ and $z_t = \nabla \psi(f_t) = f_t$, which yields 
	\begin{align}
		\tf_{t+1} = f_t - \eta g_t,
	\end{align}
	where the dictionary stacks past points [cf. \eqref{mddict}], and coefficients are updated as 
	%
	\begin{align}
		\tw_{t+1} = [\w_t, \quad -\eta g_t].
	\end{align}
	With KOMP-based subspace projections, this update is identical to 
	%
	POLK \cite{koppel2019parsimonious} if one additionally equates the pseudo-gradient with the stochastic gradient $g_t=\nabla r_t (f_t)$ [cf. \eqref{stochasticgrad2}]. For this choice of Bregman divergence, under positive initialization, positivity will be violated.
	

\end{enumerate}

A key observation from the case of KL-Divergence is that since the update for $f$ does not yield a linear basis expansion, it cannot be computed through weighted combinations of kernel evaluations \eqref{rrt}. However, it can be recovered through inverting the logarithmic transformation, i.e., $f_{t+1}(\cdot)=\exp(z_{t+1}(\cdot)) = \exp(\w_{z,t+1}^\top\k_{\cD_{z,t+1}}(\cdot))$. Consequently, as each new datum $\x$ arrives, we can obtain its evaluation under $f$ as $f_{t+1}(\x)=\exp(\w_{z,t+1}^\top\k_{\cD_{z,t+1}}(\x))$. In addition, $f_{t+1}(\x)$ will be positive since it is obtained by positivity preserving exponential transformation on $z_{t+1}(\x)$.
Observe that this phenomenon holds more broadly for any positivity preserving Bregman divergence when its update is nonlinear in $f$ and the gradient of the Fenchel dual $\nabla\psi^*(f)$ of its inducing map $\psi(f)$ is efficiently computable.

Next, we expand upon the form of Algorithm \ref{SPFPMD_algorithm}[Steps 4-5] for Poisson point process intensity estimation (Example \ref{eg:poisson}), which are employed experimentally in Sec. \ref{sec:experiments}.

\begin{example}\label{PPP_weight_updates} Observe that for \eqref{neg_log_likelihood2}, the instantaneous loss for single data for the function $f_t$ takes the form 
	\begin{align} \label{obj_func_poisson3}
		\ell(f_t(\x_t)) = - \log(f_t(\x_t)) + \int_{\cX} f_t(\x) d\x.
	\end{align}
	When no special form of $f_t$ is known for analytically, one must evaluate the integral numerically using, e.g., Bayesian quadrature or kernel smoothing. We adopt the later approach, inspired by \cite{flaxman2017poisson}: the quantity $\int_{\cX} f_t(\x) d\x$ is approximated by $h \sum_{j \in \cU} f_t(\u_j)$, where $\u_j$ are uniform grid points over the sample space $\cX$ whose indices lie in set $\cU:=\{\u_j\}$, and $h$ is the infinitesimally small uniform grid area, similar to Trapezoidal rule for discrete approximations of integrals. Then differentiating with respect to $f_t$ yields the pseudo-gradient 
	\begin{align}\label{PPP_grad}
		g_t = - \frac{1}{f_t(\x_t)}\kappa(\x_t,\cdot) + h \sum_{j =1}^{|\mathcal{U}|} \kappa(\u_j,\cdot).
	\end{align}
	\eqref{PPP_grad} is a pseudo-gradient as the second term of \eqref{PPP_grad} can be seen as an approximated version of the integral using kernel embedding. Note that for the KL divergence/I-divergence, the transformation $f_t = \exp(z_t)$ yields the pseudo-gradient $g_t$: 
	\begin{align} \label{stochastic_grad}
		g_t = - \frac{1}{\exp (z_t(\x_t))}\kappa(\x_t,\cdot) + h \sum_{j =1}^{|\mathcal{U}|} \kappa(\u_j,\cdot).
	\end{align}
	The subtlety of grid points versus Poisson points causes slight differences in the updates relative to \eqref{mddict}. Specifically, 
	in \eqref{stochastic_grad}, the point $\x_t$ are samples from the unknown Poisson process, while the points $\u_j$ are fixed grid points across updates. These components come together to specify Algorithm \ref{SPFPMD_algorithm} as follows:
	\begin{itemize}
		\item Initialize the dictionary $\ctD_{1}$ with uniform grid points $\u_j$ for all $\u_j\in \mathcal{U}$, with corresponding weight vector elements as $[\tw_{1}]_j=-\eta h$ for each $\u_j\in\cU$.
		\item Receive Poisson samples $\x_t$, compute the pseudo-gradient \eqref{stochastic_grad}, and update dictionary $\ctD_{t+1}= \cD_{t} \cup \!\{\!\x_t\!\}$. Given that the set of grid points satisfies $\mathcal{U}\subset \mathcal{D}_t$, the corresponding weight update is 
		\begin{align} \label{weight_update_poisson}
			&[\tilde{\mathbf{w}}_{t+1}]_n=\begin{cases}
				[\mathbf{w}_{t}]_n- \eta h & \x_n \in \mathcal{U} \\
				[\mathbf{w}_{t}]_n &  \x_n \notin \mathcal{U} \\
				\frac{\eta}{\exp (z_t(\x_t))} & \x_n = \x_t \; .
			\end{cases}
		\end{align}
	\end{itemize}
	KOMP is then applied to the dictionary $\ctD_{t+1}$ and weights $\tilde{\mathbf{w}}_{t+1}$ to yield $\cD_{t+1}$ and $\mathbf{w}_{t+1}$ . Since grid points $\{\u_j\}$ are required to approximate the integral in the pseudo-gradient, their presence in the dictionary is fixed. Therefore, our selection scheme discerns which Poisson samples $\x_t$ and associated weights $w_t$ are statistically significant for estimating the inhomogeneous intensity parameter $f(\cdot)=\lambda(\cdot)$ as in Example \ref{eg:poisson}.
	
	Next, we verify that \eqref{stochastic_grad} is indeed a pseudo-gradient. Following similar reasoning to\cite{yang2019learning}, the inner product between $\nabla R(f_t)$ and $\mathbb{E}[g_t \vert \mathcal{F}_t]$ can be written as $\langle \nabla R(f_t), \mathbb{E}[g_t \vert \mathcal{F}_t] \rangle = \int_{\mathcal{X}} \int_{\mathcal{X}} \left(-\frac{1}{f_t(\mathbf{x})} + 1 \right) \kappa(\mathbf{x},\mathbf{y}) \left(-\frac{1}{f_t(\mathbf{y})} + 1 \right) d \mathbf{x} d \mathbf{y}$, where we have taken a continuous domain data for algebraic simplicity and hence writing in terms of empirical sum is avoided. The above equation is lower bounded by $\zeta_{\min} \left \lVert -\frac{1}{f_t} + \mathbf{1} \right\rVert ^2$ using the positive definiteness property of the kernel, where $\mathbf{1}$ is vector of $1$'s matching the dimension of $f_t$ and $\zeta_{\min}$ is the minimum eigenvalue of the integral operator associated with $\kappa(\cdot,\cdot)$. Thus, the inner-product will be non-negative, meaning the pseudo-gradient property [cf. \eqref{inner_product_pseudograd}] holds.
\end{example}	

With Poisson intensity estimation example detailed, we next shift to incorporating approximate Hessian information.

%

\section{Second-Order Online Pseudo-Mirror Descent}\label{sec:algorithm_Newton}
Iterative updates that incorporate Hessian information, i.e., variants of Newton's Method, often surpass schemes that only use gradient information in convergence rate and numerical precision. However, they often are inapplicable to online settings due to the quadratic computational cost in the sample size of  Hessian inversion \cite{NoceWrig06}[Ch. 6]. Recently, memory-efficient Quasi--Newton schemes in the online setting based on incremental Hessian updates have been developed \cite{mokhtari2017iqn,gaoincremental}.

However, these techniques do not generalize to the RKHS setting due to their dependence on finite-dimensional parametric forms for the search space. Efforts to bridge this gap have been addressed by introducing Nystr\"{o}m approximations to the kernel matrix \cite{williams2001using}, resulting in memory-efficient schemes \cite{calandriello2017second}. Noticeably, however, Nystr\"{o}m approximation does not permit the points of approximation to be well-calibrated to the data domain. By contrast, methods which optimize the location from which kernel matrices are sampled via \emph{inducing} inputs can refine their statistical error \cite{snelson2006sparse,wang}. Hence, hereafter, we develop a Quasi--Newton scheme via inducing input approximations to the kernel matrix, which additionally employs pseudo-gradients rather than true stochastic gradients.

To continue, we denote as $\g_t$ [cf. \eqref{inner_product_pseudograd}] the pseudo-gradient defined over a fixed subspace $|\mathcal{D}|$ associated with weight vector $\mathbf{w}_t\in\mathbb{R}^{|\mathcal{D}|}$. Associated with it is a a Hessian approximation $\mathbf{A}_{t+1}\in\mathbb{R}^{|\mathcal{D}|\times |\mathcal{D}|}$ that is approximated using rank-1 outer products of pseudo-gradients observed so far as 
\begin{align} \label{hess_update}
	\mathbf{A}_{t+1} = \mathbf{A}_t + \g_t \g_t^\top = \delta \mathbf{I} + \sum_{j=1}^t \g_j \g_j^\top ,
\end{align}
%
%
%
%
where, $\delta \mathbf{I}$ is the initialized Hessian matrix $\mathbf{A}_0$. Observe that without any subspace approximation, i.e., if $\mathcal{D}=\mathcal{D}_t$ is allowed to grow unbounded as \eqref{FPMD2}, then the dimensionality of $\mathbf{A}_t$ will be $(t-1)\times (t-1)$ at step $t$. To address this issue, we fix our subspace of inputs defined by data matrix $\cD_t = \cD$ for all $t$, in contrast to the  Nystr\"{o}m approximations in \cite{calandriello2017second}. Doing so ensures sufficient coverage of data matrix $\cD$ by assigning grid points over the feature space $\mathcal{X}$. We note that to more closely adhere to the concept of inducing inputs \cite{snelson2006sparse,wang}, one could optimize the locations by searching over a space of dimensionality $|\mathcal{D}|\times \col{d}$ after training completes. With this approximation, then, \eqref{hess_update} may be implemented online when $\g_t$ is computable, which  incurs storage complexity $|\mathcal{D}|\times|\mathcal{D}|$.

Observe that with this dimensionality fixed, optimization over the dual sequence of functions $z_t$ [cf. \eqref{FPMD2}] is transformed to the search over fixed dimension of weights $\w_t \in \Rn^{\abs{\cD}}$ for all $t$, where $\abs{\cD}$ is the cardinality of the set $\cD$. Here, the pseudo-gradient, computed by differentiating the loss with respect to the fixed dimensional weights, is generically expressed as $\g_t = \g_t^\prime \k_{\cD}(\x_t)$. For the case of a stochastic gradient, this simplifies to $\g_t^\prime = \nabla_{\w_t} \ell(f_t(\x_t)) = \ell^\prime (f_t(\x_t))$. Note the fact that $\g_t\in\mathbb{R}^{|\mathcal{D}|}$. Now, the gradient outer product on the right hand side of \eqref{hess_update} can be expressed as $\g_t \g_t^\top = [\g_t^\prime]^2 \k_{\cD}(\x_t) \k_{\cD}(\x_t)^\top$, which is always positive. Hence the Hessian matrix $\mathbf{A}_{t+1}$ is always positive definite and also symmetric. The Hessian matrix estimate is  positive definite via the initialization $\delta \mathbf{I}$.

Overall, the Hessian estimate \eqref{hess_update} and the feature space approximation together yield the online Quasi--Newton update 
\begin{align} \label{wt_up}
	\w_{t+1} = \w_t - \eta_t \mathbf{A}_{t+1}^{-1} \g_t.
\end{align}
Observe that since the matrix $\mathbf{A}_t\in\mathbb{R}^{|\mathcal{D}|\times |\mathcal{D}|}$ belongs to a fixed subspace of $|\mathcal{D}|\times |\mathcal{D}|$, its inversion may be executed with complexity $|\mathcal{D}|^3$ in the worst case with a naive implementation, or with $|\mathcal{D}|^2$ through application of matrix inversion lemmas. While this is less favorable than a near-linear dependence on the parameter dimension achievable by the best memory-limited Quasi--Newton methodologies, it is far better than the computational cost of a full Newton step. In particular, computing a full Newton step for this setting would require evaluating the Hessian with respect to the $|\mathcal{D}|$-dimensional kernel truncation, which depends on doubly infinitely many realizations $\mathbf{x}_n$ of the random variable $\mathbf{x}$.
Via \eqref{wt_up}, the auxiliary functional iterates may be expressed as 
\begin{align} \label{fn_up}
	z_{t+1} = \w_{t+1}^\top \k_{\cD}(\cdot) = z_t - \eta_t \g_t^\top \mathbf{A}_{t+1}^{-1} \k_{\cD}(\cdot) .
\end{align}
\eqref{fn_up} defines a Quasi--Newton update on the dual iterate $z_t$ over the fixed-dimensional feature subspace $|\cD|$. One may recover the original function estimate via computation of the Fenchel dual \eqref{iterateconv} when the Bregman divergence admits a closed-form expression for the evaluation of the gradient of its conjugate (as is the case for KL divergence, and others). Next we specify the form of these updates for Poisson process.


\begin{algorithm}[t]
	\caption {Online Quasi--Newton Mirror Descent in RKHS}
	\label{Newton_step}
	\begin{algorithmic}[1]
		\REQUIRE{kernel $\kappa$, step-size $\eta_t$, regularizer $\delta$}
		\STATE \textbf{Initialize } Arbitrary small negative $\w_0 \in \Rn^{\abs{\cU}}$, $\mathbf{A}_0 = \delta \mathbf{I} \in \Rn^{\abs{\cU}\times\abs{\cU}}$, dictionary as grid points $\{\cD_t\} = \{\cU\}\ \forall t$.
		\FOR{$t = 1, 2, \ldots $}
		\STATE \textbf{\hspace{-2mm}Read: } sample $\x_t$
		\STATE \textbf{\hspace{-2mm}Evaluate: } Pseudo Gradient $\g_t$ as \eqref{PPP_natural_grad}
		\STATE \textbf{\hspace{-2mm}Update: } $\mathbf{A}_{t+1} = \mathbf{A}_t + \g_t \g_t^\top$
		\STATE \textbf{\hspace{-2mm}Update: } $\w_{t+1} = \w_t - \eta_t \mathbf{A}_{t+1}^{-1} \g_t$
		\STATE \textbf{\hspace{-2mm}To evaluate functional estimate, invert mirror map:}
		$$f_{t+1}(\x) = (\nabla\psi)^{-1}(z_t(\x)) = \exp(\w_{t+1} \k_{\cD_{t+1}}(\x))$$ for KL divergence.			%
		\ENDFOR
	\end{algorithmic}
\end{algorithm}

\begin{example}\textit{Quasi--Newton Poisson Intensity Estimation.}\label{eg:quasi_newton} Consider the objective in \eqref{obj_func_poisson3} that integral approximations with respect to the uniform grid points $\u_j\in\cU$. In the context of the second-order scheme, we fix the dictionary with this specification, i.e. $\cD_t = \cD = \cU$ for all $t$, meaning that no compression is required. Then, the first-order component of the update manifests by updating the auxiliary variable $z_t$ according to the gradient of the Fenchel dual \eqref{FPMD2}, as in mirror descent (natural gradient) \cite{raskutti2015information}. In particular, we rewrite the loss in \eqref{obj_func_poisson3} with respect to the auxiliary function $z_t$ as 
	\begin{align} \label{PPP_MLE_loss2}
		\ell(f_t(\x_t)) = \ell((\nabla\psi)^{-1}(z_t(\x_t))) = - \log((\nabla\psi)^{-1}(z_t(\x_t))) + h \sum_{\u_j\in\cU} (\nabla\psi)^{-1}(z_t(\u_j)) .
	\end{align}
	Specialized to the case of KL divergence, the expression $(\nabla\psi)^{-1}(z_t(\x_t))=\exp(z_t(\x_t))$ substituted into \eqref{PPP_MLE_loss2} yields 
	\begin{align} \label{PPP_MLE_loss3}
		\ell(f_t(\x_t)) = - z_t(\x_t) + h \sum_{\u_j\in\cU} \exp(z_t(\u_j)) = -\w_t^\top \k_{\cD}(\x_t) + h \sum_{\u_j\in\cU} \exp(\w_t^\top \k_{\cD}(\u_j)) .
	\end{align}
	Then, differentiating \eqref{PPP_MLE_loss3} with respect to weights $\w_t\in\mathbb{R}^{|\mathcal{U}|}$ (since $\mathcal{U}=\mathcal{D}$) yields the pseudo gradient $\g_t \in \Rn^{\abs{\cU}}$:
	\begin{align}\label{PPP_natural_grad}
		\g_t \!=\! \nabla_{\! \w_t} \ell(f_t(\x_t)) \!=\! - \k_{\cD}(\x_t) \!+\! h \! \sum_{j\in\cU} \! \exp(\w_t^\top \k_{\cD}(\u_j)) \k_{\cD}(\u_j) .
	\end{align}
	We proceed to formalizing a special case of update direction in \eqref{fn_up} and ensure it defines pseudo-gradient in the sense of \eqref{inner_product_pseudograd}. Begin by noting that we may write the pseudo-gradient as an RKHS element through appropriate multiplication with the feature map $\k_{\cD}(\cdot)$ associated with fixed dictionary $\cD$. Then, the functional gradient representation $\g_t^\top \k_{\cD}(\cdot)$ represents the derivative of the loss in \eqref{PPP_MLE_loss3} with respect to auxiliary variable $z_t$, i.e.,  $\nabla_{z_t}(\ell(f_t(x_t)))=\g_t^\top \k_{\cD}(\cdot)$. Further note that  $\nabla R_\psi(z_t) = \nabla R(f_t)$ where $f_t \in \cH$ and $z_t \in \cH_\ast$. Also denote the Gram matrix of kernel evaluations as $\K_{\cD\cD}$, whose entries are $[\kappa(\mathbf{d}_n,\mathbf{d}_m)]_{n\times m=1}^{|\cD|,|\cD|}$. We compute the conditional expectation of $\g_t^\top \k_{\cD}(\cdot) $ and $\nabla R_{\psi}(z_t)$ as
	\begin{align}\label{eq:pseudo_grad_Poisson_Newton}
		\langle \nabla R_{\psi}(z_t), \mathbb{E}[\g_t^\top \k_{\cD}(\cdot) \vert \mathcal{F}_t] \rangle = \int_{\mathcal{X}} \int_{\mathcal{X}} \left(-1 + f_t(\mathbf{x}) \right) \k_{\cD}(\x)^\top \K_{\cD\cD} \k_{\cD}(\mathbf{y}) \left(-1 + f_t(\mathbf{y}) \right) d \mathbf{x} d \mathbf{y}\; ,
	\end{align}
	which can be lower bounded by $\zeta_{\min} \left\lVert -\mathbf{1} + f_t  \right\lVert ^2 \geq 0$ by the same logic as detailed in Example \ref{PPP_weight_updates}. This provides a basis for showing our second-order step also defines a pseudo-gradient. In particular, the functional gradient and expected value of Quasi--Newton pseudo-gradient [cf. \eqref{fn_up}] given the past sigma algebra which can be expressed as
	\begin{align} \label{pseudograd_2nd_order1_Newton}
		\langle \nabla R_{\psi}(z_t), \EE[\g_t^\top \mathbf{A}_{t+1}^{-1} \k_{\cD}(\cdot) | \cF_t] = \left\langle \nabla R_{\psi}(z_t), \EE\left[\frac{\g_t^\top \k_{\cD}(\cdot)}{\delta + \sum_{j=1}^t \g_j^\top \g_j} \vert \cF_t \right] \right\rangle ,
	\end{align}
	where the matrix simplification is similar to \cite{calandriello2017efficient} since $\mathbf{A}_{t+1}$ is positive definite. Now since the inner product between $\nabla R_{\psi}(z_t)$ and $\mathbb{E}[\g_t^\top \k_{\cD}(\cdot) \vert \mathcal{F}_t]$ is non-negative and $\delta + \sum_{j=1}^t \g_j^\top \g_j$ is positive, we may conclude the quantity in \eqref{pseudograd_2nd_order1_Newton} is non-negative. Therefore, it defines a pseudo-gradient.
\end{example}

\begin{algorithm}[t]
	\caption {Hybrid First and Second-order Online Iteration}
	\label{SPPPOT_Newton_step}
	\begin{algorithmic}[1]
		\REQUIRE{kernel $\kappa$, constant step-size $\eta$, regularizer $\delta$}
		\STATE \textbf{Initialize } Arbitrary small negative $\w_0 \in \Rn^{\abs{\cU}}$ and $\{\cD_0\} = \{\cU\}$.
		\STATE \textbf{Run: } Alg. \ref{SPFPMD_algorithm} until stable model order $M^\infty$ at step $t=T$.
		\STATE Denote $(\mathbf{w}_T,\mathcal{D}_T)$ as output of Alg. \ref{SPFPMD_algorithm} on step $T$.
		%
		\STATE \textbf{Initialize } Step size $\eta_T$ and $\mathbf{A}_0 = \delta \mathbf{I} \in \Rn^{M \times M}$.
		\FOR{$k = T, T+1, \ldots $}
		\STATE \textbf{\hspace{-2mm}Read: } sample $\x_k$
		\STATE \textbf{\hspace{-2mm}Run: } Algorithm \ref{Newton_step}
		\STATE \textbf{\hspace{-2mm} To obtain function estimate, invert mirror map:} 
		$$f_{t+1}(\x) = (\nabla\psi)^{-1}(z_t(\x)) = \exp(\w_{t+1} \k_{\cD_{t+1}}(\x))$$ for KL divergence.
		\ENDFOR
	\end{algorithmic}
\end{algorithm}

The overall procedure for incorporating Quasi--Newton updates with pseudo-gradients is presented in Algorithm \ref{Newton_step}. It is essentially similar to Algorithm \ref{SPFPMD_algorithm} with the additional outer product of gradients defining a recursively constructed pre-conditioner for the pseudo-gradient updates that are executed in the dual space. We also present a hybrid algorithm as a mixture of Algorithms \ref{SPFPMD_algorithm} and \ref{Newton_step}, which is presented in Algorithm \ref{SPPPOT_Newton_step}. This combination of SPPPOT and the Quasi--Newton Mirror Step allows one to flexibly learn the appropriate dictionary to obtain a functional estimate of reasonable precision, and then employ the learned dictionary to define the approximation subspace for executing Newton steps to obtain a function estimate of maximum precision. In this case, SPPPOT employs updates based upon functional differentiation of \eqref{obj_func_poisson3}, whereas the gradient for kernel online Newton mirror step is with respect to the weights of the auxiliary function [cf. \eqref{PPP_natural_grad}] over a fixed dictionary. Next, we establish convergence and complexity tradeoffs of Algorithms \ref{SPFPMD_algorithm} and \ref{Newton_step}.

%

\section{Convergence analysis}\label{sec:convergence}
We shift towards analyzing the convergence behavior of Algorithms \ref{SPFPMD_algorithm} and \ref{Newton_step} in terms of iteratively solving \eqref{ff}, when the compression budget $\epsilon$ and the step-size parameter $\eta$ are held constant for Algorithm \ref{SPFPMD_algorithm}. For Algorithm \ref{Newton_step}, we prove the convergence under constant step-size in terms of the fixed sub-space approximation for the Hessian. 
	A key point of departure in the analysis of Algorithm \ref{Newton_step} from prior works (e.g. \cite{calandriello2017efficient, calandriello2017second}) is the need to quantify the subspace approximation error associated with a fixed dictionary as compared to employing a fully infinite dimensional feature map $\k_{\cD_{\infty}}(\cdot)$ in the Quasi-Newton direction. Denote as $\g_{\infty,t}$ and $\mathbf{A}_{\infty,t+1}$, respectively,  as corresponding infinite dimensional gradient and doubly infinite dimensional Hessian at iterate $t$. Moreover, for notational simplicity, we denote the associated Quasi-Newton directions using a fixed subspace of points $\cD$ its infinite-dimensional counterpart as
\begin{align} \label{grad_notation}
	G_t = \g_{t}^\top \mathbf{A}_{t+1}^{-1} \k_{\cD}(\cdot) \;, \;\; G_t^{(\infty)} = \g_{\infty,t}^\top \mathbf{A}_{\infty,t+1}^{-1} \k_{\cD_{\infty}}(\cdot).
\end{align}
To proceed with our analysis, some technical conditions are required on the objective function and its gradients, as well as stochastic approximation errors, which we state next.

\begin{assumption}\label{pseudograd}
	The inner product between the gradient and the expectation of the pseudo-gradient \eqref{FPMD1} and its fixed-subspace Quasi--Newton variant \eqref{wt_up} conditioned on the filtration $\mathcal{F}_t=\sigma(\{\x_i\}_{i=1}^{t-1})$, are nonnegative.
	\begin{align}\label{pseudograd1}
		\!\!\!\!\langle \nabla R(f_t), \! \EE[g_t | \cF_t] \rangle \! \geq 0 \; , 
		\langle \nabla R_{\psi}(z_t),\! \EE[{\g_{\infty,t}}^\top \k_{\cD_{\infty}}(\cdot) | \cF_t] \rangle \! \geq\! 0 . 
	\end{align}
	Moreover for $f_t \neq f^\ast$, the product-moment is lower bounded by the second-moment of the gradient in the dual norm:
	\begin{align}
		&\!\!\!\mathbb{E}[\langle \nabla R(f_t), \EE[g_t | \cF_t] \rangle] \geq D \EE[||\nabla R(f_t)||_\ast^2] \label{pseudograd2} \; ,\\
		&\!\!\!\EE[\langle \nabla R_\psi(z_t), \EE[{\g_{\infty,t}}^\top \k_{\cD_{\infty}}(\cdot) | \cF_t] \rangle] \geq D \EE[||\nabla R_\psi(z_t)||_\ast^2] \label{pseudograd2_Newton} ,
	\end{align}
	where $D$ is a positive constant.
\end{assumption}
\begin{assumption}\label{fin}
	The optimizer of \eqref{ff} is finite $\ns{f^\star} \leq B$. 
\end{assumption}
\begin{assumption}\label{rcs}
	The function $R(\cdot)$ satisfies the Polyak--\L{ojasiewicz} (P--\L) condition
	\begin{align}
		\ns{\nabla R(f_t)}_{\ast} \geq 2 \lambda [R(f_t) - R(f^\ast)] \; ,
	\end{align}
	where $\lambda$ is a positive constant.
\end{assumption}
\begin{assumption}\label{dualdomainloss}
	The function $R_\psi(\cdot)$ which takes as inputs the dual functions $z=\nabla \psi(f)$ is $L_1$-smooth. 
\end{assumption}
\begin{assumption}\label{grad}
	The total expectation of the pseudo-gradient $g_t$ [cf. \eqref{FPMD1}] is upper-bounded in dual norm as
	\begin{align}\label{gradbound}
		\EE[\ns{g_t}_{\ast}] &\leq b^2 \!+ \!c^2 \EE [\ip{\nabla R(f)}{\EE[g_t | \cF_t]}]\; ,
	\end{align}
	whereas we assume the Quasi--Newton direction satisfies
	\begin{align}
			\EE[\ns{\g_{\infty,t}^\top \k_{\!\cD_{\!\infty}}\!(\cdot\!)}_{\ast}] \! &\leq \!\! [b^\prime ]^{2} \!\!+\! [c^\prime]^2 \EE [\ip{\nabla R_\psi\!(z_t\!)}{\EE[\g_{\infty, t}^\top \k_{\!\cD_{\!\infty}}\!(\cdot\!) | \cF_t]}] \label{gradbound_Newton} ,
	\end{align}
	for all $f \in \cH$, $t\in\Nn$, and some real constants $b, c$ and $b^\prime, c^\prime$. 
\end{assumption}


Assumption \ref{pseudograd} asserts that for $f \neq f^\ast$, an acute angle exists between the pseudo-gradient $g_t$ and actual gradient  $\nabla R(f)$ which is no larger than 90 degrees, and the constant $D$ determines the minimal degree of collinearity. Gradients for SPPPOT (end of Sec. \ref{sec:algorithm}) and Quasi--Newton Mirror Step [cf. \eqref{eq:pseudo_grad_Poisson_Newton}] satisfies this assumption. Note that any stochastic gradient is also a pseudo gradient. This condition is used in establishing decrement-like relationships in the second term on the right hand side of \eqref{basicbound_sync_simplified}. Assumption \ref{fin} is employed to ensure that $R(f_0) - R(f^\ast)$ is finite. The {P--\L} {condition} in Assumption \ref{rcs}, often presented as a consequence of strong convexity, holds for many non-convex functions (such as matrix completion \cite{wang2014optimal}), and connects decrease in the gradient norm to progress towards the optimal objective $R(f^*)$.
Assumption \ref{dualdomainloss} is standard in the analysis of mirror descent and proximal methods \cite{rockafellar2009variational}, and is used in the Bregman ``three-point inequality" \eqref{main1_sync}. Estimating the exact parameters of this condition is difficult, but may be conducted numerically. Doing so is beyond the scope of this work.
Assumption \ref{grad} (similar to \cite{yang2019learning}[Theorem 3]) is weaker than the standard second-moment boundedness condition on the stochastic gradient, and instead permits possibly \emph{unbounded} gradients under certain growth conditions (see \cite{bertsekas2000gradient}[Sec. 4], \cite{luo1993error}, and \cite{bottou2018optimization} eqn. (4.9) and Assumptions 4.3.). It holds in the RKHS setting when the dual norm of the universal kernel is bounded, and the optimizer belongs to a Sobolev space. The exact values of these constants are difficult to estimate, but we note that experimentally all data distributions under consideration had bounded noise, which translates to no arbitrarily large spikes in the pseudo-gradient.
Assumption \ref{grad} (eq. \eqref{gradbound}) is used in analyzing the direction associated with projected pseudo-gradients (Lemma \ref{qbound}). Next, we shift to presenting our convergence results regarding the first-order projected mirror descent in RKHS.
%

\subsection{Accuracy and Complexity Tradeoffs of Algorithm \ref{SPFPMD_algorithm}}\label{sec:first_order_convergence}
We proceed to establish the convergence of Algorithm \ref{SPFPMD_algorithm} under constant compression budget. To do so, we define the projected pseudo-gradient through the difference of the optimality condition reformulation of the pseudo-gradient \eqref{FPMD1} before and after projection. Specifically, KOMP-based projection is applied to the auxiliary sequence $\tz_{t+1}$ to obtain the dual function iterates $z_{t+1}$. Therefore, define the \emph{projected} pseudo-gradient $\gh_t$ as
\begin{align}\label{gh}
	\gh_t:= \frac{1}{\eta}(z_t-z_{t+1}) = \frac{1}{\eta}\left(\nabla \psi(f_t) - \nabla \psi(f_{t+1})\right) .
\end{align}
which simplifies the evolution of the projected iterates $z_t$ \eqref{FPMD2}:
\begin{align} \label{function_dual_update}
	z_{t+1} = z_t - \eta \gh_t .
\end{align}
Next we define some operations employed in the analysis: $\EE$ denotes the total expectation taken with respect to the unknown joint distribution $P(\x)$, and $\EE[g | \cF_t]$ denotes the expectation of $g$ conditioned on the filtration $\cF_t$. For brevity, we also define:
\begin{align}\label{deltat}
	\Ga_t := \EE{\ns{\nabla R(f_t)}_\ast}.
\end{align}
where $f^\star$ is defined in \eqref{ff}, $B_\psi$ is given in \eqref{bregsplit}, and $\norm{\cdot}_\ast$ denotes the RKHS dual-norm as defined in Sec. \ref{sec:problem}. We begin by bounding the directional error of the projected pseudo-gradient \eqref{gh} with respect to its un-projected variant \eqref{FPMD1}.
\begin{lemma}\label{qbound}
	The directional error of the projected pseudo-gradient \eqref{gh} relative to the pseudo-gradient [cf. \eqref{FPMD1}] as quantified by the RKHS dual-norm is bounded by the ratio of the compression budget to the step-size $\epsilon/\eta$. Moreover, it has bounded second-moment in the dual-norm. That is,
	\begin{align}
		&\norm{\gh_t-g_t}_\ast \leq \frac{\epsilon}{\eta} \label{approxerr} ,\\
		&\EE\ns{\gh_t}_\ast \leq 2\left( \! \left(\!\frac{\epsilon}{\eta}\right)^{\!\!2} + b^2 + c^2 \EE [\ip{\nabla R(f_t)}{\EE[g_t | \cF_t]}]\right) \label{projected_grad_bound} .
	\end{align}
\end{lemma}
With Lemma \ref{qbound} (proof in Appendix \ref{lemma1_proof}) established, we now focus on establishing the mean convergence of Algorithm \ref{SPFPMD_algorithm} under suitable selections of $\eta$ and $\epsilon$.
\begin{theorem}\label{thm_sync}
	Under Assumptions \ref{pseudograd}-\ref{grad}, upon running Algorithm \ref{SPFPMD_algorithm} for $t+1$ iterations, the objective  sub-optimality attenuates linearly up to a bounded neighborhood  when run with constant step-size $\eta<\min(\frac{1}{q_1},\frac{q_1}{q_2})$ and compression budget $\epsilon=\alpha \eta$, 
	\begin{align} \label{R_bound_thm_sync}
		\EE[R(f_{t+1})-R(f^\ast)] \leq (1-\rho)^t \EE[R(f_0)-R(f^\ast)] + \frac{1}{\rho} \left[ L_1 \eta^2 b^2 + \left(\frac{\eta \omega_1}{2} + L_1 \eta^2 \right) \alpha^2 \right] ,
	\end{align}
	where $\rho=q_1 \eta - q_2 \eta^2$ with $q_1 = 2\lambda \left(D - \frac{1}{2 \omega_1}\right)$ and $q_2 = 2\lambda D L_1 c^2$, $D$ are positive constants: $D$ is the correlation constant in Assumption \ref{pseudograd},  $\omega_1$ is a constant of Young's inequality (aka Peter--Paul inequality) \cite{royden1988real} such that $\omega_1>\frac{1}{2D}$,  $\alpha$ is the parsimony constant, $\lambda>0$ is the {P--\L} {inequality} constant of Assumption \ref{rcs}, and positive constant $c^2$ coming from Assumption \ref{grad}.
\end{theorem}

\begin{proof}
	This proof generalizes \cite{nemirovski2009robust}[cf. eqn. (2.38)] to the case of function sequences in RKHS with pseudo-gradients that are observed with deterministic errors. We begin by recalling the definition of $R_\psi$ for the convex function $\psi$ defined by Bregman divergence \eqref{bregsplit}, i.e., $R_\psi(z)=(R \circ \nabla\psi^\ast)(z)=R(\nabla\psi^\ast(z))$ for $z\in\cH_\ast$. Consider this quantity evaluated at auxiliary functions $z_{t+1}=\nabla\psi(f_{t+1})$ and $z_t=\nabla\psi(f_{t})$, and apply Assumption \ref{dualdomainloss} regarding its Lipschitz continuity:
	\begin{align}
		R_\psi(\nabla\psi(f_{t+1})) - R_\psi(\nabla\psi(f_t)) - \langle \nabla R_\psi(\nabla\psi(f_t)),\nabla\psi(f_{t+1})-\nabla\psi(f_t) \rangle  \leq \frac{L_1}{2} ||\nabla\psi(f_{t+1})-\nabla\psi(f_t)||_{\ast}^2 \label{main1_sync} .
	\end{align}
	Now, consider the expression for the projected pseudo-gradient in \eqref{gh}, which may be rearranged to obtain $\nabla\psi(f_{t+1})-\nabla\psi(f_t)=-\eta \gh_t$. Taken together with the fact that $R_\psi(\nabla\psi(f_t))=R(f_t)$ since $\nabla \psi^\ast = (\nabla \psi)^{-1}$ one may rewrite \eqref{main1_sync} as an approximate descent relation as
	\begin{align}
		&R(f_{t+1}) \leq R(f_t) - \langle \nabla R(f_t), \eta \gh_t \rangle + \frac{L_1 \eta^2}{2} ||\gh_t||_{\ast}^2 \nonumber\\
		&\qquad\quad \ = R(f_t) - \eta \langle \nabla R(f_t), g_t \rangle + \eta \langle \nabla R(f_t), g_t - \gh_t \rangle + \frac{L_1 \eta^2}{2} ||\gh_t||_{\ast}^2 \label{basicbound_sync} .
	\end{align}
	The right-hand side of this expression decomposes into four terms: (i) the objective at the previous step, (ii) the product-moment between the gradient and the pseudo-gradient,  (iii) the directional error associated with pseudo-gradient projections, (iv) a second-moment of the projected pseudo-gradient.

	%
	%
	%
	We proceed to the third term of the right-hand side of \eqref{basicbound_sync}. Via Young's inequality with constant $\omega_1$ and the KOMP stopping criteria \eqref{approxerr}, its expected value is upper-estimated as
	\begin{align}\label{term2_sync}
		\eta \EE \left[ \nabla R(f_t), g_t - \gh_t \right] \leq \frac{\eta}{2 \omega_1} \Ga_t + \frac{\eta \omega_1}{2} \left(\frac{\epsilon}{\eta}\right)^2 , 
	\end{align}
	where we employ the short-hand notation \eqref{deltat} for $\Ga_t$. (iv) Observe that the expected value of the fourth term of the right-hand side of \eqref{basicbound_sync} can be bounded as \eqref{projected_grad_bound} (Lemma \ref{qbound}). Now, we substitute the right-hand sides of \eqref{term2_sync}, \eqref{projected_grad_bound} and gather terms:
	\begin{align}\label{basicbound_sync_simplified}
		&\!\!\EE[R(f_{t+1})] \leq \EE[R(f_t)] -  (\eta \!-\! L_1 c^2 \eta^2) \EE [\ip{\nabla R(f_t)}{\EE[g_t | \cF_t]}] + \left(\frac{\eta}{2 \omega_1} - D \eta\right) \Ga_t + \left(\frac{\eta \omega_1}{2} + L_1 \eta^2 \right) \left(\frac{\epsilon}{\eta}\right)^{2} + L_1 \eta^2 b^2 .
	\end{align}
	%
	Proceed by subtracting $\EE[R(f^\ast)]$ from both sides, and then apply \eqref{pseudograd2} (Assumption \ref{pseudograd}) to the gradient-pseudo-gradient inner-product term to obtain the expected decrement relation:
	\begin{align}\label{basicbound_sync2}
		&\EE[R(f_{t+1})-R(f^\ast)] \leq \EE[R(f_t)-R(f^\ast)] -  \left(D \eta - \frac{\eta}{2 \omega_1} - D L_1 c^2 \eta^2 \right) \Ga_t + \left(\frac{\eta \omega_1}{2} + L_1 \eta^2 \right) \alpha^2 + L_1 \eta^2 b^2 ,
	\end{align}
	where we have substituted in the choice of compression budget $\epsilon=\alpha \eta$ for some scalar $\alpha>0$ in order to simplify fractions of the compression budget to the step-size that appear in the preceding expressions. Now we use the {P--\L} {condition} as stated in Assumption \ref{rcs} to the term $\Ga_t$ [cf. \eqref{deltat}] in \eqref{basicbound_sync2} as follows:
	\begin{align}
		\EE[R(f_{t+1})-R(f^\ast)] \leq (1-\rho) \EE[R(f_t)-R(f^\ast)] + L_1 \eta^2 b^2 + \left(\frac{\eta \omega_1}{2} + L_1 \eta^2 \right) \alpha^2 \label{basicbound_sync3} ,
	\end{align}
	where constant  $\rho = q_1 \eta - q_2 \eta^2$ determines the transient rate of convergence, with $q_1 = 2\lambda \left(D - \frac{1}{2 \omega_1}\right)$ and $q_2 = 2\lambda D L_1 c^2$. Next, we obtain conditions on constants $D$ and $\omega_1$ such that:
	\begin{align}\label{rho_condition}
		0 \leq \rho \leq 1 .
	\end{align}
	The constant $q_1$ is required to be positive to satisfy \eqref{rho_condition}. Imposing this constraint implies that the Peter Paul inequality constant $\omega_1$ in \eqref{term2_sync} satisfies $\omega_1 > \frac{1}{2 D}$, where $D$ is the correlation constant in \eqref{pseudograd2}.
	These conditions together imply
	\begin{align}\label{eta2ndcondn_sync}
		\eta < \frac{1}{q_1 \left(1-\frac{q_2}{q_1}\eta\right)} .
	\end{align}
	Now if $\rho \leq 1$, and if $\left(1-\frac{q_2}{q_1}\eta\right) \leq 1$, then $\eta < \frac{1}{q_1}$ implies \eqref{eta2ndcondn_sync} holds. On the contrary, $\rho \geq 0$ implies $\eta<\frac{q_1}{q_2}$. Overall, then, we obtain the valid step-size range as $\eta<\min(\frac{1}{q_1},\frac{q_1}{q_2})$.
	
	Returning to \eqref{basicbound_sync3}, we iteratively break down the right-hand together with the fact that $\sum\limits_{i=0}^t (1-\rho)^i \leq \sum\limits_{i=0}^\infty (1-\rho)^i = \frac{1}{\rho}$ to obtain \eqref{R_bound_thm_sync}, which concludes the proof.
\end{proof}

Theorem \ref{thm_sync} characterizes the trade-off between the rate of the convergence and the radius of the ball to which the iterates converge at steady state. First note that regardless of the choice of $\eta$ and $t$, the mean distance from the optimum will always be $\O(\alpha^2)$ in the worst case. The bound in \eqref{R_bound_thm_sync} is for $\epsilon>0$, which causes the additional $\alpha^2$ to appear. For $\epsilon=0$, the $\alpha^2$ term of \eqref{R_bound_thm_sync} vanishes and hence simplifies to $\O(\eta)$ asymptotically since $\rho$ is approximately of order $\eta$ for $\eta < 1$.

Further observe that the second term on the right-hand side of \eqref{R_bound_thm_sync} simplifies to $\O(\eta b^2 + (1+\eta)\alpha^2)$, which is in accordance with the convergence rates of the iterates of stochastic mirror descent for vector-valued problems \cite{nemirovski2009robust,doan2018convergence}. Theorem \ref{thm_sync} is a generalization to the RKHS setting, where we additionally require the range of functions to be nonnegative. Moreover, we explicitly characterize the error in the convergence behavior incurred due to subspace projections of the auxilliary sequence $z_t$ [cf. \eqref{eq:main_alg}]. Relative to \cite{yang2019learning}[Theorem 6], our convergence result holds under comparable conditions, but incorporates the additional aspect of trading off parameterization efficiency and convergence accuracy due to projections. Thus, for $\epsilon=0$ our result simplifies to the aforementioned result, but requires a parameterization that grows unbounded with $t$ [cf. \eqref{eq:pseudo_gradient_parametric_recursion}]. \medskip


{\noindent \bf Parameterization efficiency.}	We analyze the complexity of the function parameterization associated with Algorithm \ref{SPFPMD_algorithm} when employing  sparse projections defined by KOMP. To do so, we require additional two conditions.

\begin{assumption}\label{C_lip}
	The pseudo-gradient may be written in the form $g_t = g_t^\prime \kappa(\x_t,\cdot)$ with scalar $g_t^\prime$  bounded by constant $C$
	\begin{align}
		\vert g_t^\prime \vert \leq C .
	\end{align}
\end{assumption}

\begin{assumption}\label{feature}
	The feature space $\mathcal{X}$ is compact.
\end{assumption}

Assumption  \ref{C_lip}  implies that either the objective is differentiable or there exists a suitable kernel embedding such that the chain rule is applicable. Moreover, Assumption \ref{feature} ensures that the data domain has finite covering number \cite{anthony2009neural}. Under these conditions, we establish that the model order of the function parameterization defined by Algorithm \ref{SPFPMD_algorithm} is finite via analogous logic to \cite{koppel2019parsimonious}[Theorem 4], stated here as a corollary.
\begin{corollary}\label{KOMP_model_order_thm}
	Denote as $M_t$ the model order of the function $z_t$ obtained from running Algorithm \ref{SPFPMD_algorithm} with fixed compression budget $\epsilon>0$. Then we have that $M_t \leq M^\infty$, where $M^\infty$ is the maximum model order upper-estimated as 
	\begin{align}\label{max_model_order_possible}
		M^\infty \leq \O\left(\frac{1}{\epsilon} \right)^d .
	\end{align}
\end{corollary}


A detailed proof is given in Appendix \ref{model_conv_proof}. We note that in contrast to \cite{koppel2019parsimonious}, compactness is only required to establish the finite model order property in Corollary \ref{KOMP_model_order_thm}, but not for Theorem \ref{thm_sync}. By contrast, Assumption \ref{feature} is directly required for convergence in \cite{koppel2019parsimonious}. Moreover, \cite{koppel2019parsimonious} necessitates the use of exact stochastic gradients, and therefore is inapplicable to the primary focus of this work, which is the interpolation of the intensity parameter in inhomogeneous Poisson Processes.

\subsection{Convergence in Mean of Algorithm \ref{Newton_step}}\label{subsec:second_order_convergence}
To analyze Algorithm \ref{Newton_step}, we require similar conditions employed in the previously conducted analysis for the first-order scheme. We impose Assumptions \ref{pseudograd} [cf. \eqref{pseudograd2_Newton}], \ref{fin}, \ref{rcs}, \ref{dualdomainloss} and \ref{grad} [cf. \eqref{gradbound_Newton}] and also introduce assumptions on maximum eigen-value $\mu_{t+1}^{\max}$ of the Hessian matrix $\mathbf{A}_{\infty,t+1}$ and on the second moment of norm difference between finite and infinite dictionary Quasi--Newton gradients.
\begin{assumption} \label{Hessian_up_bound}
	For finite $t$, the maximum eigenvalue $\mu_{t+1}^{\max}$ of the Hessian estimate $\mathbf{A}_{\infty,t+1}$ [cf. \eqref{hess_update}] is finite.
\end{assumption}
\begin{assumption}\label{infinite_dim_grad_bound}
	The mean-square error between finite and infinite dictionary Quasi--Newton directions satisfies the growth condition:
	\begin{align}
		\EE [\ns{G_t^{(\infty)} - G_t}_{\ast}] \leq \nu_{\mathcal{D}}^2 \EE[\ns{G_t^{(\infty)}}_{\ast}],
	\end{align}
	for $t \in \Nn$ and some real positive constant $\nu_{\cD}$ that exponentially decays as the dictionary subspace size $|\mathcal{D}|$ increases.
\end{assumption}
Observe that the analysis we conduct is for the Quasi--Newton algorithm developed in Section \ref{sec:algorithm_Newton}, which employs finite-dimensional truncations. Its un-truncated counterpart defines an update direction, which we use as a barometer of the quality of our truncation through Assumptions \ref{Hessian_up_bound} and \ref{infinite_dim_grad_bound}. In particular, these conditions quantify the error due to truncation, and allow us to express algorithm performance in terms of the objective at the true optimizer in the RKHS, rather than any truncated optimizer.

Assumption \ref{Hessian_up_bound} enforces that for any finite $t$, the doubly infinite dimensional Hessian matrix formed by the iterative update \eqref{hess_update} with $\g_{\infty,t}$ does not blow up to infinity in terms of its maximum eigenvalue $\mu_{t+1}^{\max}$, which determines its operator norm. As noted in \cite{yang2019learning} for the case of diminishing step-size, i.e. when $f_t \rightarrow f^\ast$, the inner product between the true gradient and conditional expectation of the first order pseudo-gradient becomes orthogonal. Analogous reasoning is applicable to the approximate Newton step. Assumption \ref{infinite_dim_grad_bound} generalizes the eigenvalue decay conditios that are typically imposed on the solution path of kernel ridge regression, i.e., that the eigenvalues of the kernel matrix form a summable series \cite{guo2017learning,lin2017distributed}, which holds when one restricts the RKHS to a Sobolov space. Here we impose a related condition, which is that the tails of this series are a finite quantifiable number satisfying $\nu_{\cD} \propto \exp(- \tilde{a} |\mathcal{D}| )$, where $\tilde{a}$ is a positive constant.

We begin by employing Assumption \ref{Hessian_up_bound} to characterize the upper and lower-boundedness of Hessian approximations \eqref{hess_update}.
%
\begin{lemma}\label{lemma_HS_bound}
	The inverse of the Hessian matrix generated through the Quasi--Newton updates is bounded as
	\begin{align} \label{HS_bound_eq}
		\frac{1}{\mu_{t+1}^{\max}} \mathbf{I} \preccurlyeq \mathbf{A}_{\infty,t+1}^{-1} \preccurlyeq \frac{1}{\delta} \mathbf{I} .
	\end{align}
\end{lemma}
Proof of Lemma \ref{lemma_HS_bound} is given in Appendix \ref{lemma_HS_bound_proof}. We now focus on the convergence analysis of Algorithm \ref{Newton_step}. For brevity, here also we use $\Ga_t := \EE{\ns{\nabla R(f_t)}_\ast}$ as detailed in \eqref{deltat}.

\begin{theorem}\label{thm_Newton}
	{\bf (i)} Under Assumptions \ref{pseudograd}-\ref{grad}, \ref{Hessian_up_bound} and \ref{infinite_dim_grad_bound} and for finite $t$, upon selecting a \emph{constant step-size} $\eta < \min \left(\frac{1}{2 \lambda q_3}, \frac{q_3}{q_4} \right)$ and regularizer $\delta>\frac{c^\prime \nu_{\cD} \mu_{t+1}^{\max}}{D}$, the average sub-optimality satisfies 
	\begin{align} \label{Newton_const_step_conv}
		\EE[R(f_{t+1}) - R(f^\ast)] \leq (1-\rho)^t \EE[R(f_t) - R(f^\ast)] + \frac{{b^\prime}^2 L_1 (1 + \nu_{\cD}^2)}{\rho \delta^2} \eta^2  + \frac{{b^\prime}^2 \nu_{\cD}^2 \omega_2}{2 \rho \delta^2} \eta ,
	\end{align}
	where $\rho_t = 2\lambda q_3 (\eta_t - \frac{q_4}{q_3} \eta_t^2)$ with $q_3 = \frac{D}{\mu_{t+1}^{\max}} - \frac{D{c^\prime}^2 \nu_{\cD}^2 \omega_2}{2 \delta^2} - \frac{1}{2 \omega_2}$ and $q_4 = \frac{D{c^\prime}^2}{\delta^2} L_1 (1 + \nu_{\cD}^2)$. The positive constants $D$, $\lambda$, $L_1$, ${b^\prime}^2$, ${c^\prime}^2$, $\mu_{t+1}^{\max}$ and $\nu_{\cD}^2$ are defined in Assumptions \ref{pseudograd}, \ref{rcs}, \ref{dualdomainloss}, \ref{grad}, \ref{Hessian_up_bound} and \ref{infinite_dim_grad_bound} respectively, and the choice of Young's inequality constant $\omega_2$ is given in \eqref{omega_2_select}.\\
	{\bf (ii)} Under the same conditions, if we select a \emph{diminishing step-size} $\eta_t = \min\left( \frac{q_3}{2 q_4}, \frac{(2t+1)}{\lambda q_3 (t+1)^2} \right)$, the average sub-optimality converges as 
	\begin{align} \label{Newton_dimi_step_conv}
		\EE[R(f_{t+1}) - R(f^\ast)] \leq \frac{4 L_1 {b^\prime}^2 (1+\nu_{\cD}^2)}{\delta^2 \lambda^2 q_3^2 (t+1)} + \frac{{b^\prime}^2 \nu_{\cD}^2 \omega_2}{2 \lambda \delta^2 q_3} .
	\end{align}
\end{theorem}

\begin{proof}
	The proof follows analogous logic to that which yields Theorem \ref{thm_sync}. Begin by employing Assumption \ref{dualdomainloss} to write
	\begin{align}
		R_\psi(z_{t+1}) \!-\! R_\psi(z_t) \!-\! \langle \nabla R_\psi(z_t),z_{t+1} \!-\! z_t \rangle \!\leq\! \frac{L_1}{2} \!\ns{z_{t+1} \!-\! z_t}_{\ast} .
	\end{align}
	Now, substitute \eqref{fn_up} in the definition of $z_{t+1}$ associated with the Quasi--Newton update and add and subtract the infinite dimensional second order gradient $G_t^{(\infty)}$ [cf \eqref{grad_notation}] with $z_{t+1}$ term to write
	\begin{align} \label{Nw_main1}
		R_\psi(z_{t+1}) - R_\psi(z_t) \leq - \eta_t \langle \nabla R_\psi(z_t), G_t^{(\infty)} \rangle + L_1 \eta_t^2 \ns{G_t^{(\infty)}}_\ast + \eta_t \langle \nabla R_\psi(z_t),  G_t^{(\infty)} - G_t \rangle +  L_1 \eta_t^2 \ns{G_t^{(\infty)} - G_t}_\ast ,
	\end{align}
	where we have used the inequality $(a+b)^2 \leq 2a^2 + 2b^2$ to split the last term. We first focus on the expectation of the third term on the right hand side of \eqref{Nw_main1} which can be bounded using Young's inequality of constant $\omega_2$ as
	\begin{align}\label{peter_paul_ineq}
			\eta_t \EE \langle \nabla R_\psi(z_t), G_t^{(\infty)} \!-\! G_t \rangle \!\leq\! \frac{\eta_t}{2 \omega_2} \Ga_t \!+\! \frac{\eta_t \omega_2}{2} \EE \ns{G_t^{(\infty)} \!-\! G_t}_\ast ,
	\end{align}
	where $\Ga_t$ is given in \eqref{deltat} and noting that $\nabla R_\psi(z_t) = \nabla R(f_t)$. The last terms of \eqref{peter_paul_ineq} and \eqref{Nw_main1} can be grouped together and use of assumption \ref{infinite_dim_grad_bound} further simplifies expectation of \eqref{Nw_main1} as
	\begin{align}
		\EE [R_\psi(z_{t+1}) - R_\psi(z_t)] \leq - \eta_t \EE \langle \nabla R_\psi(z_t), \EE[G_t^{(\infty)} | \cF_t] \rangle  + \frac{\eta_t}{2 \omega_2} \Ga_t + \left(\frac{\eta_t \nu_{\cD}^2 \omega_2}{2}+L_1 (1+\nu_{\cD}^2) \eta_t^2 \right) \EE \ns{G_t^{(\infty)}}_\ast . \label{simplified1}
	\end{align} 
	Next focusing on the third term on the right hand side of \eqref{simplified1}, the term $\ns{G_t^{(\infty)}}_\ast$ [cf \eqref{grad_notation}] can be factorized as
	\begin{align} \label{Nw_term2_split}
		\!\!\!\ns{G_t^{(\infty)}}_\ast = \k_{\cD_\infty}(\cdot)^\top \mathbf{A}_{\infty,t+1}^{-1} \g_{\infty,t} \g_{\infty,t}^\top \mathbf{A}_{\infty,t+1}^{-1} \k_{\cD_\infty}(\cdot) .
	\end{align}
	Using the definition $\g_{\infty,t} \g_{\infty,t}^\top = [\g_{\infty,t}^\prime]^2 \k_{\cD_{\infty}}(\x_t) \k_{\cD_{\infty}}(\x_t)^\top$, which is non-negative, along with using the upper-bound on the Hessian inverse estimate $\mathbf{A}_{t+1}^{-1}$ in Lemma \ref{lemma_HS_bound}, we may simplify the right-hand side of \eqref{Nw_term2_split} as $\|G_t^{(\infty)}\|_{\ast}^2 \leq \frac{1}{\delta^2} \k_{\cD_{\infty}}(\cdot)^{\top} \g_{\infty,t} \g_{\infty,t}^\top \k_{\cD_{\infty}}(\cdot) = \frac{1}{\delta^2} \ns{\g_{\infty,t}^\top \k_{\cD_\infty}(\cdot)}_\ast$.
	%
	%
	Using Assumption \ref{grad}, its expected value can be bounded as
	\begin{align} \label{Nw_term2_split3}
		\!\!\EE\!\ns{G_t^{(\infty)}}_\ast \!\leq \frac{{b^\prime}^2}{\delta^2} \!+\! \frac{{c^\prime}^2}{\delta^2} \EE [\ip{\nabla R_\psi(z_t\!)}{\EE[\g_{\infty,t}^\top \k_{\cD_\infty}\!(\cdot\!) | \cF_t]}] .\!
	\end{align}
	Now, let us return to the first-term on the right hand side of \eqref{simplified1}. Apply Lemma \ref{lemma_HS_bound} to replace $\mathbf{A}_{t+1}^{-1}$ by its upper-bound in terms of identity and $\mu_{t+1}^{\max}$:
	\begin{align}\label{Nw_term1_upbound}
		-\eta_t \EE \langle \nabla R_\psi(z_t), \EE[G_t^{(\infty)} | \cF_t] \rangle \leq - \frac{\eta_t}{\mu_{t+1}^{\max}} \EE \langle \nabla R_\psi(z_t), \EE [\g_{\infty,t}^\top \k_{\cD_\infty}(\cdot) | \cF_t] \rangle .
	\end{align}
	So now, via \eqref{Nw_term2_split3} and \eqref{Nw_term1_upbound}, the expected value of the objective difference across time in \eqref{simplified1} is bounded as
	\begin{align}
		\EE [R_\psi(z_{t+1}) - R_\psi(z_t)] \leq & \frac{\eta_t}{2 \omega_2} \Ga_t + \frac{{b^\prime}^2}{\delta^2} \left[ \frac{\eta_t \nu_{\cD}^2 \omega_2}{2} + L_1 (1 + \nu_{\cD}^2) \eta_t^2 \right] \nonumber\\
		& - \left[\frac{\eta_t}{\mu_{t+1}^{\max}} - \frac{{c^\prime}^2}{\delta^2} \left(\frac{\eta_t \nu_{\cD}^2 \omega_2}{2} + L_1 (1 + \nu_{\cD}^2) \eta_t^2 \right) \right] \times \EE \langle \nabla R_\psi(z_t), \EE [\g_{\infty,t}^\top \k_{\cD_\infty}(\cdot) | \cF_t] \rangle \label{Nw_main2}
	\end{align}
	The third term on the right-hand side of \eqref{Nw_main2} can be bounded via the pseudo-gradient inequality [cf.  \eqref{pseudograd2_Newton}] of Assumption \ref{pseudograd}, after which we gather like terms. Then, by substituting $\nabla R_{\psi}(z_t) = \nabla R(f_t)$ and subtracting $R(f^\ast)$ on both sides of \eqref{Nw_main2}, and using \eqref{deltat}, we may write
	\begin{align} \label{Nw_main3}
		\EE[R(f_{t+1}) - R(f^\ast)] \leq \EE[R(f_t) - R(f^\ast)] - (q_3 \eta_t - q_4 \eta_t^2) \Ga_t + \frac{{b^\prime}^2}{\delta^2} \left[\frac{\eta_t \nu_{\cD}^2 \omega_2}{2} + L_1 (1 + \nu_{\cD}^2) \eta_t^2 \right] ,
	\end{align}
	where $q_3 = \frac{D}{\mu_{t+1}^{\max}} - \frac{D{c^\prime}^2 \nu_{\cD}^2 \omega_2}{2 \delta^2} - \frac{1}{2 \omega_2}$ and $q_4 = \frac{D{c^\prime}^2}{\delta^2} L_1 (1 + \nu_{\cD}^2)$. The quantity $q_3$ has to be positive, which yields a quadratic in terms of $\omega_2$. Via the quadratic formula, we have:
	\begin{align} \label{omega_2_select}
		\frac{\frac{D}{\mu_{t+1}^{\max}} \!-\! \sqrt{\frac{D^2}{(\mu_{t+1}^{\max})^2} \!-\! \frac{{c^\prime}^2 \nu_{\cD}^2}{\delta^2}}}{\frac{{c^\prime}^2 \nu_{\cD}^2}{\delta^2}} < \omega_2 < \frac{\frac{D}{\mu_{t+1}^{\max}} \!+\! \sqrt{\frac{D^2}{(\mu_{t+1}^{\max})^2} \!-\! \frac{{c^\prime}^2 \nu_{\cD}^2}{\delta^2}}}{\frac{{c^\prime}^2 \nu_{\cD}^2}{\delta^2}} .
	\end{align}
	The lower bound on $\omega_2$ is always positive and from the discriminant term we get the regularizer selection as $\delta \geq \frac{c^\prime \nu_{\cD} \mu_{t+1}^{\max}}{D}$. Assumption \ref{rcs} applied to the second term of the right-hand side of \eqref{Nw_main3} followed by grouping terms yields
	\begin{align} \label{Nw_main4}
		\EE[R(f_{t+1}) &- R(f^\ast)] \leq (1-\rho_t) \EE[R(f_t) - R(f^\ast)] + \frac{{b^\prime}^2 L_1 (1 + \nu_{\cD}^2)}{\delta^2} \eta_t^2  + \frac{{b^\prime}^2 \nu_{\cD}^2 \omega_2}{2 \delta^2} \eta_t ,
	\end{align}
	where $\rho_t = 2\lambda q_3 ( \eta_t - \frac{q_4}{q_3} \eta_t^2)$ with $q_3$ and $q_4$ defined after \eqref{Nw_main3}.\\
	\noindent \textbf{Constant step-size:} Set the step-size as constant $\eta_t = \eta$ and hence $\rho_t$ is time-invariant: $\rho_t = \rho$. To ensure \eqref{Nw_main4} defines a contraction up to an error bound same condition as \eqref{rho_condition} shall be satisfied for $\rho$. Following the similar analysis for Algorithm \ref{SPFPMD_algorithm} step-size selection criterion in \eqref{basicbound_sync3} - \eqref{eta2ndcondn_sync}, here we get the condition on step-size as 
	\begin{align}\label{eta3rdcondn_sync}
		\eta < \min \left(\frac{1}{2 \lambda q_3}, \frac{q_3}{q_4} \right) .
	\end{align}
	Note that the maximum eigen value $\mu_{t+1}^{\max}$ (inside $q_3$) of $\mathbf{A}_{\infty,t+1}$, obtained by addition of rank-1 positive outer products of pseudo-gradient vectors $\tilde{g}_{\infty,t}$, is the maximum eigenvalue of Hessian estimates up to the present time $t$, and increases slowly as the iterate progresses. Hence our convergence results pertain to finite $t$.
	Returning to \eqref{Nw_main4} with the step-size selection \eqref{eta3rdcondn_sync}, we may simplify via recursive substitution of the first term on the right-hand side of \eqref{Nw_main4} yielding \eqref{Newton_const_step_conv}. We know that $\rho = \O (\eta)$. Hence the mean distance from optima depends on $\O(\eta) + \O(\frac{\nu_{\cD}^2 \omega_2}{\delta^2})$, where the last term with fixed bias came as a result of approximating the infinite dimensional dictionary in terms of finite dimension $\cD$.

	\noindent \textbf{Diminishing step-size:} For this case, we choose $\eta_t = \min\left( \frac{q_3}{2 q_4}, \frac{(2t+1)}{\lambda q_3 (t+1)^2} \right)$. For such a step-size choice, the factor $\eta_t - \frac{q_4}{q_3} \eta_t^2$ in $\rho_t$ of \eqref{Nw_main4} can be lower bounded by
	\begin{align}
		\eta_t - \frac{q_4}{q_3} \eta_t^2 \geq \eta_t - \frac{q_4}{q_3} \times \frac{q_3}{2 q_4} \eta_t = \frac{\eta_t}{2} . \label{lb}
	\end{align}
	Hence the lower bound of \eqref{lb} allows us to lower bound $\rho_t$ with $\lambda q_3 \eta_t$. Using this lower bound on $\rho_t$ finally in \eqref{Nw_main4} and using the fact that $\eta_t = \frac{(2t+1)}{\lambda q_3 (t+1)^2}$ for large $t$, we get finally
	\begin{align} \label{dimi_step}
		\EE[R(f_{t+1}) - R(f^\ast)] \leq \frac{t^2}{(t+1)^2} \EE[R(f_t) - R(f^\ast)] + \frac{L_1 {b^\prime}^2 (1+\nu_{\cD}^2) (2t+1)^2}{\delta^2 \lambda^2 q_3^2 (t+1)^4} + \frac{{b^\prime}^2 \nu_{\cD}^2 \omega_2 (2t+1)}{2 \lambda \delta^2 q_3 (t+1)^2} .
	\end{align}
	Multiplying both sides of \eqref{dimi_step} with $(t+1)^2$ and taking $\beta(t) = t^2 \EE[R(f_t) - R(f^\ast)]$ and using the fact that $\frac{2t+1}{t+1} < 2$ for all $t$, \eqref{dimi_step} simplifies as
	\begin{align}
		\beta(t+1) - \beta(t) \leq \frac{4 L_1 {b^\prime}^2 (1+\nu_{\cD}^2)}{\delta^2 \lambda^2 q_3^2} + \frac{{b^\prime}^2 \nu_{\cD}^2 \omega_2 (2t+1)}{2 \lambda \delta^2 q_3} . \label{dimi_step3}
	\end{align}
	Doing a telescopic sum of \eqref{dimi_step3} from $t=0$ to $t$ and noting that $\beta(0)=0$ yields
	\begin{align}
		(t+1)^2 \EE[R(f_{t+1}) - R(f^\ast)] \leq \frac{4 L_1 {b^\prime}^2 (1+\nu_{\cD}^2) (t+1)}{\delta^2 \lambda^2 q_3^2} + \frac{{b^\prime}^2 \nu_{\cD}^2 \omega_2 (t+1)^2}{2 \lambda \delta^2 q_3} . \label{dimi_step4}
	\end{align}
	Dividing both sides of \eqref{dimi_step4} by $(t+1)^2$ yields \eqref{Newton_dimi_step_conv}.
\end{proof}

The convergence result of \eqref{Newton_const_step_conv} regarding Algorithm \ref{Newton_step} is similar to its first-order counterpart, but deals with the additional technicalities associated with the second-order information incorporated into the pseudo-gradients as preconditioning. Moreover, due to the use of a fixed subspace size $|\mathcal{D}|$, there is no need to discuss projection-induced errors in update directions. In addition, for Algorithm \ref{Newton_step}, we achieve a diminishing rate of convergence of order $\O(\frac{1+\nu_{\cD}^2}{\delta^2 q_3^2 (t+1)})$ plus an additional bias $\O(\frac{\nu_{\cD}^2 \omega_2}{\delta^2 q_3})$ due to approximating the subspace with finite dictionary $\cD$, for diminishing step-size selection as given in \eqref{Newton_dimi_step_conv}. We acknowledge that faster local rate analyses of Quasi--Newton methods exist (such as superlinear \cite{rodomanov2021new} and quadratic \cite{NoceWrig06}[Ch. 6]), but our priority was global rate analysis and analyzing the evolution in the RKHS. Thus, we defer sharpening the rate analysis to future work.


%

	\section{Experiments}\label{sec:experiments}
	We shift to experimental validation of Algorithms \ref{SPFPMD_algorithm}, \ref{Newton_step} and \ref{SPPPOT_Newton_step} as compared with state of the art second order offline and first order online benchmarks for normalized intensity estimation in inhomogeneous Poisson Point Processes (PPP, Example \ref{eg:poisson}). Throughout, we select as our pseudo-gradient the kernel embedding -- see Secs. \ref{sec:algorithm} and \ref{sec:algorithm_Newton}.
	
	We employ KOMP in SPPPOT with both constant (same as \cite{koppel2019parsimonious}) and adaptive budget selection. Constant budgeting requires trial and error to yield a stabilizing model order, whereas for adaptive budget, we tune $\alpha$ defined by budget specification $\epsilon=\alpha\eta$ at time $t$ as 
	\begin{align*}
		\alpha(t+1)\!=\!\alpha(t)\!\left[1\!+\!\max \left( \min \left( (M_t
		\!-\! d_{mo}\!)\!\times\! 0.1\%,0.1 \!\right)\!, \!-\!0.1 \!\right) \right] ,
	\end{align*}
	where $\alpha(t+1)$ has an enforced upper and lower bound to ensure it remains nonzero even when near the target dictionary size $d_{mo}$. The intuition is that this constant increases when the model order $M_t$ starts increasing from $d_{mo}$ so that we allow a bit higher error tolerance budget $\epsilon_{t+1} = \alpha(t+1) \eta$, that results in reduction of $M_{t+1}$. The contrary is true when $M_t<d_{mo}$. 
%

%

\begin{table*}[t]\centering
	%
	\begin{tabular}{|P{1.5cm}|P{0.6cm}|P{0.32cm}|P{0.7cm}|P{2.8cm}|P{2cm}|P{2.6cm}|P{0.2cm}|P{0.4cm}|P{0.7cm}|P{1.3cm}|}
		\hline
		Algorithm & Epochs & a & Kernel B.W. & $\eta$ & Mini-Batch Size & Budget & $\delta$ & $\abs{\cU}$ & Model order & Time\\
		\hline
		\multicolumn{11}{|c|}{ \vspace{-2mm}} \\
		\multicolumn{11}{|c|}{\bf 1-D Gaussian toy example} \\
		\hline
		BFGS & 1 & 0.15 & 0.0065 & \textendash & \textendash & \textendash & \textendash & 100 & 10211 & $\sim$ 8 hrs \\
		\hline
		PMD & 8 & \textendash & 0.0065 & 0.05 & 30 & \textendash & \textendash & 100 & 100 & $\sim$ 5 mins\\
		\hline
		POLK & 8 & \textendash & 0.0065 & 0.006 & 30 & $\epsilon = 3.5 \times 10^{-8}$ & \textendash & 100 & 179 & $\sim$ 5 mins\\
		\hline
		Algorithm \ref{SPFPMD_algorithm} ($\epsilon$ fixed)& 8 & \textendash & 0.0065 & 0.012 & 30 & $\epsilon = 6.6 \times 10^{-6}$ & \textendash & 100 & 100 & $\sim$ 5 mins\\
		\hline
		Algorithm \ref{SPFPMD_algorithm} ($\epsilon_t$ adaptive)& 8 & \textendash & 0.0065 & 0.012 & 30 & $\alpha(0) = 2 \times 10^{-6}$, $d_{mo} = 105$ & \textendash & 100 & 105 & $\sim$ 5 mins\\
		\hline
		Algorithm \ref{Newton_step}& 8 & \textendash & 0.0065 & 1.25 & 1 & \textendash & 1 & 100 & 100 & $\sim$ 20 mins\\
		\hline
		Algorithm \ref{Newton_step}& 8 & \textendash & 0.0065 & 1 & 1 & \textendash & 1 & 50 & 50 & $\sim$ 20 mins\\
		\hline
		Algorithm \ref{SPPPOT_Newton_step}& 8 & \textendash & 0.0065 & SPPPOT: 0.1, Quasi--Newton: 1.8182 & SPPPOT: 30, Quasi--Newton: 1 & $\alpha(0)=1.5 \times 10^{-6}$, $d_{mo}=90$ & 1 & 50 & 90 & $\sim$ 20 mins\\
		\hline
		\multicolumn{11}{|c|}{ \vspace{-2mm}} \\
		\multicolumn{11}{|c|}{\bf Stephen Curry Data} \\
		\hline
		BFGS & 1 & 1 & 0.0025 & \textendash & \textendash & \textendash & \textendash & 100 & 8298 & $\sim$ 8 hrs\\
		\hline
		PMD & 10 & \textendash & 0.0025 & 0.1 & 30 & \textendash & \textendash & 100 & 100 & $\sim$ 5 mins \\
		\hline
		Algorithm \ref{SPFPMD_algorithm}& 10 & \textendash & 0.0025 & 0.03 & 30 & $\epsilon = 10^{-5}$ & \textendash & 100 & 100 & $\sim$ 5 mins\\
		\hline
		Algorithm \ref{Newton_step}& 10 & \textendash & 0.0025 & 0.2381 & 1 & \textendash & 1 & 100 & 100 & $\sim$ 20 mins\\
		\hline
		Algorithm \ref{Newton_step}& 10 & \textendash & 0.0025 & 0.2381 & 1 & \textendash & 1 & 50 & 50 & $\sim$ 20 mins\\
		\hline
		Algorithm \ref{SPPPOT_Newton_step}& 10 & \textendash & 0.0025 & SPPPOT: 0.03, Quasi--Newton: 0.1408 & SPPPOT: 30, Quasi--Newton: 1 & $\epsilon = 10^{-5}$ & 1 & 100 & 100 & $\sim$ 20 mins\\
		\hline
		\multicolumn{11}{|c|}{ \vspace{-2mm}} \\
		\multicolumn{11}{|c|}{\bf Chicago Crime Spatial Data} \\
		\hline
		BFGS & 1 & 1 & 0.01 & \textendash & \textendash & \textendash & \textendash & 441 & 10000 & $\sim$ 12 hrs\\
		\hline
		PMD & 4 & \textendash & 0.01 & 0.8 & 30 & \textendash & \textendash & 441 & 441 & $\sim$ 5 mins\\
		\hline
		Algorithm \ref{SPFPMD_algorithm}& 4 & \textendash & 0.01 & 0.15 & 30 & $\epsilon = 3 \times 10^{-4}$ & \textendash & 441 & 441 & $\sim$ 5 mins\\
		\hline
		Algorithm \ref{Newton_step}& 4 & \textendash & 0.01 & 6.6667 & 1 & \textendash & 1 & 441 & 441 & $\sim$ 20 mins\\
		\hline
		Algorithm \ref{SPPPOT_Newton_step}& 4 & \textendash & 0.01 & SPPPOT: 0.4, Quasi--Newton: 5 & SPPPOT: 30, Quasi--Newton: 1 & $\epsilon = 3 \times 10^{-4}$ & 1 & 441 & 441 & $\sim$ 20 mins\\
		\hline
	\end{tabular}
	\caption{Parameters selected for the various experiments.}
	\label{tab:hyperparameters}\vspace{-3mm}
\end{table*}
%

For PPP pdf estimation (Example \ref{eg:poisson}) of function $f$, the negative log-likelihood \eqref{obj_func_poisson3} (identical to \cite[Eq. (3.1)]{flaxman2017poisson}) defines the loss, whose minimization over time may be executed via Algorithms \ref{SPFPMD_algorithm}, \ref{Newton_step} and \ref{SPPPOT_Newton_step}. For comparison, batch offline BFGS (Quasi--Newton) minimization method \cite{flaxman2017poisson} and also online Pseudo Mirror Descent (PMD) \cite[Eq. (8)]{yang2019learning} are implemented, where the later uses only grid points with no dictionary adaptation. We also note that the PMD loss function \cite[Eq. (8)]{yang2019learning} requires the optimal intensity $f^\ast$ term, which is unknown for the real world data. BFGS is included in 'fminunc' package of MATLAB with intensity given by $f(\cdot) = a f^\prime(\cdot)^2$ where $a$ is a positive scalar used for tuning. $f(\cdot)$ is the actual intensity function expressed via $f^\prime(\cdot)^2$ to preserve positivity. We also compare with the online benchmark of POLK \cite{koppel2019parsimonious}, i.e., functional stochastic gradient method with the RKHS norm as the Bregman divergence, combined with KOMP in the primal RKHS, not the dual.

For all online schemes that use mini-batching, dividing the cumulative gradient by the mini-batch size is needed to remove scale ambiguity of integral approximations in the pseudo-gradient \eqref{PPP_grad}. The function learnt in BFGS is scaled by the number of data points. Thus, the function learnt for all the schemes is the normalized intensity. To obtain the intensity, the density must be multiplied by the sample size. For all the experiments we use Gaussian kernel, whose bandwidth is chosen using via Silverman's rule of thumb \cite{silverman1986density} for synthetic 1-D Gaussian toy example, and by trial and error for the real world data sets. The rest of the system parameters for different experiments are chosen by trial and error -- see Table \ref{tab:hyperparameters}.


{\bf \noindent 1-D Gaussian toy example.}
%
We consider 10211 training and 1001 test data points generated from a non-homogeneous PPP model having Gaussian density $f(x) = \frac{10}{\sqrt{2 \pi}} \exp\left( 50 (x-0.5)^2 \right)$. All data points $x\in(0,1)$ are univariate and have restricted range.

\noindent {\it Results.} The normalized intensity estimate obtained through different algorithms and their corresponding model order, log of RMSE and the test loss evaluation with respect to the number of samples processed are given in Figs. \ref{pdf_gaussian_toy}, \ref{model_order_gaussian_toy}, \ref{log_RMSE_gaussian_toy} and \ref{loss_gaussian_toy} respectively. From the intensity plot in Fig. \ref{pdf_gaussian_toy}, we observe that BFGS, PMD, Algorithms \ref{SPFPMD_algorithm}, \ref{Newton_step} and \ref{SPPPOT_Newton_step} yield densities near the ground truth, whereas POLK, with high test loss: $2.0148$ and test RMSE: $5.8317$, yields density estimates that become negative at a few places. This underscores the importance of using the KL-Divergence in Algorithms \ref{SPFPMD_algorithm}, \ref{Newton_step} and \ref{SPPPOT_Newton_step} to preserve positivity. Algorithm \ref{SPPPOT_Newton_step} with model order $90$ ($|\mathcal{U}|=50$ uniform grid points and $40$ Poisson points) outperforms alternatives including Algorithm \ref{Newton_step} since application of SPPPOT in Algorithm \ref{SPPPOT_Newton_step} allows it to discern the optimal subspace for the Quasi--Newton updates. Overall, then, SPPPOT yields a favorable accuracy/complexity tradeoff among first order algorithms whereas Algorithms \ref{Newton_step} and \ref{SPPPOT_Newton_step} achieve superior performance.

\noindent {\bf Stephen Curry Data.}
We experiment with the NBA dataset of Stephen Curry which has shot distances from the basket up to $40$ feet normalized as samples $x \in \mathbb{R}$. These samples are split into  $8298/1000$ points for training/testing.


\noindent {\it Results.} The training loss, the estimated normalized intensity, and the model order are given in Figs. \ref{loss_Curry}, \ref{intensity_Curry}, and \ref{model_order_Curry}, respectively. The runtime difference between BFGS and the other online methods (given in Table \ref{tab:hyperparameters}) is reflected in the model complexity difference in \ref{model_order_Curry}. Among first order algorithms, SPPPOT outperforms the batch solver after a few training epochs (Fig. \ref{loss_Curry}), and yields an intensity much closer to the offline baseline BFGS compared to PMD. Algorithms \ref{Newton_step} and \ref{SPPPOT_Newton_step} outperforms the existing state of the art schemes and the offline BFGS by a healthy margin.


\noindent {\bf Chicago Crime Spatial Data.}
Chicago 2018 crime \href{https://data.cityofchicago.org/Public-Safety/Crimes-2018/3i3m-jwuy}{dataset} has 2D features (X-Y Coordinates) determining the crime location in the city. $10,000$ training points and $2000$ test points are randomly selected and normalized to lie within a unit square.


\noindent {\it Results.} Fig. \ref{Histogram_Chicago} shows the histogram of the $10,000$ training data (normalized) with $30 \times 30$ bins in the square grid. The histogram scale is not to be confused with the learned density since histogram is dependent on the number of data points and the size of each bin whereas the later is not. Fig. \ref{loss_Chicago} is negative at someplaces since density aka pdf is used in the negative log-likelihood. Observe that the precision and accuracy of Algorithms \ref{Newton_step} and \ref{SPPPOT_Newton_step} in terms of test loss evaluation (Fig. \ref{loss_Chicago}) and determination of the peaks/troughs of the normalized intensity with respect to the ground truth (Figs. \ref{Quasi_Newton_intensity} and \ref{SPPPOT_then_Newton_intensity}) helps it achieve the state of the art.
%
%
The run time complexity (given in Table \ref{tab:hyperparameters}) is justified from Fig. \ref{Model_order_Chicago}.

\begin{figure*}[t]\centering
	\begin{subfigure}{\columnwidth}
		\hspace{1.5cm}\includegraphics[scale = 0.7, trim=0mm 6.9mm 0mm 6.3mm, clip=true]
		{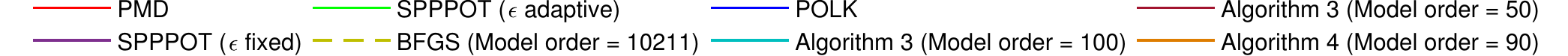}
	\end{subfigure}
	\begin{subfigure}{.245\columnwidth}
		\includegraphics[width=1.1\linewidth, height = 0.75\linewidth]
		{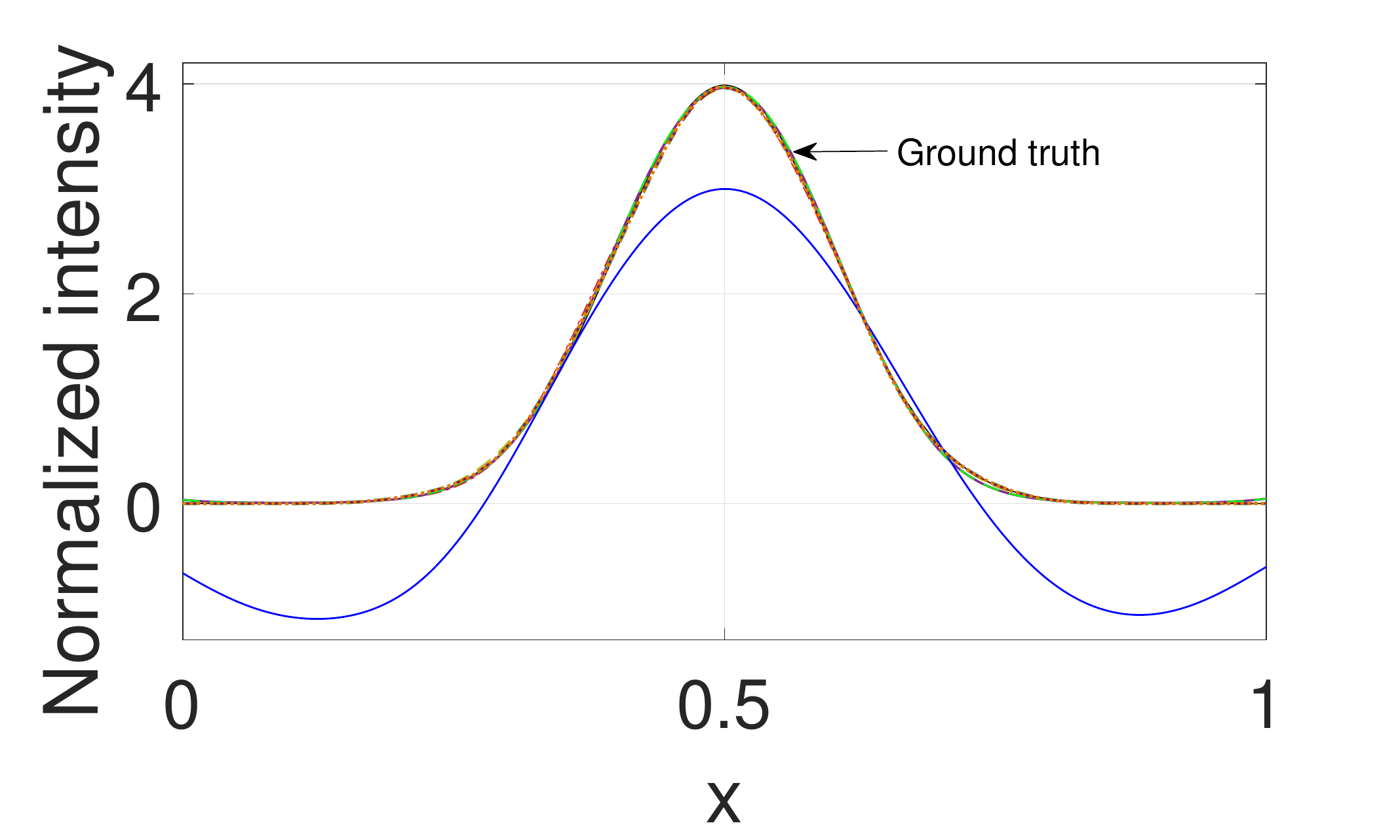}
		\caption{Fitted Normalized intensity}
		\label{pdf_gaussian_toy}
	\end{subfigure}
	\begin{subfigure}{.245\columnwidth}
		\includegraphics[width=1.1\linewidth,height = 0.75\linewidth]
		{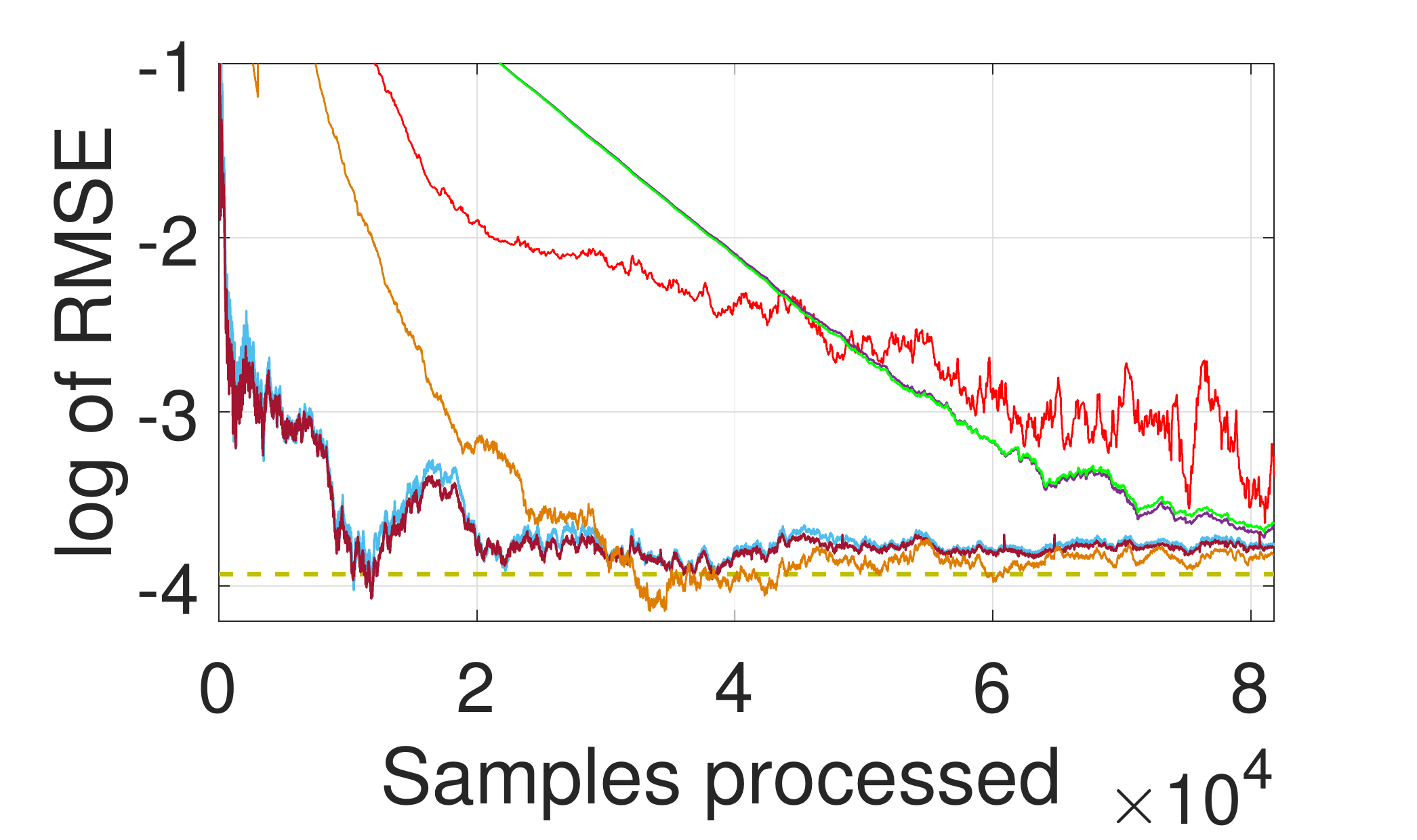}
		\caption{log of RMSE vs. samples}
		\label{log_RMSE_gaussian_toy}
	\end{subfigure}
	\begin{subfigure}{.245\columnwidth}
		\includegraphics[width=1.1\linewidth,height = 0.75\linewidth]
		{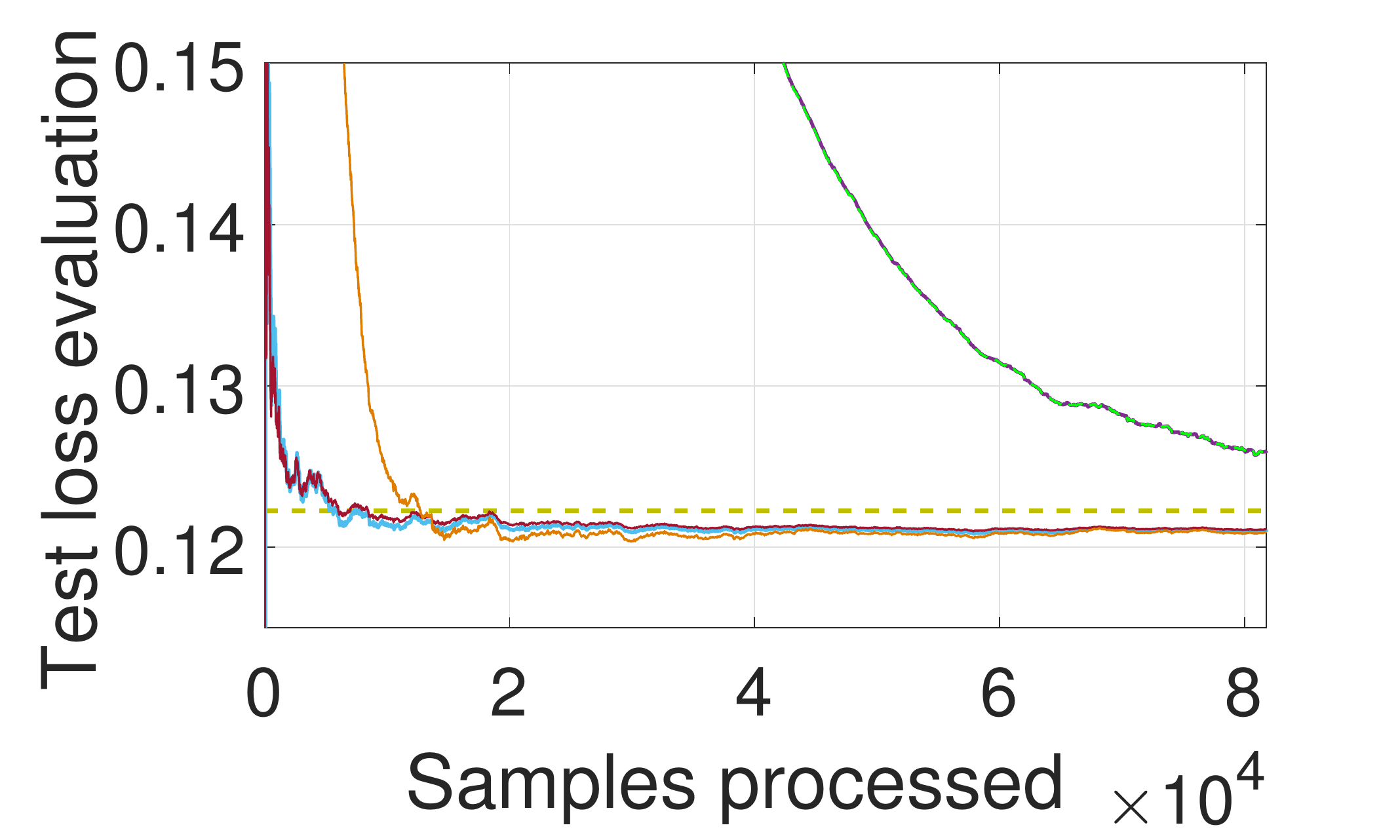}
		\caption{Test Loss vs. samples}
		\label{loss_gaussian_toy}
	\end{subfigure}
	\begin{subfigure}{.245\columnwidth}
		\includegraphics[width=1.1\linewidth,height = 0.75\linewidth]
		{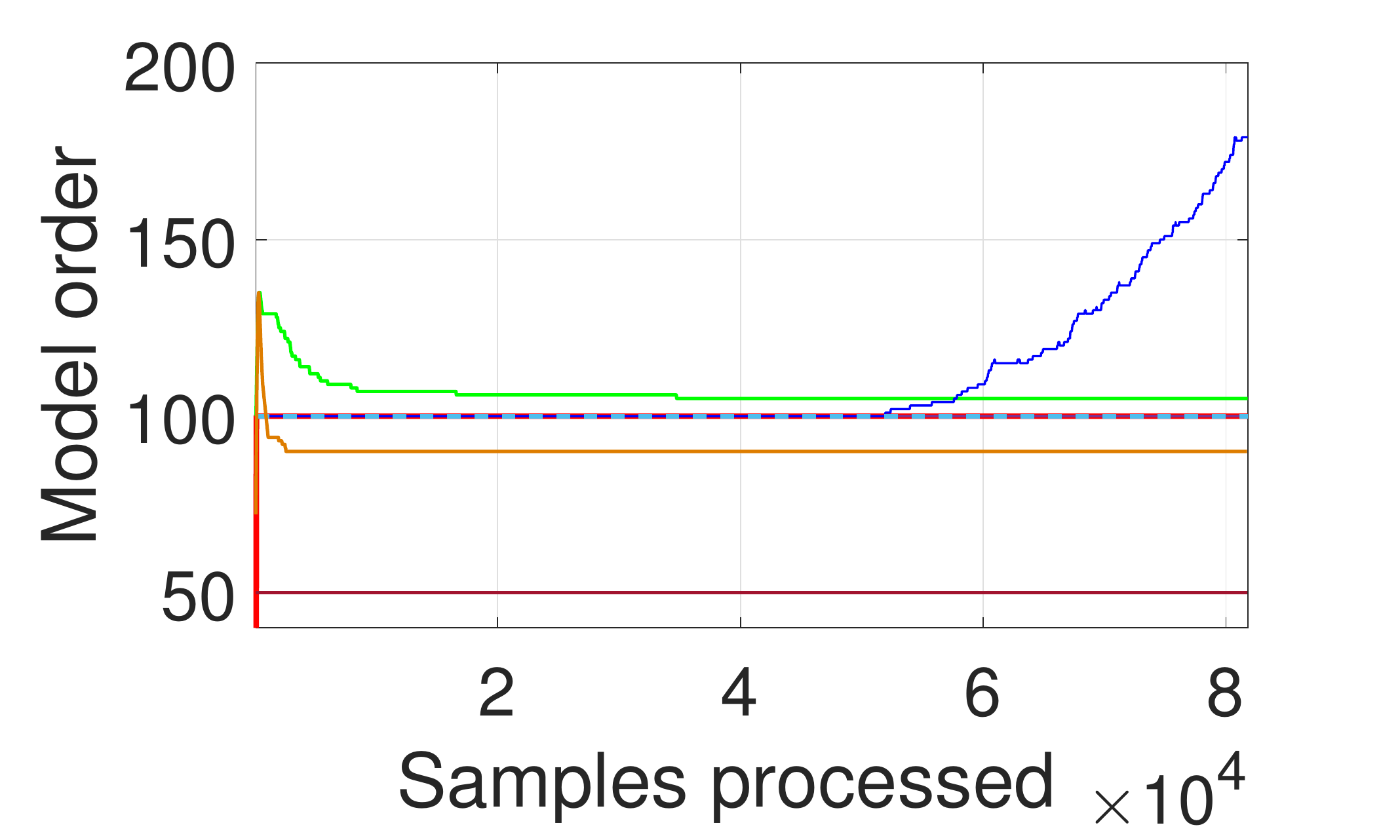}
		\caption{Model Complexity vs. samples}
		\label{model_order_gaussian_toy}
	\end{subfigure}%
	\caption{We showed the performance of Algorithms \ref{SPFPMD_algorithm}, \ref{Newton_step} and \ref{SPPPOT_Newton_step} in comparison to PMD and offline BFGS algorithm on the 1-D synthetic Gaussian toy example. POLK is also implemented with constant budget to emphasize that it violates the range constraint. Online second order Algorithms \ref{Newton_step} and \ref{SPPPOT_Newton_step} achieves the state of the art.}
	\label{fig:Gaussian_intensity_label} \end{figure*}
%

\begin{figure*}[t]\centering
	\begin{subfigure}{\columnwidth}
		\hspace{2.8cm}\includegraphics[scale = 0.73, trim=0mm 5mm 0mm 7mm, clip=true]
		{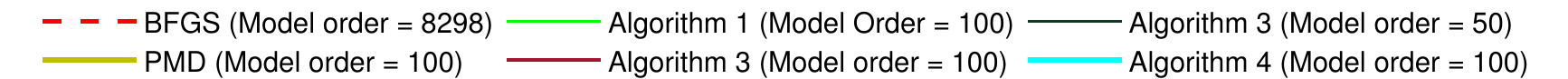}
	\end{subfigure}
	\begin{subfigure}{.33\columnwidth}
		\includegraphics[width=1.1\linewidth, height = 0.75\linewidth]
		{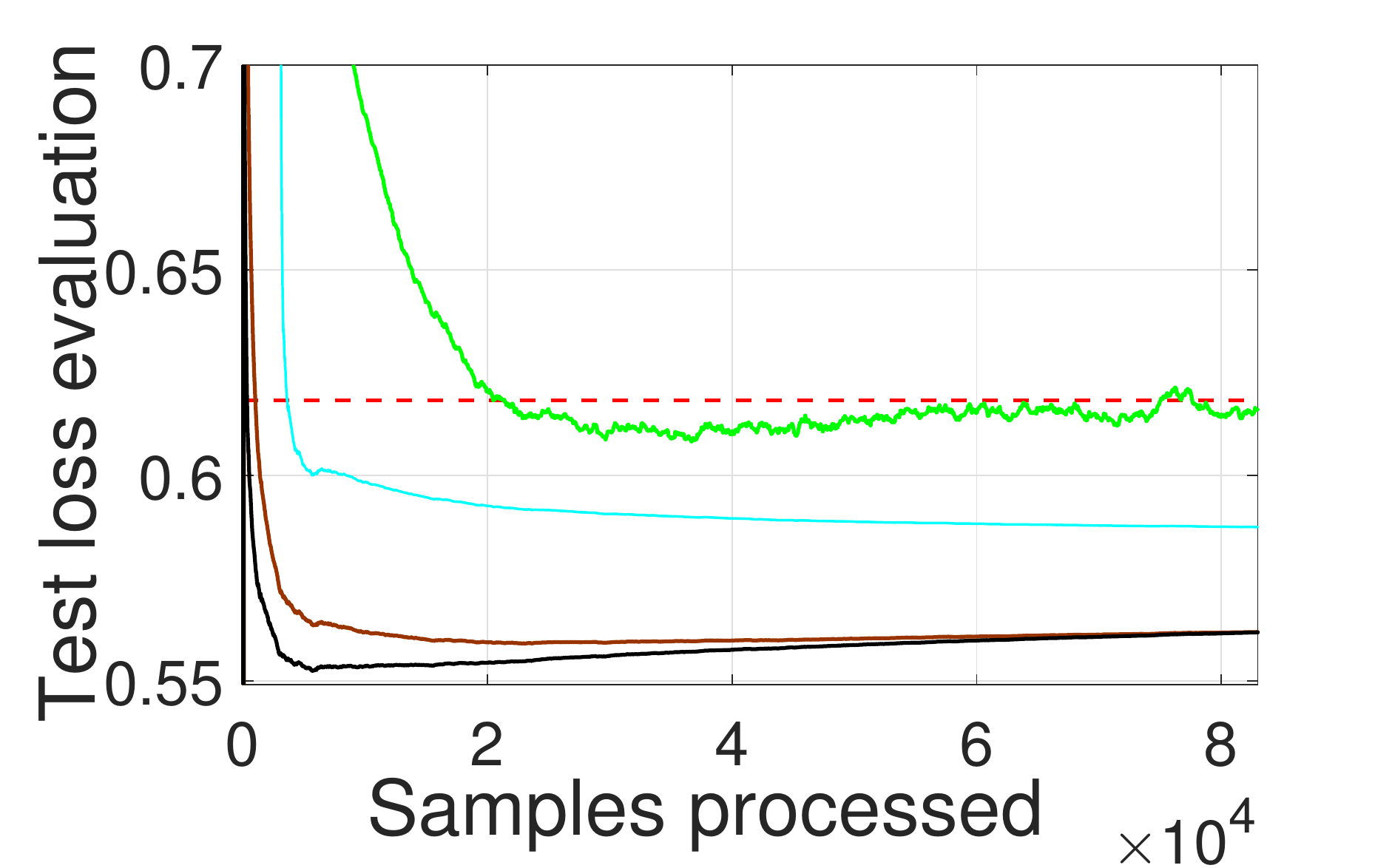}
		\caption{Test Loss vs. samples processed}
		\label{loss_Curry}
	\end{subfigure}
	\begin{subfigure}{.325\columnwidth}
		\includegraphics[width=1.05\linewidth,height = 0.75\linewidth]
		{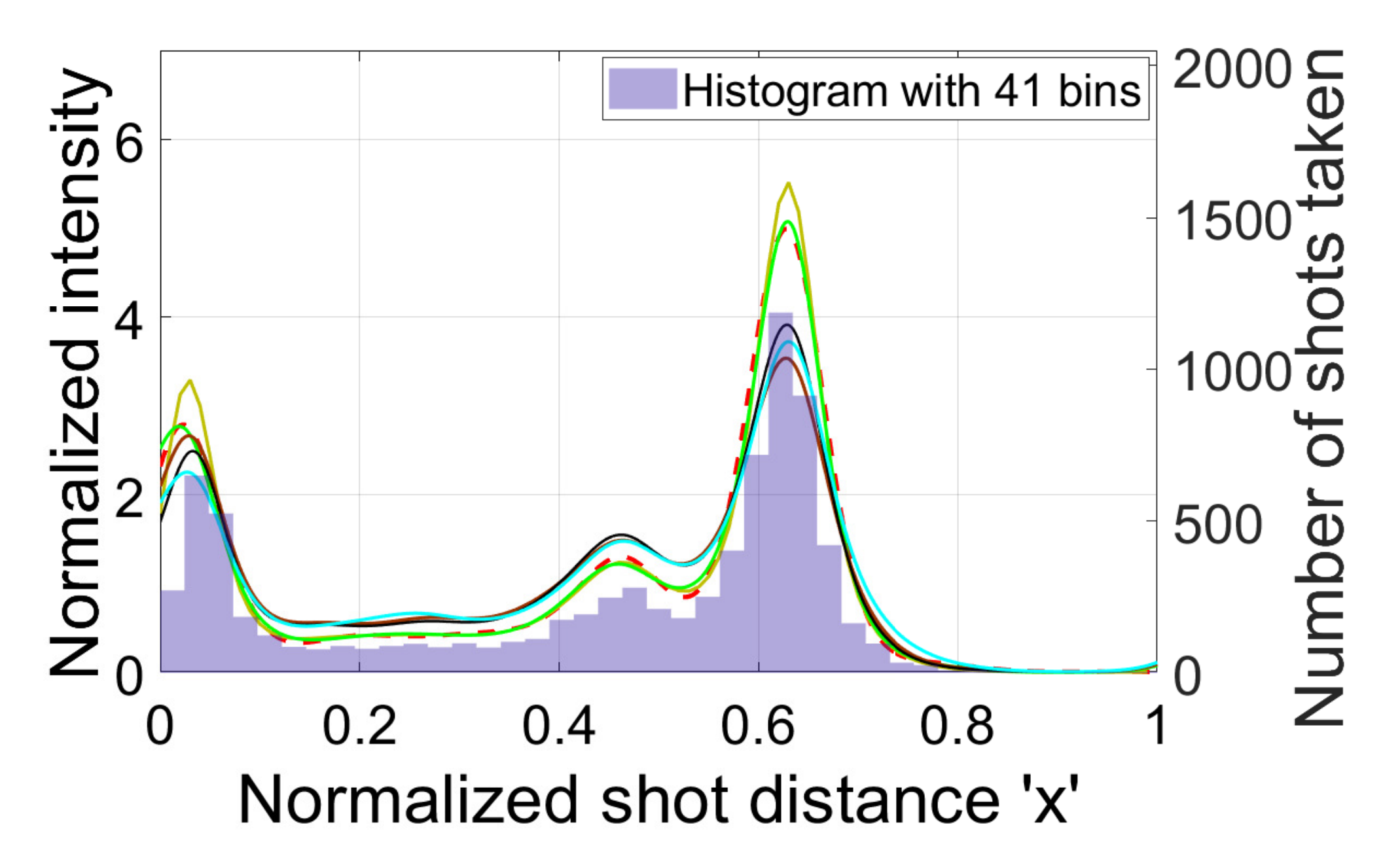}
		\caption{Fitted normalized Poisson Intensity}
		\label{intensity_Curry}
	\end{subfigure}
	\begin{subfigure}{.33\columnwidth}
		\includegraphics[width=1.05\linewidth,height = 0.75\linewidth]
		{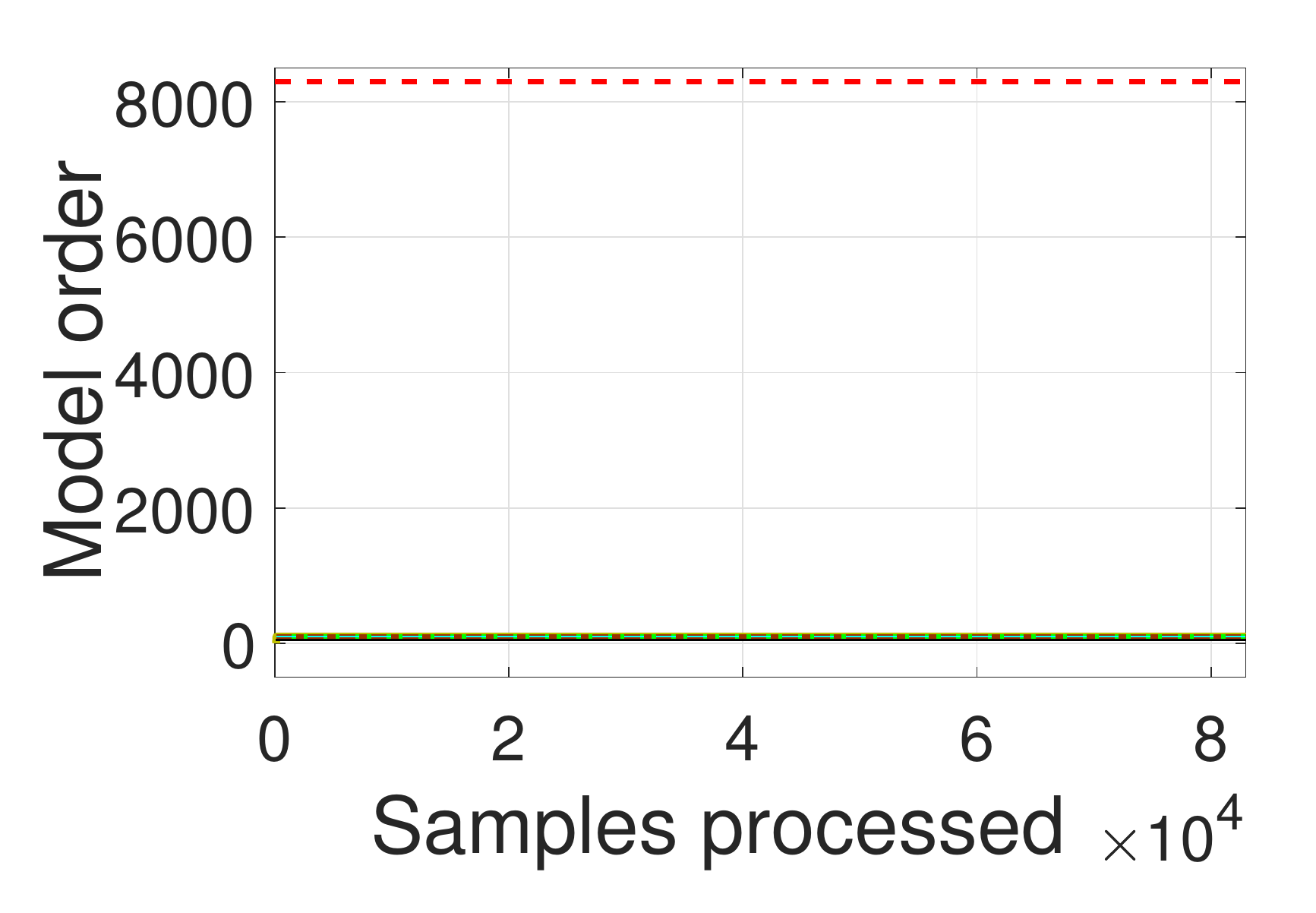}
		\caption{Model Complexity vs. samples processed}
		\label{model_order_Curry}
	\end{subfigure}%
	\caption{We compare our Algorithms \ref{SPFPMD_algorithm}, \ref{Newton_step} and \ref{SPPPOT_Newton_step} with BFGS solver for an offline MLE problem \cite{flaxman2017poisson} as well as with online Pseudo-mirror descent (PMD) \cite{yang2019learning}, i.e., which do not incorporate complexity-reducing projections on the NBA dataset of Stephen Curry. Online second order Algorithms \ref{Newton_step} and \ref{SPPPOT_Newton_step} is much more efficient and achieves the state of the art.}
	\label{fig:label} \end{figure*}


\begin{figure*}[t]\centering
	\begin{subfigure}{\columnwidth}
		\hspace{0cm}\includegraphics[scale = 0.69, trim=0mm 4mm 0mm 13.2mm, clip=true]
		{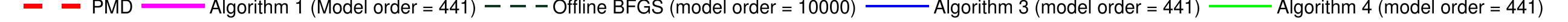}
	\end{subfigure}
	\begin{subfigure}{.245\columnwidth}
		\includegraphics[width=1.08\linewidth, height = 0.75\linewidth]
		{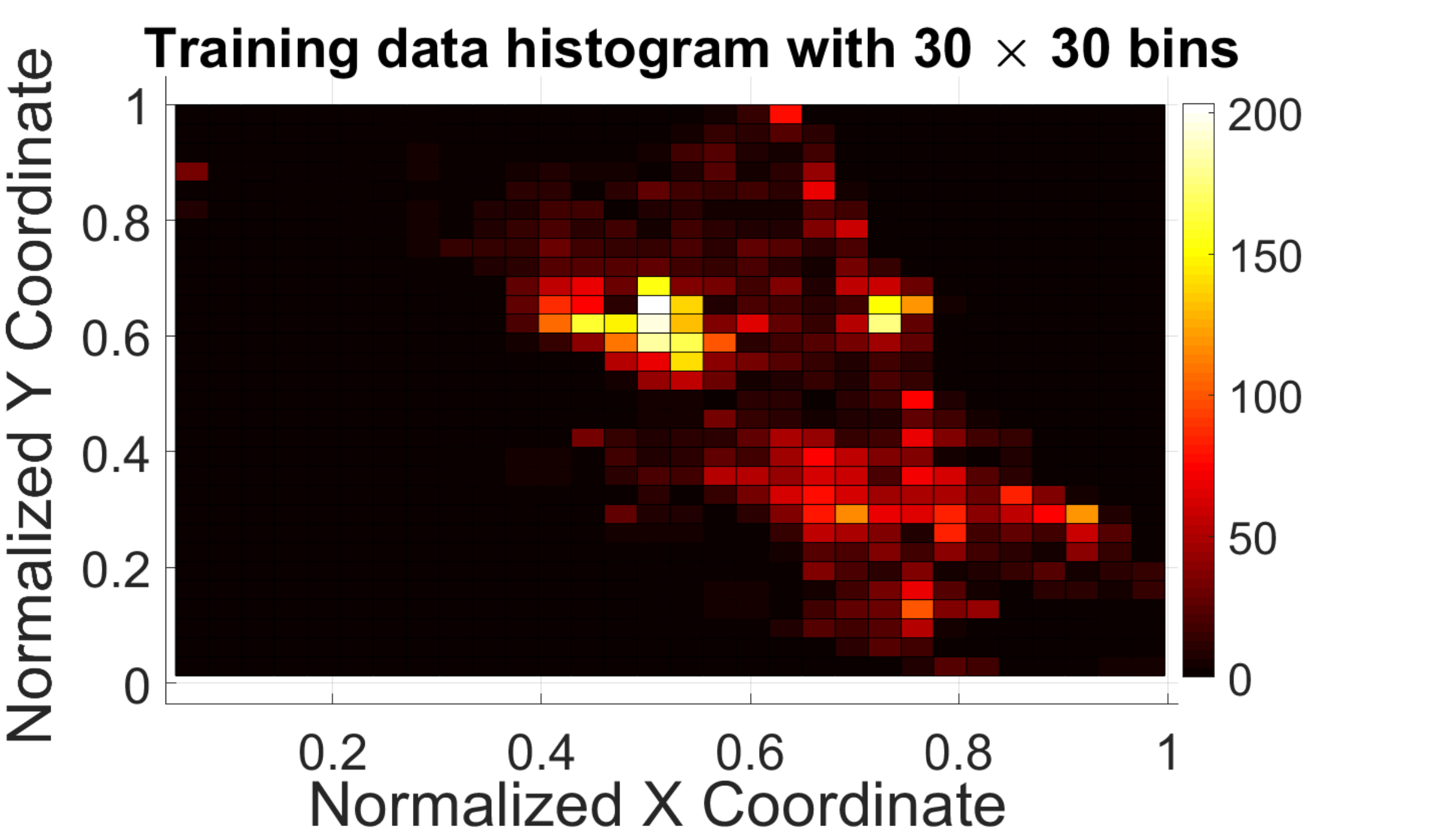}
		\caption{Histogram of training data}
		\label{Histogram_Chicago}
	\end{subfigure}
	\begin{subfigure}{.245\columnwidth}
		\includegraphics[width=1.08\linewidth,height = 0.75\linewidth]
		{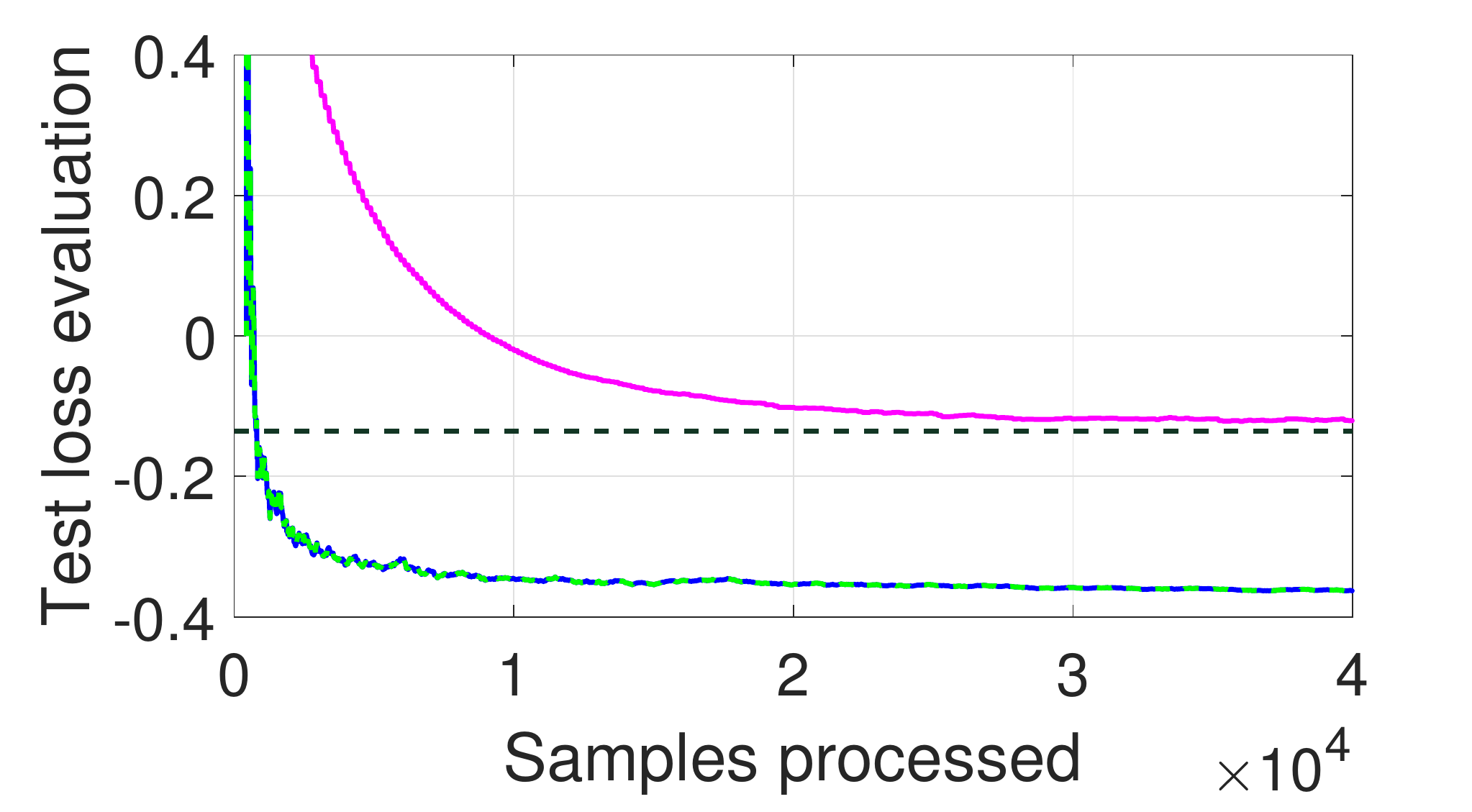}
		\caption{Test Loss vs. samples}
		\label{loss_Chicago}
	\end{subfigure}
	\begin{subfigure}{.245\columnwidth}
		\includegraphics[width=1.1\linewidth,height = 0.75\linewidth]
		{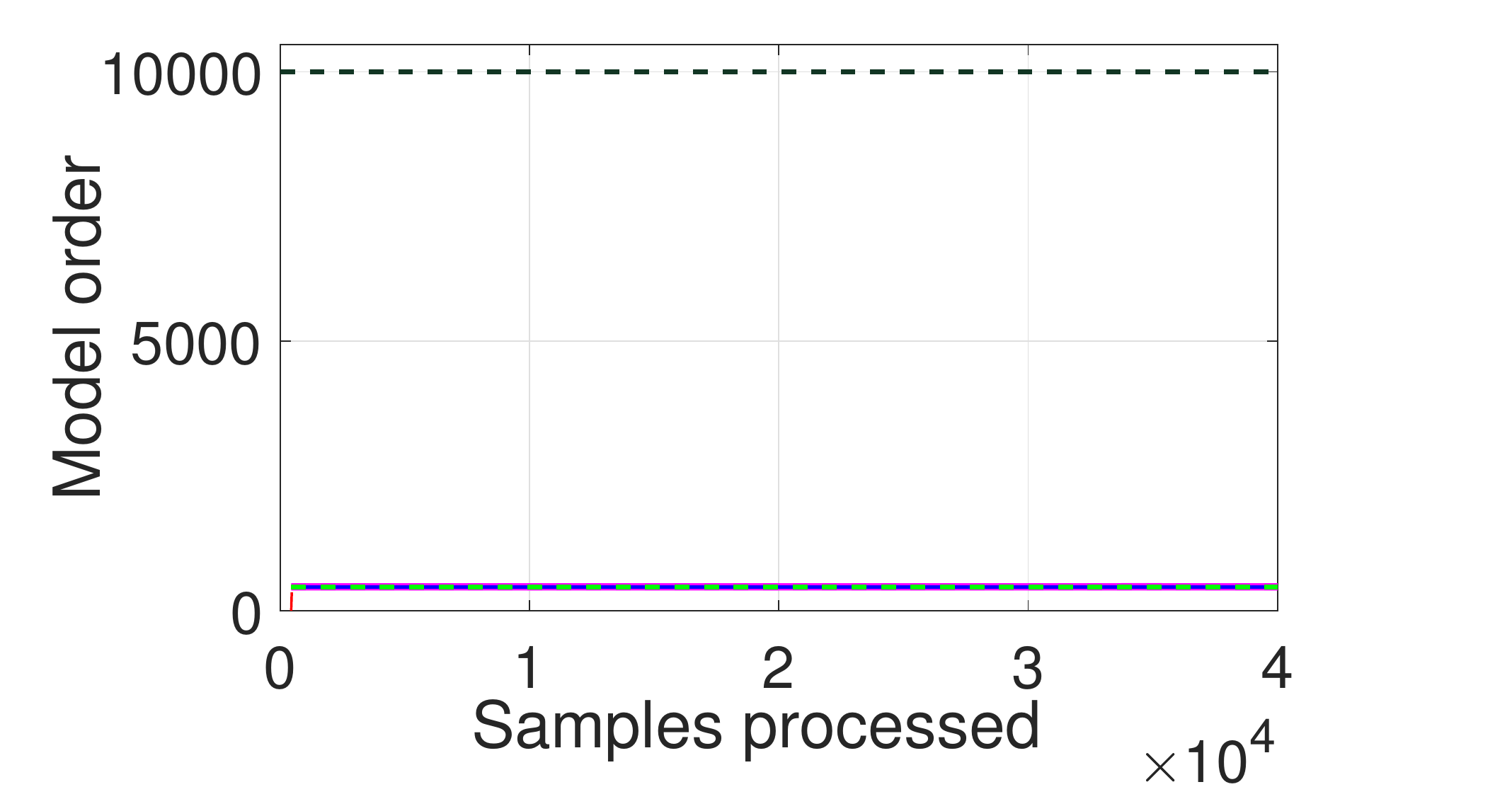}
		\caption{Model Complexity vs. samples}
		\label{Model_order_Chicago}
	\end{subfigure}
	\begin{subfigure}{.24\columnwidth}
		\includegraphics[width=1\linewidth,height = 0.75\linewidth]
		{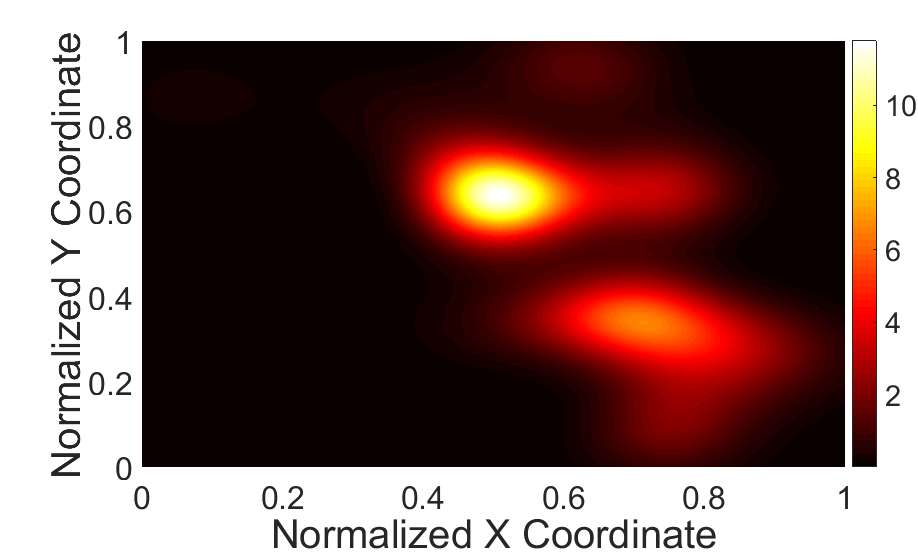}
		\caption{Pdf using offline BFGS}
		\label{Offline_BFGS_intensity}
	\end{subfigure}
	\begin{subfigure}{.24\columnwidth}
		\includegraphics[width=1.02\linewidth,height = 0.75\linewidth]
		{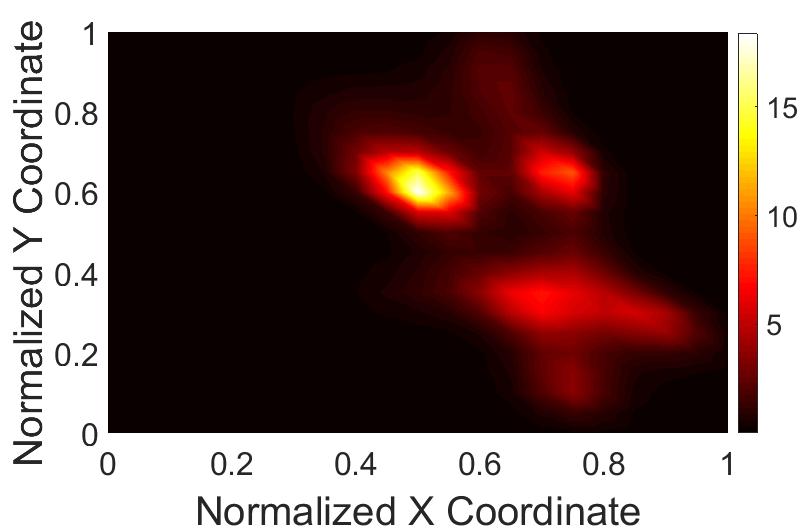}
		\caption{Pdf using PMD}
		\label{PMD_intensity}
	\end{subfigure}
	\begin{subfigure}{.24\columnwidth}
		\includegraphics[width=1.02\linewidth,height = 0.75\linewidth]
		{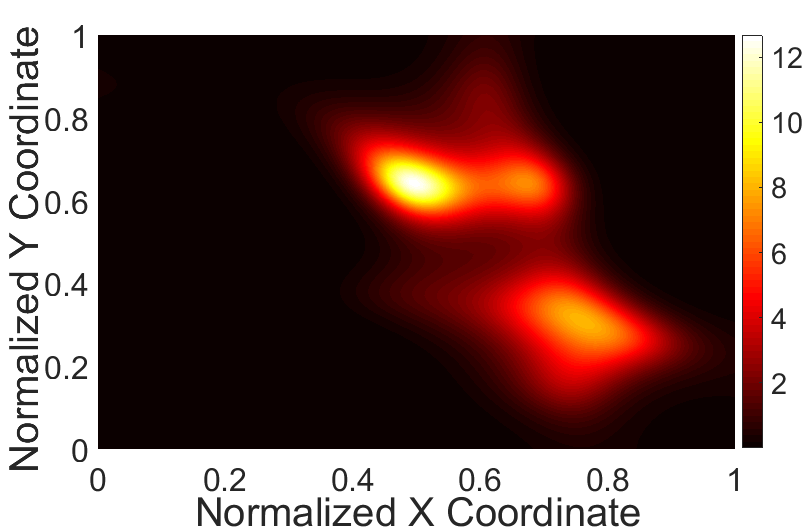}
		\caption{Pdf using Algorithm \ref{SPFPMD_algorithm}}
		\label{SPPPOT_const_budget_intensity}
	\end{subfigure}
	\begin{subfigure}{.24\columnwidth}
		\includegraphics[width=1.02\linewidth,height = 0.75\linewidth]
		{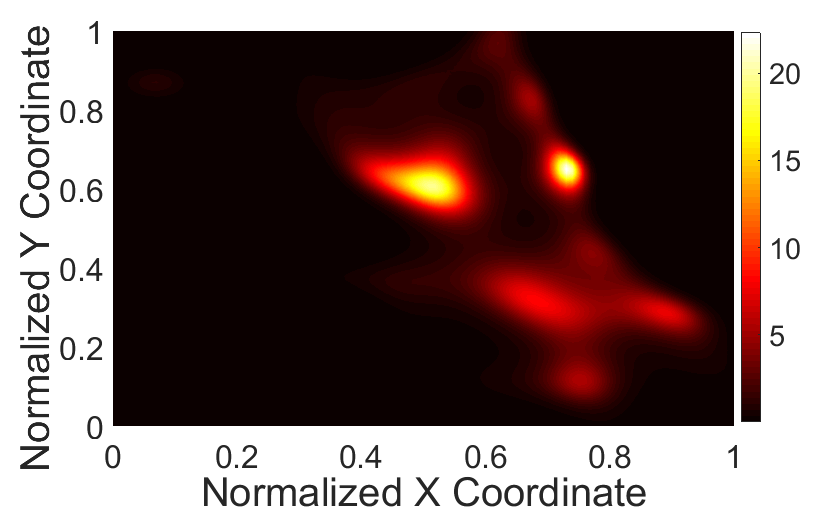}
		\caption{Pdf using Algorithm \ref{Newton_step}}
		\label{Quasi_Newton_intensity}
	\end{subfigure}
	\begin{subfigure}{.24\columnwidth}
		\includegraphics[width=1.02\linewidth,height = 0.75\linewidth]
		{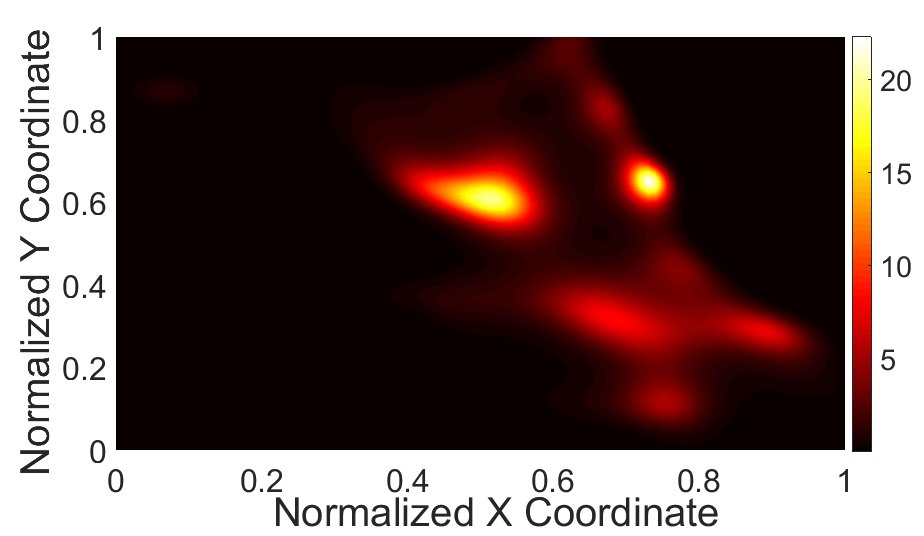}
		\caption{Pdf using Algorithm \ref{SPPPOT_Newton_step}}
		\label{SPPPOT_then_Newton_intensity}
	\end{subfigure}
	\caption{We compare the performance of our first order Algorithm \ref{SPFPMD_algorithm}, second order Algorithm \ref{Newton_step} and their combination Algorithm \ref{SPPPOT_Newton_step} with PMD and offline BFGS on Chicago crime data set of 2018. Second order Algorithms \ref{Newton_step} and \ref{SPPPOT_Newton_step} beats the state of the art algorithms in normalized intensity estimation.}
	\label{fig:Chicago_intensity_label} \end{figure*}


%

	\section{Conclusion}
	
	We studied expected value minimization when the decision variable belongs to a Reproducing Kernel Hilbert Space (RKHS) whose target domain is required to be nonnegative. We developed efficient algorithms using first and second order gradient information. For the first order algorithm, we put forth a variant of stochastic mirror descent that employs (i) \emph{pseudo-gradients} and (ii) projections. Compressive projections are executed via KOMP, which mitigate the complexity issues of RKHS parameterizations. For the second order algorithm, we fixed our underlying subspace and used second order \emph{pseudo-gradients} to make a descent in the weights of the auxiliary function. We established accuracy/complexity tradeoffs between sub-optimality convergence in mean and bounds on the model complexity under standard assumptions. Experiments demonstrated competitive performance with the state of the art for Poisson intensity estimation. Future directions include scaling the intensity estimation approach to higher dimensions through convolutional kernels, and employing the resulting intensity estimates for event triggers in queueing systems.
	%
	%
%

\appendices
	\section{Proof of Lemma \ref{qbound}}\label{lemma1_proof}
\begin{proof}
This result mirrors the proof of \cite{koppel2019parsimonious}[Prop. 8] but instead studies the directional error of pseudo-gradients with respect to the dual sequence $\{z_t\}$ \eqref{FPMD2}.
First, consider the dual-norm difference between $\gh_t$ [cf. \eqref{gh}] and $g_t$ [cf. \eqref{FPMD2}], to get
%
%
\begin{align}
	\norm{\gh_t-g_t}_\ast 
	&= \| \frac{1}{\eta}(z_t-z_{t+1}) - g_t \|_\ast \nonumber \\
	&= \frac{1}{\eta} \norm{z_{t+1}-\tz_{t+1}}_\ast \leq \frac{\epsilon}{\eta} ,
\end{align}
where we have grouped terms and used the definition of $\tilde{z}_{t+1}$ in \eqref{FPMD2}, together with the KOMP stopping criterion $\norm{z_{t+1}-\tz_{t+1}}_\ast \leq \epsilon$.
To obtain \eqref{projected_grad_bound}, consider $\EE\ns{\gh_t}_\ast$, and add and subtract $g_t$ inside the expectation to write 
\begin{align}
	&\EE\ns{\gh_t}_\ast \leq 2\EE\ns{g_t}_\ast + 2\EE\ns{\gh_t-g_t}_\ast \\ 
	&\qquad\leq 2({b^2} + c^2 \EE [\ip{\nabla R(f_t)}{\EE[g_t | \cF_t]}]) + 2\EE\ns{\gh_t-g_t}_\ast \nonumber ,
\end{align}
where we have applied $(a+b)^2 \leq 2a^2 + 2b^2$ to the result of this subtraction, followed by Assumption \ref{grad}.
\end{proof}


\section{Proof of Corollary \ref{KOMP_model_order_thm}} \label{model_conv_proof}
\begin{proof}
	Begin by recalling the definition of the function approximation error of KOMP (Algorithm \ref{kompalgo}) 
	$$\gamma_j := \norm{ \tz- \sum\limits_{\x_n \in \cD \setminus \{\x_j\}}w_n \kappa(\x_n,.) }_\ast$$ 
	used to determine the compression stopping criteria. $M_t$ is the model order of $z_t$ and the model order of $\tz_{t+1}$ is $M_t+1$ for addition of single data point. Then, $\gamma_{M_t+1}$ may be written as 
	\begin{align}
		\gamma_{M_t+1} &= \norm {\tz_{t+1} - \sum\limits_{\x_n \in \cD_{t+1} \setminus \{\x_{M_t+1}\}}w_n \kappa(\x_n,.) }_\ast . 
	\end{align}
	Substitute in the definition $\tz_{t+1}=z_t - \eta g_t$ [cf. \eqref{FPMD2}], and note that dictionary elements are deleted if $\min \limits_{j=1,\dots,M_{t+1}} \gamma_j \leq \epsilon$. We note that $\gamma_j$ for ${j=1,\dots,M_{t+1}}$ is lower-bounded by $\gamma_{M_t + 1}$, so it suffices to consider $\gamma_{M_t + 1}$.
	Since the pseudo-gradient admits an expansion via the chain rule as $g_t =  g_t^\prime \kappa(\x_t,\cdot)$ (Assumption \ref{C_lip}), one may reduce this condition to the set distance between the dual of the subspace defined by the current dictionary and the kernel evaluated at the latest sample $\mathbf{x}_t$ by applying similar logic to  \cite[Appendix D.1 Eq. (70)-(75) and Lemma 10]{koppel2019parsimonious}, that is, the following holds	
	\begin{align}\label{KOMP_dist_criteria}
		\text{dist}\left(\kappa(\mathbf{x}_{t},\cdot),\mathcal{H}_{\cD_t}^\ast \right) \leq \frac{\epsilon}{\eta |g_t^\prime|} = \frac{\alpha}{|g_t^\prime|} .
	\end{align}
	There are two possibilities: either a new point is added or it is not. If it is not, then $M_{t+1} \leq M_{t}+1$ and the model order does not change. Let's consider when this condition is violated, such that the model order increases, i.e., $M_{t+1}=M_t+1$:	
	\begin{align}\label{violation}
		\text{dist}\left(\kappa(\mathbf{x}_{t},\cdot),\mathcal{H}_{\cD_t}^\ast \right) > \frac{\epsilon}{\eta |g_t^\prime|} = \frac{\alpha}{|g_t^\prime|} .
	\end{align}
	Assumption \ref{C_lip} implies $\frac{1}{|g_t^\prime|} \geq \frac{1}{C}$, which allows us to lower bound \eqref{violation} with $\frac{\alpha}{C}$, where $\alpha = \frac{\epsilon}{\eta}$. KOMP stopping criteria is violated when distinct dictionary points $\mathbf{d}_n$ and $\mathbf{d}_j$ for $j,n \in \{1,..., M_t\}$ satisfy the condition $\norm{\phi(\mathbf{d}_n)-\phi(\mathbf{d}_j)}_\ast > \frac{\epsilon}{\eta C} = \frac{\alpha}{C}$. By Assumption \ref{feature}, the number of Euclidean balls of radius $\frac{\epsilon}{\eta C}$ required to cover the feature kernelized space $\phi(\x)=\kappa(\x,\cdot)$ is always finite by the continuity of the kernel. Therefore, there exists a finite $t$ such that \eqref{KOMP_dist_criteria} holds which implies that equation $\min \limits_{j=1,\dots,M_{t+1}} \gamma_j \leq \epsilon$ is always valid, meaning the model order does not grow. More specifically, there exists a  finite maximum model order $M^\infty<\infty$ such that $M_t<M^\infty$, holds for all $t$.
	
	Now if the kernel taken is Lipschitz continuous on the set $\cX$ which is compact by Assumption \ref{feature}, then similar to \cite[Proposition 2.2]{engel2004kernel}, the maximum model order $M^\infty$ is bounded as \eqref{max_model_order_possible}, which yields an upper-bound on $M^\infty$  for all $t$.
\end{proof}


\section{Proof of Lemma \ref{lemma_HS_bound}}\label{lemma_HS_bound_proof}
\begin{proof}
	The Hessian matrix $\mathbf{A}_{\infty,t+1}$ has finite maximum eigenvalue $\mu_{t+1}^{\max}$ via Assumption \ref{Hessian_up_bound}, which implies $\frac{1}{\mu_{t+1}^{\max}}$ is the minimum eigenvalue of $\mathbf{A}_{\infty,t+1}^{-1}$. Thus we may write the lower bound of \eqref{HS_bound_eq}. 
	%
	%
	The upper-bound employs the update \eqref{hess_update} with $\g_{\infty,t}$ recursively backwards in time as \vspace{-2mm}
	\begin{align}
		\mathbf{A}_{\infty,t+1}^{-1} = \left(\delta \mathbf{I} + \sum_{j=1}^t \g_{\infty,j} {\g_{\infty,j}}^\top \right)^{-1} \preccurlyeq \frac{1}{\delta} \mathbf{I} .
	\end{align}
\end{proof}
%


\section{Functional Pseudo-Dual Averaging}\label{app:func_dual_averaging}

Dual averaging (see \cite{3}) is closely related to mirror descent. Here, we expand upon these connections. Specifically, dual averaging with pseudo-gradients in the functional RKHS setting takes the form
	\begin{align}
	\th_{t+1}&=h_t+g_t \label{dav1}\\ 
	f_{t+1} &= \argmin \limits_{f \in \cH} \Big(\langle h_{t+1}, f \rangle _\cH + \frac{1}{\eta} \psi (f) \Big) \nonumber\\
	& = \nabla \psi^{\star}(-\eta h_{t+1}) \label{dav2}
	\end{align}
	and then one may apply KOMP on $\th_{t+1}$ with budget parameter $\epsilon$ in order to ensure an efficient RKHS parameterization in terms of weights and dictionary elements. Here, $\psi^\star$ denotes the conjugate of $\psi$.  In order for \eqref{dav2} to be valid, one requires $\nabla\psi^\star:\cH^\ast \rightarrow \cH$ so that $f_{t+1} \in\cH$. For the stochastic case (a special case of pseudo gradient), the update can be written recursively as 
	\begin{align}
	\th_{t+1} = h_t + \nabla r_t(\nabla \psi^{\star}(-\eta h_t)) = h_t - \eta \g_t
	\end{align}
	where $\g_t := -\frac{1}{\eta}\nabla r_t(\nabla \psi^{\star}(-\eta h_t))$. To see the link between the \eqref{dav1}-\eqref{dav2} and Algorithm \ref{SPFPMD_algorithm}, multiply \eqref{dav1} by $-\eta$ and use the second update equation to obtain
	\begin{align}
	\nabla \psi(\tf_{t+1}) = \nabla \psi(f_t) - \eta g_t
	\end{align}
	where $\tf_{t+1} := \nabla \psi^\star(-\eta \th_{t+1})$. This means that under an additional hypothesis that the gradient of the Fenchel conjugate of $\psi$ does not take the functions $h_t$ outside the RKHS, update expressions for dual averaging in updating $h_t$ exactly coincides with $z_t$ in \eqref{FPMD2}. This link provides alternative potential ways of approaching the convergence theory of Algorithm \ref{SPFPMD_algorithm}, which are the subject of future work.


\section{Extensions of SPPPOT to supervised learning}\label{app:SPPPOT_KLR}

The first order algorithm SPPPOT can be easily extended to high dimensional supervised learning problems owing to its data adaptive dictionary and can give excellent performances compared to the current state of the art techniques. Let us first elucidate here the problem of supervised learning and then shall demonstrate the example of multi-class kernel logistic regression as an use case.

\smallskip
{\noindent \bf Supervised Learning.}	For supervised learning, target variables $y_t \in \cY$ are also received corresponding to each $\x_t$, where $y_t \in \cY$ denote the labels or real values, i.e., $\cY=\{1,\dots,C\}$ in the case of classification and $\cY\subset \Rn$ in the case of regression. The examples $(\x_t,y_t)$ are realizations from an unknown distribution $\mathbb{P}(\x,y)$. The goal is to fit a predictive model $f:\cX \rightarrow \cY$ that belongs to a particular hypothesized class of functions $\cH$, i.e., to form estimates of the target of the form $\hat{y}_t=f(\x_t)$. The merit of a given estimator $f$ is quantified by a convex loss $\ell: \mathcal{H}\times \cX \rightarrow \cY$ which is small when $f(\x)$ and $y$ are close. We seek to minimize the statistical loss $R(f) := \Ex{\ell(f(\x),y)}$ in expectation over the data. Example include the squared loss $(y-f(\x))^2$, the hinge loss $\max\{0,1-yf(\x)\}$, and the logistic loss derived from minimizing the class-conditional misclassification probability -- see \cite{murphy}. For this case also, we solve the same optimization problem as \eqref{ff}, where the instantaneous cost $r_t(f):=\ell(f(\x_t),y_t)$.

Next we present the demonstrative example of multi-class Kernel Logistic Regression (KLR) as a representative of supervised learning.

\begin{example}\textit{Kernel Logistic regression for multi-class classification: }\label{eg:logistic}
	In this case, the target domain $\cY=\{1,\dots,C\}$ is the set of classes, and the loss is defined by the negative log-likelihood of a generalized odds ratio. Specifically, define a vector-function associated with a one-hot encoding, i.e., the target domain is $\{0,1\}^C$, such that each $\x$ is assigned to a binary vector of length $C$. Then, one hypothesizes that the probability of the predicted label $y$ to lie in the class $c$ is given by the logit (softmax) function $\PP(y=c | \x) := \frac{\exp(f_c(\x))}{\sum_{\tilde{c}} f_{\tilde{c}}(\x)}$, which gives rise to the negative log-likelihood: \cite{murphy}
	\begin{align}
		\ell(\textbf{f},\x_t,y_t) &= -\log (\PP(y=y_t | \x_t)) \\
		& = \log\left(\sum_{\tilde{c}} f_{\tilde{c}}(\x_t)\right) - f_{y_t} (\x_t) \label{logistic_regression_loss}
	\end{align}
	where $\textbf{f}$ denotes vector of functions $f_{\tilde{c}}$ where $\tilde{c}$ denotes each of the class labels, and the classification of a point $\x$ is defined by the maximum a posteriori assignment: $\hat{c}=\argmax_{c\in\{1,\dots,C\}} f_c (\x)$. Observe that the space of binary sequences $\{0,1\}^C$ is nonnegative, meaning that minimizing \eqref{logistic_regression_loss} in expectation over $\mathbb{P}(\x,y)$ is an instance of \eqref{ff}.
\end{example}

\smallskip
\noindent {\bf First order gradient calculation for KLR. } Consider the setting of multi-class classification using a probabilistic generalization of the odds ratio which defines the loss \eqref{logistic_regression_loss}. We differentiate it with respect to each $f_{\tilde{c}}$ to obtain the functional stochastic gradient for each class where $\tilde{c} \in \{1,\ldots,C \}$. The gradient expression takes the form:
	\begin{align} \label{lr_grad}
	&g_t=\begin{cases}
	\frac{\exp\left(f_{c'}(\x_t)\right)}{\sum_{\tilde{c}} \exp\left(f_{\tilde{c}}(\x_t)\right)} \kappa(\x_t,\cdot) & c' \neq y_t \\
	\left(\frac{\exp\left(f_{y_t}(\x_t)\right)}{\sum_{\tilde{c}} \exp\left(f_{\tilde{c}}(\x_t)\right)} - 1\right) \kappa(\x_t,\cdot) & c' = y_t	
	\end{cases}
	\end{align}
	The parametric updates for the entire function vector $\textbf{f}$ are carried out via \eqref{lr_grad}, which operates upon the basis of a stream of samples $\{\x_t, y_t\}$ in tandem with subspace projections. Here also pseudo-gradient property holds since simple stochastic gradients are used.
	
In Appendix \ref{app:KLR_simulations}, we will numerically experiment on multi-class MNIST handwritten dataset using the KLR formulation. There, we will also show the merit of SPPPOT by using the Bregman divergence as KL-divergence compared to its squared RKHS norm difference variant that is nothing but POLK \cite{koppel2019parsimonious}.


\section{First-order Kernel logistic regression on multi-class MNIST}\label{app:KLR_simulations}

	
	\begin{figure*}[t]\centering
		\begin{subfigure}{.33\columnwidth}
			\includegraphics[width=1.1\linewidth, height = 0.75\linewidth]
			{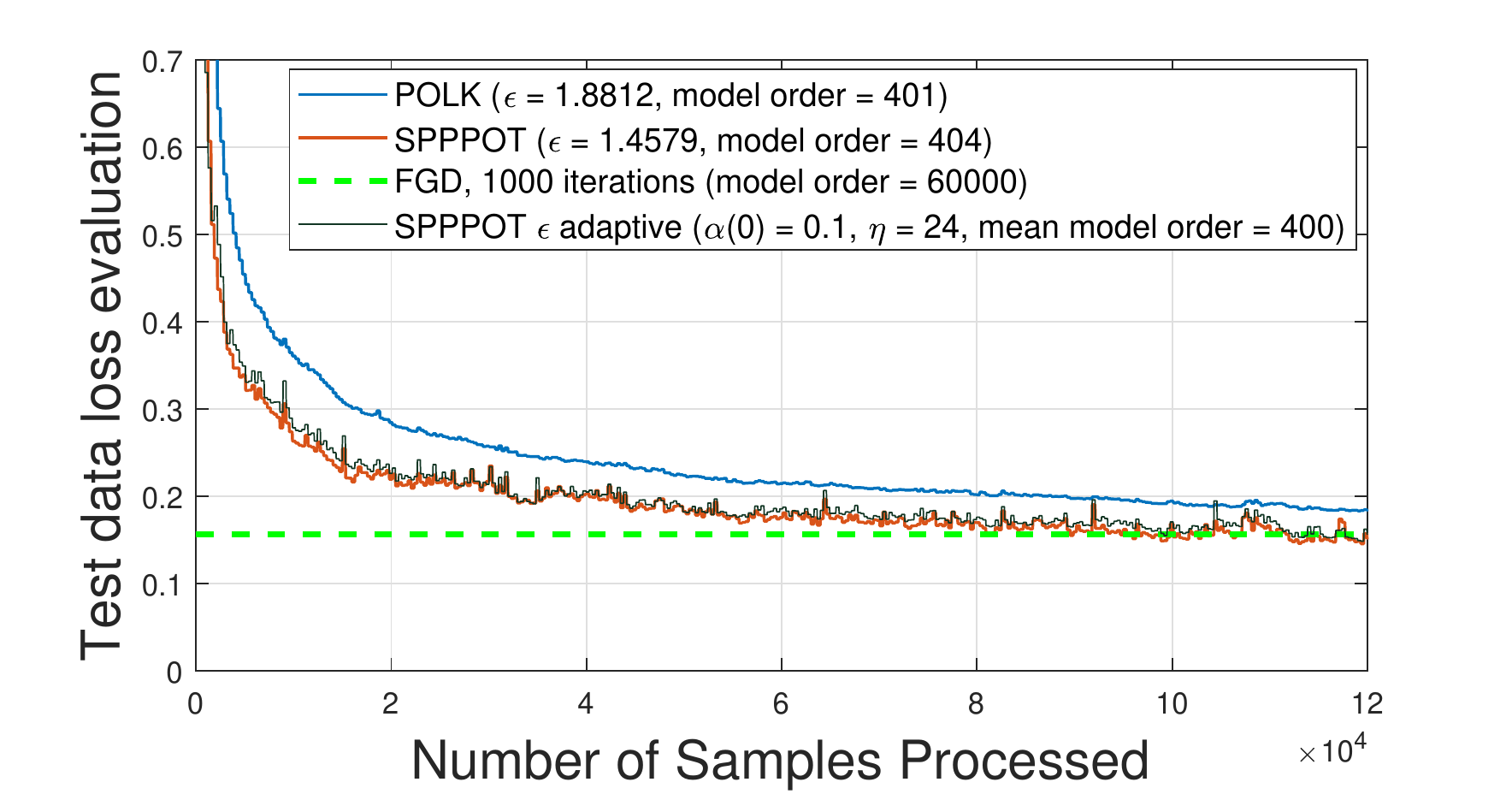}
			\caption{Test Loss vs. samples processed}
			\label{lr_loss_fig}
		\end{subfigure}
		\begin{subfigure}{.32\columnwidth}
			\includegraphics[width=1.1\linewidth,height = 0.75\linewidth]
			{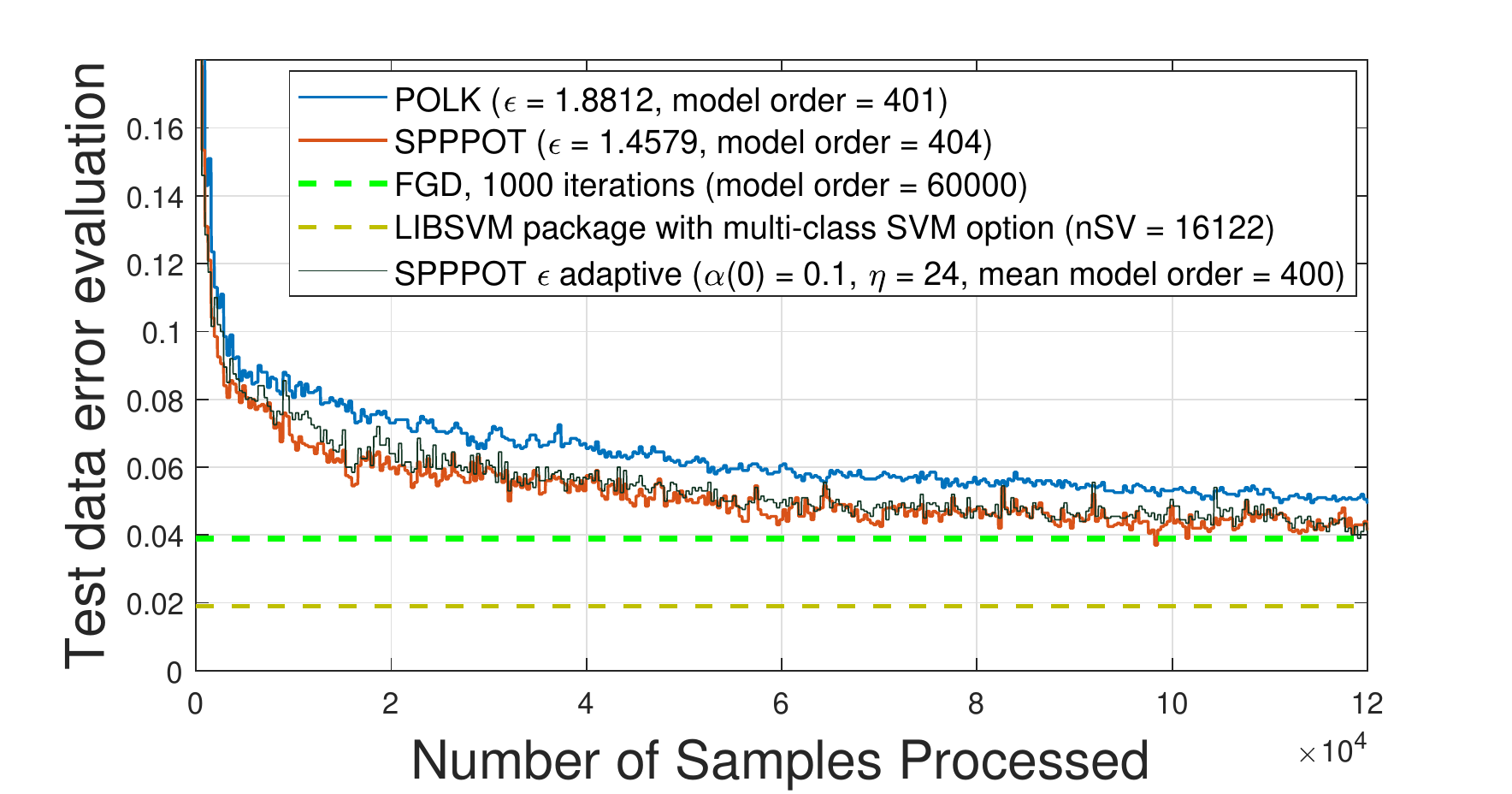}
			\caption{Test Error vs. samples processed}
			\label{lr_error}
		\end{subfigure}
		\begin{subfigure}{.33\columnwidth}
			\includegraphics[width=1.1\linewidth,height = 0.75\linewidth]
			{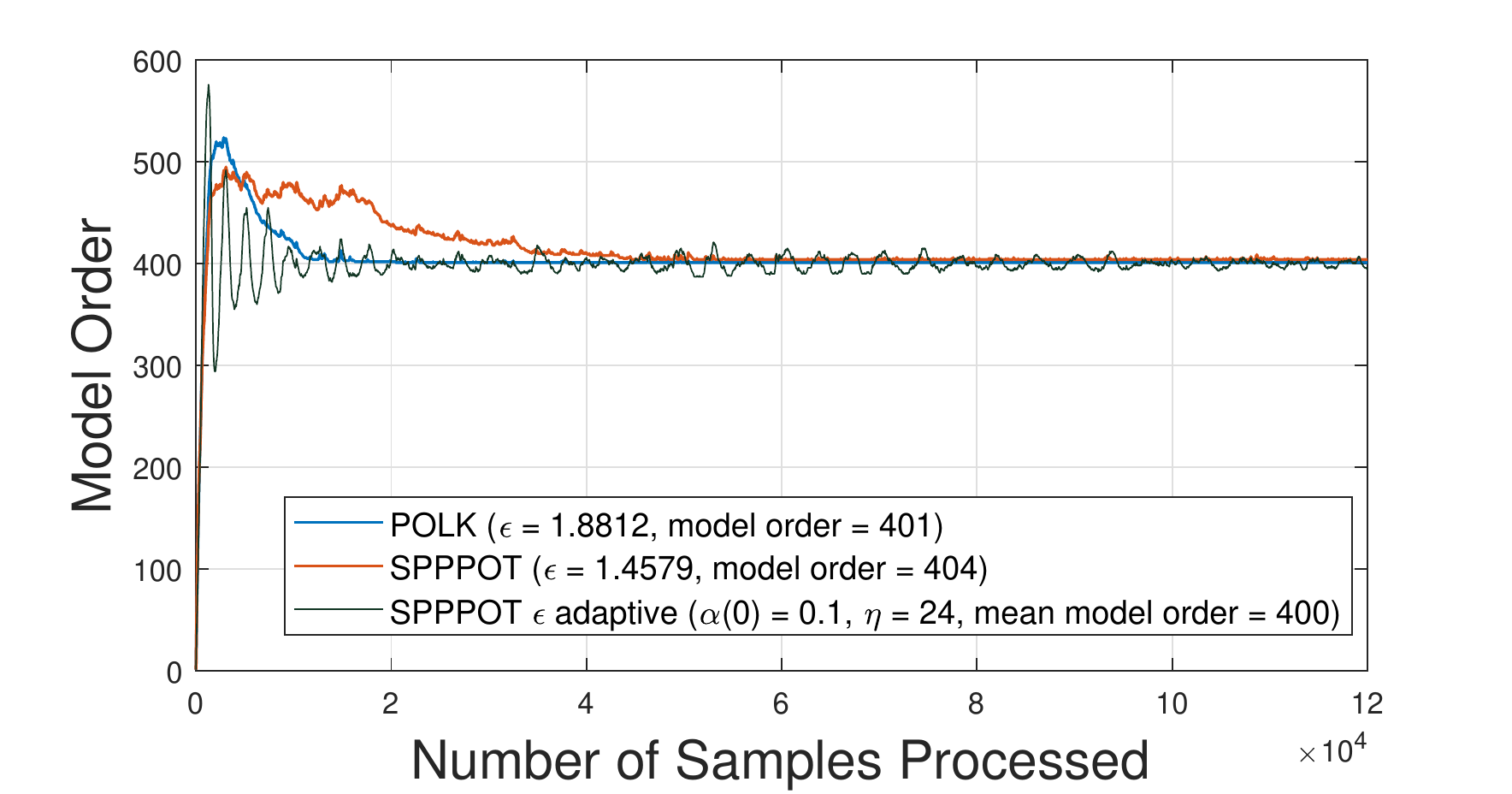}
			\caption{Model Complexity vs. samples processed}
			\label{lr_model_order}
		\end{subfigure}%
		\caption{We compared SPPPOT algorithm (both constant and adaptive budgets) with POLK \cite{koppel2019parsimonious} and offline FGD over MNIST dataset for solving KLR. SPPPOT obtains  competitive accuracy with LIBSVM package, although it is computable online.}
		\label{fig:MNIST_label}
		\end{figure*}
	

We consider the problem of Kernel Logistic Regression (KLR) (Example \ref{eg:logistic}), whose loss is \eqref{logistic_regression_loss} and stochastic gradient takes the form \eqref{lr_grad}. For SPPPOT, we have taken KL divergence as the Bregman divergence and considered both constant and adaptive parsimony constants. We have compared SPPPOT with the online algorithm POLK \cite{koppel2019parsimonious} (without the regularizer term) and with the offline batch algorithm Functional Gradient Descent (FGD) without any dictionary compression. We ran FGD for 1000 iterations. We have also compared with the multi-class kernel SVM (KSVM) solver in LIBSVM package \cite{CC01a}, which is also a batch algorithm. Even though it uses the KSVM algorithm, not KLR and the loss function is also different, still it is a classification algorithm and is used as a baseline to plot the classification error at the end which is shown as a dotted baseline in Fig. \ref{lr_error}.

	\smallskip
	\noindent {\it Dataset.} We consider algorithm performance for the problem of classifying MNIST handwritten digits \cite{lecun1998mnist}, which consists of 60000 data points for training and 10000 data points for testing. Out of 10000 training points, we randomly sampled 2000 points for our training purpose. Each data point is an image of a handwritten digit, and the goal is to label each images into the corresponding digit. We simply read the images into a $784$-dimensional feature vector and scaled the pixel values ranging from 0 to 255 to the intensities in $[0,1]$.

	\smallskip
	\noindent {\it Parameter selection.}
	We select the Gaussian kernel with bandwidth selected by cross-validation over $3.33\%$ of the training data (randomly chosen $2000$ points) and the KOMP budget is kept as zero during cross-validation for both POLK and SPPPOT, which yielded bandwidth is taken as $16$ (standard deviation $4$) for all experiments. A single epoch of training data is fed to both FGD and LIBSVM (C-SVC option) whereas $2$ epochs are being run for the online algorithms POLK and SPPPOT (both constant and adaptive budget). The constant learning rate is taken by trial and error to be $24$ for FGD, POLK and SPPPOT for both budget combinations. The KOMP budget is taken as $\epsilon = 1.4579$ for SPPPOT with constant budget and $\epsilon = 1.8812$ for POLK, which yield converged model orders, respectively, of $401$ and $404$. Matching the model orders for both the experiments is hard and here adaptive budgeting becomes handy. For SPPPOT with adaptive budget, we selected the desired model order as $400$ and the constant $\alpha(0)=0.1$. It is seen in Fig. \ref{lr_model_order} that the mean model order is $400$ and the model order fluctuations with respect to the samples processed are all bounded and eventually converge. The  storage requirement is finite and comparable to POLK with model order $401$.

	\smallskip
	\noindent {\it Results.}
	The test loss, error and the model order plots with respect to samples processed has been presented in Figs. \ref{lr_loss_fig}, \ref{lr_error} and \ref{lr_model_order} respectively. Observe that SPPPOT with both constant and adaptive budgets outperforms POLK and obtains performance competitive with the offline benchmarks FGD and LIBSVM (C-SVC) whose parameterization complexity is $60000$ and $16122$, respectively, which are not depicted in the same figure due to scale. SPPPOT, after the loss and error settles down, performs comparably to FGD with 1000 iterates, whereas for POLK, there is still a considerable gap. FGD took about $12$ hours to compute a single epoch with $1000$ iterations. As seen in Fig. \ref{lr_error}, KVSM has better accuracy than the asymptotic performance of KLR algorithms. Still we can say the test error for SPPPOT is within $2\%$ of the LIBSVM solver and yields a more efficient parameterization.


\section{Almost sure convergence of SPPPOT for stochastic gradients}\label{app:as_proof}

Stochastic gradients are special case of pseudo-gradients which are available whenever the instantaneous cost is evaluable in closed form as in, for instance,  kernel logistic regression [cf. \eqref{lr_grad}]. Note that stochastic gradients satisfy the stochasticity rule, i.e. $\EE[g_t | \cF_t] = \nabla R(f_t)$. To prove the almost sure convergence of RKHS-valued stochastic mirror descent with projections, only Assumption \ref{fin} and the next three conditions are required.

\begin{assumption}\label{dualdomainloss_stochastic}
	The function $R_\psi(\cdot)$ is $\lambda_1$-strongly convex.
\end{assumption}
\begin{assumption}\label{psi}
	The function $\psi$ is $1$-strongly convex. 
\end{assumption}
\begin{assumption}\label{grad_bound}
	The projected stochastic gradient has bounded second moment given the filtration $\mathcal{F}_t=\sigma(\{\x_i\}_{i=1}^{t-1})$.
	\begin{align}\label{eq_grad_bound}
		\Ex{\ns{\gh_t}_\ast \vert \cF_t} \leq G^2
	\end{align}
\end{assumption}

Assumption \ref{psi} is required to lower bound the norm difference between the auxiliary iterates from the optima, i.e. $\norm{z_t - z^\ast}_\ast$ by the the norm difference between the actual function iterates and the optima, i.e. $\norm{f_t - f^\ast}$. Note that Assumption \ref{dualdomainloss_stochastic} can be enforced explicitly by adding a regularizer term in the dual space with respect to the auxiliary function $z_t$. Assumptions \ref{dualdomainloss_stochastic} and \ref{grad_bound} are standard in the literature and is required to show the almost sure convergence theorem for the stochastic case.

\begin{theorem}\label{thm_stochastic_sync}
	Under Assumption \ref{fin} and Assumptions \ref{dualdomainloss_stochastic}-\ref{grad_bound}, upon running Algorithm \ref{SPFPMD_algorithm} with constant step-size $\eta$ and constant compression budget $\epsilon$, the iterates $f_t$ converges almost surely to a neighbourhood in limit infrimum as $t \rightarrow \infty$.
	\begin{align}
		\liminf_{t \rightarrow \infty} \norm{f_t - f^\ast} \leq \xi := \frac{\epsilon + \sqrt{\epsilon^2 + \eta^3 \lambda_1 G^2}}{2 \eta \lambda_1}
	\end{align}
	where $G$ is defined in \eqref{eq_grad_bound} as the bound on the conditional second-moment of the stochastic gradient in the dual norm and $\lambda_1$ is the strong-convexity parameter.
\end{theorem}

\begin{proof} The proof starts by taking the dual Hilbert norm difference between $z_{t+1}$ and $z^\ast = \nabla \psi(f^\ast)$, where $z^\ast$ is taken to be the transformed function of the optimum $f^\ast$ in the dual space, and then expanding $z_{t+1}$ same as \eqref{function_dual_update} as shown below
\begin{align}
	\ns{z_{t+1} - z^\ast}_\ast &= \ns{z_t - \eta \gh_t - z^\ast}_\ast \\
	&= \ns{z_t - z^\ast}_\ast -2 \eta \langle z_t - z^\ast, \gh_t  \rangle + \eta^2 \ns{\gh_t}_\ast \\
	& = \ns{z_t - z^\ast}_\ast -2 \eta \langle z_t - z^\ast, g_t  \rangle + 2 \eta \langle z_t - z^\ast, g_t - \gh_t  \rangle + \eta^2 \ns{\gh_t}_\ast \; , \label{eq_bound1}
\end{align}
where the last equation is obtained by addition and subtraction of the actual stochastic gradient $g_t$. The third term on the right hand side of \eqref{eq_bound1} can be upper bounded using Cauchy-Schwartz inequality and then application of \eqref{approxerr} on it yields the bound
\begin{align}
	\ns{z_{t+1} - z^\ast}_\ast \leq &\ns{z_t - z^\ast}_\ast -2 \eta \langle z_t - z^\ast, g_t  \rangle + 2 \epsilon \norm{z_t - z^\ast}_\ast + \eta^2 \ns{\gh_t}_\ast \; .
\end{align}
So now taking conditional expectation given the past sigma algebra $\cF_t$ and using Assumption \ref{grad_bound}, we get
\begin{align}
	\Ex{\ns{z_{t+1} - z^\ast}_\ast \vert \cF_t} &\leq \ns{z_t - z^\ast}_\ast -2 \eta \langle z_t - z^\ast, \nabla R(f_t)  \rangle + 2 \epsilon \norm{z_t - z^\ast}_\ast + \eta^2 G^2 \; . \label{eq_bound2}
\end{align}
We assumed $R_\psi(\cdot)$ to be the risk function which takes the auxiliary function $z_t$ as an argument such that $R_\psi(z_t) = R_\psi(\nabla \psi(f_t)) = R(f_t)$, where $R(\cdot)$ is our actual risk function. Hence, using this property and convexity of $R_\psi$ and then Assumption \ref{dualdomainloss_stochastic}, the second term on the right hand side of \eqref{eq_bound2} can be simplified as
\begin{align}
	\langle z_t - z^\ast, \nabla R(f_t)  \rangle &= \langle z_t - z^\ast, \nabla R_\psi(z_t)  \rangle \\
	& \geq R_\psi(z_t) - R_\psi(z^\ast) \\
	& \geq \frac{\lambda_1}{2} \ns{z_t - z^\ast}_\ast \; . \label{eq_bound2_2nd_term}
\end{align}
So now using \eqref{eq_bound2_2nd_term} in \eqref{eq_bound2}, the inequality further simplifies to
\begin{align}
	\Ex{\ns{z_{t+1} - z^\ast}_\ast \vert \cF_t} &\leq (1-\eta \lambda_1) \ns{z_t - z^\ast}_\ast + 2 \epsilon \norm{z_t - z^\ast}_\ast + \eta^2 G^2 \; . \label{eq_main_bound}
\end{align}
A stopping stochastic process can be constructed using inequality \eqref{eq_main_bound}, that tracks when $\ns{z_t - z^\ast}_\ast$ reaches the sub-optimal threshold. When the threshold is achieved, the function converges to the neighbourhood of the optimum. In order to establish this fact, we construct a stochastic process $\delta_t$ such that the supermartingale condition is satisfied, i.e. $\Ex{\delta_{t+1} \vert \cF_t} \leq \delta_t$. We would like to find out the threshold from \eqref{eq_main_bound} by finding out the roots for $\norm{z_t-z^\ast}$ when the following holds true
\begin{align}
	&\Ex{\ns{z_{t+1} - z^\ast}_\ast \vert \cF_t} \leq \ns{z_t-z^\ast}_\ast \label{supermartingale_function_inequality} \\
	&(1-\eta \lambda_1) \ns{z_t - z^\ast}_\ast + 2 \epsilon \norm{z_t - z^\ast}_\ast + \eta^2 G^2 \leq \ns{z_t-z^\ast}_\ast \; . \label{required_inequality} 
\end{align}
\eqref{required_inequality} simplifies to the following quadratic inequality
\begin{align}
	-\eta \lambda_1 \ns{z_t - z^\ast}_\ast + 2 \epsilon \norm{z_t - z^\ast}_\ast + \eta^2 G^2 \leq 0 \; . \label{inequality_using_supermartingale}
\end{align}
The roots of the polynomial are
\begin{align}
	\norm{z_t-z^\ast}_\ast &= \frac{-2\epsilon \pm \sqrt{4\epsilon^2 + 4 (\eta \lambda_1)(\eta^2 G^2)}}{-4 \eta \lambda_1} \\
	&= \frac{\epsilon \mp \sqrt{\epsilon^2 + \eta^3 \lambda_1 G^2}}{2 \eta \lambda_1} \; .
\end{align}
The inequality \eqref{inequality_using_supermartingale} opens downwards and since $\norm{z_t-z^\ast}_\ast$ is always positive, hence only positive root will be taken. So the radius of the convergence can be represented as
\begin{align}\label{radius_of_conv}
	\xi := \frac{\epsilon + \sqrt{\epsilon^2 + \eta^3 \lambda_1 G^2}}{2 \eta \lambda_1} \; .
\end{align}
So using \eqref{radius_of_conv}, let's now define the stopping stochastic process $\delta_t$ as
\begin{align}
	\delta_t = \norm{z_t - z^\ast}_\ast \II \Big\{ \min_{u \leq t} -\eta \lambda_1 &\ns{z_u - z^\ast}_\ast + 2 \epsilon \norm{z_u - z^\ast}_\ast + \eta^2 G^2 > \xi \Big\} \; ,
\end{align}
where $\II$ denotes the indicator function. Note $\delta_t \geq 0$ since both the norm and the indicator are non-negative. So when the indicator condition $\min_{u \leq t} -\eta \lambda_1 \ns{z_u - z^\ast}_\ast + 2 \epsilon \norm{z_u - z^\ast}_\ast + \eta^2 G^2 > \xi$ is satisfied, we may compute the square root of \eqref{supermartingale_function_inequality} yielding
\begin{align} \label{sup_martingale}
	\Ex{\delta_{t+1} \vert \cF_t} \leq \delta_t .
\end{align}
Now taking the contrary condition, i.e., $\min_{u \leq t} -\eta \lambda_1 \ns{z_u - z^\ast}_\ast + 2 \epsilon \norm{z_u - z^\ast}_\ast + \eta^2 G^2 \leq \xi$, makes the indicator to be $0$ for the subsequent iterates, since min is used inside the indicator. So, in either case, \eqref{sup_martingale} holds, which implies that $\delta_t \xrightarrow{a.s.} 0$ as $t \rightarrow \infty$. This further implies the fact that either $\lim_{t \rightarrow \infty} \norm{z_t - z^\ast}_\ast - \xi = 0$ or the indicator goes to null for large $t$ almost surely, i.e. $\lim_{t \rightarrow \infty} \II \Big\{ \min_{u \leq t} -\eta \lambda_1 \ns{z_u - z^\ast}_\ast + 2 \epsilon \norm{z_u - z^\ast}_\ast + \eta^2 G^2 > \xi \Big\} = 0$. Hence any of these conditions terminating to null translates to our convergence in the neighbourhood
\begin{align}\label{final_condn_in_z}
	\liminf_{t \rightarrow \infty} \norm{z_t - z^\ast}_\ast \leq \xi \; .
\end{align}
Now using \ref{psi}, we can write
\begin{align}\label{transform_z_to_f}
	\norm{z_t - z^\ast}_\ast = \norm{\nabla \psi(f_t) - \nabla \psi(f^\ast)}_\ast \geq \norm{f_t - f^\ast} \; .
\end{align}
Using \eqref{transform_z_to_f} in \eqref{final_condn_in_z} concludes the proof of theorem \ref{thm_stochastic_sync}.
\end{proof}

Note that the almost sure convergence established for stochastic gradients in theorem \ref{thm_stochastic_sync} is a much stronger bound than that of the pseudo-gradient convergence result given in theorem \ref{thm_sync}. Hence when stochastic gradient computation is feasible, as in Kernel logistic regression for multi-class classification, on running SPPPOT we can almost surely reach to the neighborhood of the optima.


	\footnotesize
	\bibliographystyle{IEEEtran} 
	\bibliography{IEEEabrv,reference}
	
\end{document}